%% file: main.tex
\documentclass{article}

\usepackage{microtype}
\usepackage{graphicx}
\usepackage{subfigure}
\usepackage{booktabs} % for professional tables

\usepackage{hyperref}

\usepackage[accepted]{icml2025}

\usepackage{amsmath}
\usepackage{amssymb}
\usepackage{mathtools}
\usepackage{amsthm}

\allowdisplaybreaks
\usepackage{enumitem}
\usepackage{nccmath}

\newcommand{\alglinelabel}{%
  \addtocounter{ALC@line}{-1}% Reduce line counter by 1
  \refstepcounter{ALC@line}% Increment line counter with reference capability
  \label% Regular \label
}
\usepackage{float}

\usepackage[normalem]{ulem}
\usepackage{url}            % simple URL typesetting
\usepackage{amsfonts}       % blackboard math symbols
\usepackage{nicefrac}       % compact symbols for 1/2, etc.
\usepackage{algorithm}
\usepackage{algorithmic}
\usepackage[utf8]{inputenc}
\usepackage{subcaption}
\usepackage{tikz}
\usetikzlibrary{calc}
\usepackage{caption}
\usepackage{bbm}
\usepackage{bm}
\newcommand{\xdashrightarrow}[2][]{\ext@arrow 0359\rightarrowfill@@{#1}{#2}}

\usepackage{etoc}
\usepackage{minitoc}  % For localized TOCs
\usepackage[toc,page]{appendix}  % For handling appendix sections
\newcommand{\nocontentsline}[3]{}
\let\origcontentsline\addcontentsline
\newcommand\stoptoc{\let\addcontentsline\nocontentsline}
\newcommand\resumetoc{\let\addcontentsline\origcontentsline}

\usepackage[capitalize,noabbrev]{cleveref}

\theoremstyle{plain}
\newtheorem{theorem}{Theorem}[section]
\newtheorem{proposition}[theorem]{Proposition}
\newtheorem{lemma}[theorem]{Lemma}

\theoremstyle{definition}
\newtheorem{definition}[theorem]{Definition}
\newtheorem{assumption}[theorem]{Assumption}
\theoremstyle{remark}

\DeclareMathOperator*{\argmax}{arg\,max}
\DeclareMathOperator*{\argmin}{arg\,min}
\icmltitlerunning{Natural Policy Gradient for Average Reward Non-Stationary RL}

\input{includes/macros}
\begin{document}

\stoptoc

\twocolumn[
\icmltitle{Natural Policy Gradient for Average Reward Non-Stationary RL}

% It is OKAY to include author information, even for blind
% submissions: the style file will automatically remove it for you
% unless you've provided the [accepted] option to the icml2025
% package.

% List of affiliations: The first argument should be a (short)
% identifier you will use later to specify author affiliations
% Academic affiliations should list Department, University, City, Region, Country
% Industry affiliations should list Company, City, Region, Country

% You can specify symbols, otherwise they are numbered in order.
% Ideally, you should not use this facility. Affiliations will be numbered
% in order of appearance and this is the preferred way.
\icmlsetsymbol{equal}{*}

\begin{icmlauthorlist}
\icmlauthor{Neharika Jali}{cmu}
\icmlauthor{Eshika Pathak}{cmu}
\icmlauthor{Pranay Sharma}{cmu}
\icmlauthor{Guannan Qu}{cmu}
\icmlauthor{Gauri Joshi}{cmu}
\end{icmlauthorlist}

\icmlaffiliation{cmu}{Department of Electrical and Computer Engineering, Carnegie Mellon University}
% \icmlaffiliation{iitb}{Center for Machine Intelligence and Data Science, Indian Insitute of Technology Bombay}
% \icmlaffiliation{uiuc}{Department of Electrical and Computer Engineering, University of Illinois Urbana-Champaign}

\icmlcorrespondingauthor{Neharika Jali}{njali@cmu.edu}

% You may provide any keywords that you
% find helpful for describing your paper; these are used to populate
% the "keywords" metadata in the PDF but will not be shown in the document
\icmlkeywords{Non-Stationary Reinforcement Learning, Policy Gradient, Natural Actor-Critic}

\vskip 0.3in
]

% this must go after the closing bracket ] following \twocolumn[ ...

% This command actually creates the footnote in the first column
% listing the affiliations and the copyright notice.
% The command takes one argument, which is text to display at the start of the footnote.
% The \icmlEqualContribution command is standard text for equal contribution.
% Remove it (just {}) if you do not need this facility.

\printAffiliationsAndNotice{}  % leave blank if no need to mention equal contribution
% \printAffiliationsAndNotice{\icmlEqualContribution} % otherwise use the standard text.

\begin{abstract}
We consider the problem of non-stationary reinforcement learning (RL) in the infinite-horizon average-reward setting. We model it by a Markov Decision Process with time-varying rewards and transition probabilities, with a variation budget of $\Delta_T$. Existing non-stationary RL algorithms focus on model-based and model-free value-based methods. Policy-based methods despite their flexibility in practice are not theoretically well understood in non-stationary RL. We propose and analyze the first model-free policy-based algorithm, Non-Stationary Natural Actor-Critic (\algoName), a policy gradient method with a {\color{black} restart based} exploration for change and a novel interpretation of learning rates as adapting factors. {\color{black} Further, we present a bandit-over-RL based parameter-free algorithm \borlAlgoName~that does not require prior knowledge of the variation budget $\Delta_T$.} We present a dynamic regret of $\mathcal{\Tilde{O}} (|\gS|^{1/2}|\gA|^{1/2}{\color{black}\Delta_T^{1/6}T^{5/6}} )$ {\color{black} for both algorithms}, where $T$ is the time horizon, and $|\gS|$, $|\gA|$ are the sizes of the state and action spaces. The regret analysis leverages a novel adaptation of the Lyapunov function analysis of NAC to dynamic environments and characterizes the effects of simultaneous updates in policy, value function estimate and changes in the environment. 
\end{abstract}

%%%%%%%%%%%%%%%%%%%
%%%%%%%%%%%%%%%%%%%

\section{Introduction} \label{sec:introduction}
Reinforcement Learning is a sequential decision-making framework where an agent learns optimal behavior by iteratively interacting with its environment. At each timestep, the agent observes the current state of the environment, takes an action, receives a reward, and transitions to the next state. While RL has traditionally been studied in stationary environments with time-invariant rewards and state-transition dynamics, this may not always be the case. Consider the examples of a carbon-aware datacenter job scheduler that tracks the dynamic electricity prices and local weather patterns \citep{yeh2024sustaingym} and recommendation systems with evolving user preferences \citep{chen2018stabilizing}. Time-varying environments are also observed in inventory control \citep{mao2021nearOptimal}, healthcare \citep{chandak2020optimizing}, ride-sharing \citep{kanoria2024blind}, and multi-agent systems \citep{zhang2021multi}.

Motivated by these applications, we consider the problem of non-stationary reinforcement learning, modeled by a Markov Decision Process with time-varying rewards and transition probabilities, in the infinite horizon average reward setting. While many works consider discounted rewards \citep{chandak2020optimizing, igl2020transient, lecarpentier2019non}, the more challenging average-reward setting is vital in representing problems where the importance of rewards does not decay with time, such as in robotics \citep{mahadevan1996average, peters2003reinforcement} or scheduling workloads in cloud computing systems \citep{jali2024efficient, liu2022RLQN}. The key challenges for an agent operating in a dynamic environment are learning an optimal behavior policy that varies with the environment, devising an efficient exploration strategy, and effectively incorporating the acquired information into its behavior. 

Current algorithms designed for non-stationary MDPs in the average reward setting can be classified broadly into model-based and model-free value-based methods. Model-based solutions incorporate sliding windows, forgetting factors, and confidence interval management mechanisms into UCRL \citep{cheung2020reinforcement, ortner2020variational, gajane2018sliding, jaksch2010near}. Model-free value-based methods assimilate restarts and optimism into Q-Learning \citep{mao2021nearOptimal, feng2023non} and LSVI \citep{zhou2020nonstationary, touati2020efficient}. A significant gap in the literature is the absence of model-free policy-based techniques for time-varying environments. The inherent flexibility of policy-based algorithms makes them suitable for continuous state-action spaces, facilitates efficient parameterization in high-dimensional state-action spaces, and enables effective exploration through stochastic policy learning \citep{SuttonBartoBook}.

\newpage
\paragraph{Our Contributions.} We tackle the problem of non-stationary reinforcement learning in the challenging infinite-horizon average reward setting in the following manner.
\begin{enumerate}[leftmargin=*]
    \item We propose and analyze Non-Stationary Natural Actor-Critic (\algoName), a policy gradient method with a {\color{black} restart based} exploration for change and a novel interpretation of learning rates as adapting factors. To the best of our knowledge, this is the first model-free policy-based method for time-varying environments. 
    \item {\color{black} We present a bandit-over-RL based parameter-free algorithm \borlAlgoName~that does not require prior knowledge of the variation budget.}
    \item We present a $\mathcal{\Tilde{O}}\left(|\gS|^{1/2}|\gA|^{1/2}{\color{black}\Delta_T^{1/6}T^{5/6}}\right)$ dynamic regret bound {\color{black} for both algorithms} under standard assumptions where $T$ is the time horizon, $\Delta_T$ represents the variation budget of rewards and transition probabilities, $|\gS||\gA|$ is the size of the state-action space and $\mathcal{\Tilde{O}}(\cdot)$ hides logarithmic factors. The regret analysis leverages a novel adaptation of the Lyapunov function analysis of NAC to dynamic environments and characterizes the effects of simultaneous updates in policy, value function estimate and changes in the environment. 
\end{enumerate}

%%%%%%%%%%%%%%%%%%%
%%%%%%%%%%%%%%%%%%%

\section{Related Work} \label{sec:relatedWork}

\textbf{Non-Stationary RL.} Solutions to the non-stationary RL problem can be categorized into passive and active methods. Active algorithms are designed to actively detect changes in the environment in contrast to passive ones which implicitly adapt to new environments without distinct recognition of the change. While we focus our attention on passive techniques with dynamic regret as the performance metric in this work, a comprehensive survey can be found in \citet{padakandla2021survey} and \citet{khetarpal2022towards}. Model-based solutions in the infinite horizon average reward setting incorporate into UCRL a sliding window or a forgetting factor for piecewise stationary MDPs \citep{gajane2018sliding}, variation aware restarts \citep{ortner2020variational} and a bandit based tuning of sliding window and confidence intervals \citep{cheung2020reinforcement} for gradual or abrupt changes constrained by a variation budget and pessimistic tree search for Lipschitz continuous changes \citep{lecarpentier2019non}. 

In the episodic setting, model-free value based methods assimilate restarts and optimism into Q-Learning \citep{mao2021nearOptimal}, LSVI \citep{zhou2020nonstationary, touati2020efficient} and sliding window and optimistic confidence set based exploration into a value function approximated learning \citep{feng2023non}. Further, in the episodic setting, \citet{lee2024pausing} proposes strategically pausing learning as an effective solution to non-stationarity with forecasts of the future. \citet{wei2021non} proposed an algorithm agnostic black-box approach that finds a non-stationary equivalent to optimal regret stationary MDP algorithms. \citet{cheung2020reinforcement, mao2021nearOptimal} also present parameter-free non-stationary RL algorithms that leverage the bandit-over-RL framework to adaptively tune algorithm without knowledge of the variation budget. Further, \citet{mao2021nearOptimal} presents an information theoretic lower bound on the dynamic regret and \citet{peng2024complexity} captures the complexity of updating value functions with any change. We note the distinction between the scope of this work and the body of research on adversarial MDPs which often allow for only changes in rewards, study the static regret and work with full information feedback instead of bandit feedback. See \Cref{app:additionalRelatedWork} for a table of comparison of regret bounds.

\textbf{Non-Stationary Bandits.} A precursor to non-stationary RL, the multi-armed bandit problem with time-varying rewards was first proposed in \citet{garivier2008upper}. Solutions include UCB with a sliding window or a discounting factor \citep{garivier2008upper}, UCB with adaptive blocks of exploration and exploitation \citep{besbes2014stochastic}, Restart-Exp3 \citep{besbes2014stochastic}, Thompson Sampling with a discounting factor \citep{raj2017taming} and bandit based sliding window tuning \citep{cheung2019learning}. Further, while most existing works assume arbitrarily (constrained by variation budget) changing reward distributions, \citep{jia2023smooth} achieves an improved regret when the reward distributions change smoothly. Recent work by \citet{liu2023definition} points out ambiguities in the definition of non-stationary bandits and how the dynamic regret performance metric causes over-exploration, and \citet{liu2023nonstationary} proposes, predictive sampling, an algorithm that deprioritizes acquiring information that loses usefulness quickly. 

\textbf{Policy Gradient Algorithms for Stationary RL.} 
\citet{Wu2020A2CPG} presents the first finite time convergence of the average reward two timescale Advantage Actor-Critic to a stationary point. \citet{chen2023finite} further improved its rate by leveraging a single timescale algorithm. 
Convergence to global optima of A2C was analyzed in \citet{bai2024regret, murthy2023performance} which use a two loop structure with the inner loop critic estimation.
Further, \citet{lazic2021improved} combines a mirror descent update with experience replay and characterized global convergence. 
Natural Policy Gradient (NPG) was analyzed in the discounted reward case in \citet{pgALKM, khodadadian2021linear} and with entropy regularization in \citet{cen2022fast}. NPG in the average reward setting with exact gradients was characterized in \citet{even2009online, murthy2023convergence}.
The most relevant to our work is the Natural Actor Critic (NAC) algorithm where the actor learns the policy by natural gradient ascent and critic estimate the value function analyzed for the discounted reward case in \citet{khodadadian2022finite} and average reward setting with (compatible) function approximation in \citet{wang2024NonAsymptotic}.

%%%%%%%%%%%%%%%%%%%
%%%%%%%%%%%%%%%%%%%

\section{Problem Setting} \label{sec:problemSetting}

In this section, we first present preliminaries of a Markov Decision Process and the Natural Actor Critic algorithm in a stationary environment. We then introduce the problem of non-stationary reinforcement learning, where the MDP has time-varying rewards and transition probabilities, and define dynamic regret as a performance metric.

\textbf{Notation.} Standard typeface (e.g., $\s$) denote scalars and bold typeface (e.g., $\rv$, $\Av$) denote vectors and matrices. $\Vert\cdot\Vert_\infty$ denotes the infinity norm and $\Vert\cdot\Vert_2$ denotes the 2-norm of vectors and matrices. Given two probability measures $P$ and $Q$, $\dtv(P,Q) = \frac{1}{2}\int_{\mathcal{X}} \lvert P(dx) - Q(dx)\rvert$ is the total variation distance between $P$ and $Q$, while $\KL(P\Vert Q) = \int_{\mathcal{X}} P(dx)\log \frac{P(dx)}{Q(dx)}$ is the KL-divergence. For two sequences $\{a_n\}$ and $\{b_n\}$, $a_n = \mathcal{O}(b_n)$ represents the existence of an absolute constant $C$ such that $a_n \leq C b_n$. Further $\mathcal{\Tilde{O}}$ is used to hide logarithmic factors. $\lvert\mathcal{S}\rvert$ denotes the cardinality of a set $\mathcal{S}$. Given a positive integer $T$, $[T]$ denotes the set $\{0, 1, 2, \cdots, T-1\}$.

\subsection{Preliminaries: Stationary RL} \label{sec:problemSettingPreliminaries}
\textbf{Markov Decision Process.} Reinforcement learning tasks can be modeled as discrete-time Markov Decision Processes (MDPs). An MDP is represented as $\mathcal{M} = (\mathcal{S}, \mathcal{A}, \pv, \rv)$ where $\mathcal{S}$ and $\mathcal{A}$ are, respectively, finite sets of states and actions, $\pv \in \mbb R^{\lvert\mathcal{S}\rvert\lvert\mathcal{A}\rvert \times \lvert\mathcal{S}\rvert}$ is the transition probability matrix, with $\p(\s' | \s, \a) \in [0, 1]$, for $s,s' \in \gS, a \in \gA$, and $\rv \in \mbb R^{\lvert\mathcal{S}\rvert\lvert\mathcal{A}\rvert}$ is the reward vector with individual entries $\{\r(\s, \a)\}$ bounded in magnitude by constant $U_R>0$. 
% \gauri{Is $U_R$ defined earlier? If not, add a few words here saying that reward magnitude is bounded by constant $U_R > 0$.} 
An agent in state $\s$ takes an action $\a\sim\pol(\cdot|\s)$ according to a policy $\polv$, where for each state $\s$, $\pol(\cdot|\s)$ is a probability distribution over the action space. The agent then receives a reward $\r(\s, \a)$ and transitions to the next state $\s' \sim \p(\cdot | \s, \a)$. We denote the \textit{policy} by $\polv \in \mbb R^{|\gS||\gA|}$, which concatenates $\{\pol(\cdot|\s)\}_s$. In a stationary MDP, the transition probabilities $\pv$ and the rewards $\rv$ are \textit{time-invariant}.

\textbf{Average Reward and Value Functions.} In this work, we consider the average reward setting (instead of discounted rewards), which is essential to model problems where the importance of rewards does not decay with time \citep{peters2003reinforcement, liu2022RLQN}. Time averaged reward of an ergodic Markov chain following policy $\polv$ converges to
\begin{equation*}
    J^\polv := \lim_{T \rightarrow \infty} \dfrac{\sum_{t=0}^{T-1} \r(\s_t, \a_t)}{T} = \mathbb{E}_{\s\sim d^{\polv, \pv}(\cdot), \a\sim\pol(\cdot|\s)}\left[\r(\s, \a) \right], 
    % \label{eq:avgReward}
\end{equation*}
where $d^{\polv, \pv}$ is the stationary distribution over states induced by policy $\polv$ and transition probabilities $\pv$. The \textit{relative} state-value function defines overall reward (relative to the average reward) accumulated when starting from state $\s$ as 
\begin{equation*}
    \v^\polv (\s) := \mathbb{E} \left[\sum_{t=0}^\infty \left(\r(\s_t, \a_t) - J^\polv \right) \Big| \s_0 = \s \right], \label{eq:stateValueFn}
\end{equation*}
where the expectation is over the trajectory rolled out by $\a_t \sim \pol(\cdot | \s_t)$ and $\s_{t+1} \sim \p(\cdot |\s_t, \a_t)$. Similarly, the relative state-action value function defines the overall reward (relative to the average reward) accumulated by policy $\polv$ when starting from state $\s$ and action $\a$ as 
\begin{equation*}
    \q^\polv (\s, \a) := \mathbb{E}\left[\sum_{t=0}^\infty \left(\r(\s_t, \a_t) - J^\polv\right) \lvert \s_0=\s, \a_0 = \a \right]. \label{eq:stateActionValueFn}
\end{equation*}

\textbf{Natural Actor-Critic.} 
The goal of an agent is to find a policy that maximizes the average reward 
\begin{equation*}
    \polv^\star = \max\limits_{\polv} J^\polv = \max\limits_{\polv} \mathbb{E}_{\s\sim d^{\polv, \pv}(\cdot), \a\sim\pol(\cdot|\s)} \left[\r(\s, \a)\right].
\end{equation*}
Here, we consider the actor-critic class of policy-based algorithms. While actor-only methods are at a disadvantage due to inefficient use of samples and high variance and critic-only methods are at a risk of the divergence from the optimal policy, actor-critic methods provide the best of both worlds \citep{Wu2020A2CPG}. An actor-critic algorithm learns the policy and the value function simultaneously by gradient methods. Further, the natural actor-critic leverages the second-order method of natural gradient to establish guarantees of global optimality \citep{bhatnagar2009natural, khodadadian2022finite}. The \textit{actor} updates the policy by performing a natural gradient ascent \citep{martens2020new} step
\begin{gather}
    \polv \leftarrow \polv + \alr F_\polv^{-1} \nabla J^{\polv}, \\
    F_\polv  := \mathbb{E}_{\s\sim d^{\polv, \pv}(\cdot), \a\sim\pol(\cdot|\s)} \lbr \nabla \log\pol(\a|\s) \left(\nabla \log\pol(\a|\s)\right)^\top \rbr \nonumber. \label{eq:NaturalGrad}
\end{gather}
$F_\polv$ is called the Fisher Information matrix. The gradient of the average reward is given by the Policy Gradient Theorem \citep[Section~13.2]{SuttonBartoBook} as
\begin{equation*}
    \nabla J^\polv = \mathbb{E}_{\s\sim d^{\polv, \pv}(\cdot), \a\sim\pol(\cdot|\s)}\left[\q^\polv(\s, \a) \nabla \log\pol(\a|\s) \right].
\end{equation*}
The \textit{critic} enables an approximate policy gradient computation by estimating the Q-Value function $\q^\polv(\s, \a)$ using TD-learning as
\begin{equation*}
    \q(\s, \a) \leftarrow \q(\s, \a) + \clr\left[\r(\s, \a) - \eta + \q(\s', \a') - \q(\s, \a)\right],
\end{equation*}
where $s' \sim \p(\cdot|s,a)$, $a' \sim \pol(\cdot|\s')$, and $\eta$ is an estimate of the average reward $J^{\polv}$.

\subsection{Non-Stationary RL} \label{sec:problemSettingNSMDP}

In this work, we study reinforcement learning with \textit{time-varying environments}. The MDP is modeled by a sequence of environments $\mathcal{M} = \{\mathcal{M}_t = (\mathcal{S}, \mathcal{A}, \pv_t, \rv_t)\}_{t=0}^{T-1}$, with time-varying rewards $\{\rv_t\}$ and transition probabilities $\{\pv_t\}$. At each time $t$, the agent in state $\s_t$ takes action $\a_t$, receives a reward $\r_t(\s_t, \a_t)$, and transitions to the next state $\s_{t+1} \sim \p_t(\cdot|\s_t, \a_t)$. The cumulative change in the reward and transition probabilities is quantified in terms of \textit{variation budgets} $\Delta_{R, T}$ and $\Delta_{P, T}$ as
\begin{gather}
    \Delta_{R, T} = \sum_{t=0}^{T-1} \Vert \rv_{t+1} - \rv_t \Vert_\infty, \quad \Delta_{P, T} = \sum_{t=0}^{T-1} \Vert \pv_{t+1} - \pv_t \Vert_\infty \nonumber \\ \Delta_T = \Delta_{R, T} + \Delta_{P, T}. \label{eq:variationBudget}
\end{gather}
Note that while the overall budgets $\Delta_{R, T}, \Delta_{P, T}$ may be used as inputs by the agent, the variations at a given time $t$, $\Vert \rv_{t+1}-\rv_t\Vert_\infty$ and $\Vert\pv_{t+1}-\pv_t\Vert_\infty$, are unknown.

We denote the long-term average reward obtained by following policy $\polv_t$ in the environment $\mathcal{M}_t$ by 
\begin{equation*}
    J_t^{\polv_t} = \mathbb{E}_{\s\sim d^{\polv_t, \pv_t}(\cdot), \a\sim\pol(\cdot|\s)} \left[\r_t(\s, \a) \right].
\end{equation*}
Further, the state and state-action value functions at time $t$ are solutions to the Bellman equations
\begin{align*}
    \v_t^{\polv_t}(\s) & = \sum_{\a \in \mathcal{A}} \pol(\a|\s)\q_t^{\polv_t}(\s, \a) \\ \q_t^{\polv_t}(\s, \a) & = \r_t(\s, \a) - J_t^{\polv_t} + \sum_{\s' \in \mathcal{S}} \p_t(\s'|\s, \a)\v_t^{\polv_t}(\s').
\end{align*}
The set of solutions to the Bellman equations above is $\qv_t^{\polv_t} = \{\qv_{t, E}^{\polv_t} + c\mathbf{1} | \qv_{t, E}^{\polv_t} \in E, c \in \mathbb{R}\}$ where $E$ is the subspace orthogonal to the all ones vector and $\qv_{t, E}^{\polv_t}$ is the unique solution in $E$ \citep{zhang2021finite}.

\begin{figure*}
\onecolumn
\begin{algorithm}[H]
    \caption{Non-Stationary Natural Actor-Critic (\algoName)} \label{algo}
    \begin{algorithmic}[1]
         \STATE \textbf{Input} time horizon $T$, variation budgets $\Delta_{R, T}$, $\Delta_{P, T}$, projection radius $R_Q$
         \STATE \textbf{Set} step-sizes of actor $\alr$, critic $\clr$, average reward $\rlr$, {\color{black} number of restarts $N$ of length $H=\lfloor \frac{T}{N} \rfloor$, time $t=0$}
         \FOR{$n = 0, 1, 2, \dots, N-1$}
             \STATE {\color{black} Set policy $\pol_{t}(\a|\s) = \frac{1}{\lvert\mathcal{A}\rvert}$, value function $\q_{t}(\s, \a) = 0$ $\forall \s, \a$, average reward estimate $\eta_{t} = 0$} \alglinelabel{line:restart}
             \STATE {\color{black} Sample $\s_{t} \sim \text{Unif}\{|\mathcal{S}|\}$, take action $\a_{t} \sim \pol_{t}(\cdot | \s_{t})$}
             \FOR{$h = 0, 1, 2, \dots, H-1$}
                \STATE Observe reward $\r_{t}(\s_{t}, \a_{t})$, next state $\s_{t+1}\sim\p_{t}(\cdot | \s_{t}, \a_{t})$, and take action $\a_{t+1} \sim \pol_{t}(\cdot | \s_{t+1})$  
                \STATE $\re_{t+1} \leftarrow \re_{t} + \rlr\left(\r_{t}(\s_{t}, \a_{t}) - \re_{t}\right)$ \hfill $\triangleright$ \textit{Average Reward Estimate} \alglinelabel{line:avgRewardUpdate}
                \STATE $\q_{t+1}(\s_{t}, \a_{t}) \leftarrow {\color{black} \Pi_{R_Q}} \left[\q_{t}(\s_{t}, \a_{t}) + \clr\left(\r_{t}(\s_{t}, \a_{t}) - \re_{t} + \q_{t}(\s_{t+1}, \a_{t+1}) - \q_{t}(\s_{t}, \a_{t}) \right)\right]$ \hfill $\triangleright$ \textit{Critic Update} \alglinelabel{line:criticUpdate}
                \STATE $\pol_{t+1}(\a|\s) \leftarrow \frac{\pol_{t}(\a|\s) \exp(\alr\q_{t}(\s,\a))}{\sum_{\a' \in \mathcal{A}} \pol_{t}(\a'|\s) \exp(\alr\q_{t}(\s,\a'))}$, $\forall \s, \a$ \hfill $\triangleright$ \textit{Actor Update} \alglinelabel{line:actorUpdate}
                \STATE $t \leftarrow t+1$
            \ENDFOR
         \ENDFOR
    \end{algorithmic}
\end{algorithm}
\vspace{-20pt}
\twocolumn
\end{figure*}

The goal of the agent is to maximize the time-averaged reward $\sum_{t=0}^{T-1} \r_t(\s_t, \a_t) / T$. We measure performance using an equivalent metric called the \textit{dynamic regret} defined as 
\begin{equation}
    \hbox{Dyn-Reg}(\mathcal{M}, T) := \mathbb{E} \left[\sum_{t=0}^{T-1} J_t^{\polv_t^\star} - \r_t(\s_t, \a_t)\right], \label{eq:DynRegret}
\end{equation}
where $\polv_t^\star = \argmax_{\polv} J_t^\polv$ is the optimal policy in the environment $\mathcal{M}_t = (\mathcal{S}, \mathcal{A}, \pv_t, \rv_t)$ at time $t$. The optimal average reward $J_t^{\polv_t^\star}$ associated with $\polv_t^\star$ can be computed by solving the linear program~(\ref{eqn:optimalJLP}) described in \Cref{app:tLemmas}. 

The model of change and notion of dynamic regret considered here have been used in several recent works \citep{cheung2020reinforcement, fei2020dynamic, zhou2020nonstationary, mao2021nearOptimal, feng2023non}. Note that it is more challenging to analyze than static regret, which compares the cumulative reward collected by an agent against that of a single stationary optimal policy \citep{even2009online, touati2020efficient}. Further, in applications such as robotics or network routing, where the underlying environment evolves over time, a single best action/policy in hindsight might not be a realistic benchmark. On the other hand, dynamic regret provides a more useful performance measure while still being computationally feasible to evaluate.

\newpage
\textbf{Challenges due to Non-Stationarity.}
When running policy-gradient methods in stationary RL, the policy evolves to efficiently learn a fixed environment $(\pv, \rv)$. However, in non-stationary case, the environment $(\pv_t, \rv_t)$ also changes over time. Therefore, the agent chases a moving target, namely, the \textit{time-varying optimal policy} $\polv_t^\star$, resulting in the following unique challenges.
\begin{itemize}[leftmargin=*]
    \item \textit{Explore-for-Change vs Exploit:} The agent needs to explore more aggressively than in the stationary setting to adapt to the changing dynamics. As an example, a sub-optimal action at the current timestep may become optimal at a later timestep, necessitating re-exploration. This is in sharp contrast to stationary RL, where sub optimal actions are picked less often as time progresses.
    \item \textit{Forgetting Old Environments:} The policy and value function estimates must evolve quickly lest they might become irrelevant when the environment changes significantly. However, observations are noisy and the agent needs to collect multiple samples to obtain confident estimates. Hence, an agent has to carefully balance the rate of forgetting the old environment versus learning a new one.
\end{itemize}

%%%%%%%%%%%%%%%%%%%
%%%%%%%%%%%%%%%%%%%

\section{Algorithm: NS-NAC} \label{sec:algorithm}

In this section, we present Non-Stationary Natural Actor-Critic (\algoName), a two-timescale natural policy gradient method with a {\color{black} restart based exploration for change} and step-sizes designed to carefully balance the rate of forgetting the old environment and adapting to a new one. {\color{black} While we use the variation budget $\Delta_{T}$ as an input to \algoName~here, we present a parameter-free algorithm \borlAlgoName~in \Cref{sec:borlAlgorithm} that does not require this knowledge.}

The \algoName~algorithm seeks to maximize the total reward received over the time horizon $T$, given the variation budgets $\Delta_{R, T}$ and $\Delta_{P, T}$. At timestep $t$, $\polv_t$ denotes the tabular policy with $\pol(\cdot|\s)$ where $\pol(\a|\s) \geq 0$ $\forall \a \in \gA$ and $\sum_\a \pol(\a|\s)=1$ $\forall s \in \gS$. $\polv_t^\star = \argmax_\polv J_t^\polv$ is the optimal policy in the environment $\mathcal{M}_t$. The estimate of the tabular state-action value function $\qv_t^{\polv_t}$ is denoted by $\qv_t \in \mathbb{R}^{\lvert\mathcal{S}\rvert\lvert\mathcal{A}\rvert}$. $\re_t$ denotes the estimate of the average reward $J_t^{\polv_t}$.

{\color{black} \algoName~divides the total horizon $T$ into $N$ segments of length $H = \lfloor T/N \rfloor$ each. At the beginning of each segment, the algorithm restarts the NAC sub-routine (line~\ref{line:restart}), thereby ensuring that the algorithm sufficiently \emph{explores for change}. Next, at each time-step $t = nH+h$ $\forall n \in [N], h \in [H]$, the \emph{actor} (slower timescale) takes a natural gradient ascent step towards the optimal policy in environment $\mathcal{M}_t$ as 
\begin{align*}
    \polv_{t+1} \leftarrow \polv_t + \alr F_{\polv_t}^{-1} \mathbb{E}_{\s, \a} \lbr \qv_t^{\polv_t}(\s, \a) \nabla\log\polv_t(\a|\s) \rbr.
\end{align*}
In the absence of knowledge of the exact natural gradient, the actor uses an estimate of the value function to update the policy with tabular softmax parameterization as in line~\ref{line:actorUpdate}.}

The \textit{critic} (faster timescale) estimates the tabular state-action value function of the current policy $\polv_t$ as $\qv_t$ using TD-Learning with step-size $\clr$ (line~\ref{line:criticUpdate}). The projection step in line~\ref{line:criticUpdate} is defined as $\Pi_{R_Q} \left[\mathbf{x}\right] := \argmin_{\norm{\mbf y}_2 \leq R_Q} \Vert \mathbf{x} - \mathbf{y}\Vert_2$ (see \Cref{assumption:maxEigValue} and following discussion on choice of $R_Q$). Further, the average reward estimate $\re_t$ is updated with step-size $\rlr$ (line~\ref{line:avgRewardUpdate}). {\color{black} Using a two timescale technique with $\clr \gg \alr$, \algoName~thus enables the actor to chase the moving target $\polv_t^\star$ facilitated by the critic updates of the value function estimates which adapt to the changed data distribution.} In the stationary RL case, this change in data distribution is induced solely by the evolving actor policy, while in non-stationary RL, the time-varying environment $(\pv_t, \rv_t)$ further exacerbates it. Further, as \Cref{maintheorem:regretUpperBound} suggests, a careful selection of the step-sizes {\color{black} as a function of the variation budgets} enables \algoName~to balance the rate of \emph{forgetting the old environment} versus learning a new one.

{\color{black} \emph{Function Approximation.} While we consider the tabular formulation here for the ease of presentation, \algoName~can also be extended to the function approximation setting. Further details are presented in \Cref{app:functionApproximation}.}

%%%%%%%%%%%%%%%%%%%
%%%%%%%%%%%%%%%%%%%

\section{Regret Analysis: NS-NAC} \label{sec:regretAnalysis}

In this section, we set up notation and assumptions and establish an upper bound on the dynamic regret of \algoName. We further present a sketch of the proof in \Cref{sec:regretAnalysisProofSketch}.

\subsection{Assumptions} \label{sec:regretAnalysisAssumptions}
\textbf{Notation.} We denote an observation $O_t = (\s_t, \a_t, \s_{t+1}, \a_{t+1})$. If $d^{\polv_t, \pv_t}(\cdot)$ is the stationary distribution induced over the states, we define the matrices $\Av(O_t), \Bar{\Av}^{\polv_t, \pv_t} \in \mathbb{R}^{\lvert\mathcal{S}\rvert\lvert\mathcal{A}\rvert \times \lvert\mathcal{S}\rvert\lvert\mathcal{A}\rvert}$ as 

\begin{align*}
    \Av(O_t)_{i,j} 
    & = 
        \begin{cases}
            -1, & \hbox{if } (\s_t, \a_t) \neq (\s_{t+1}, \a_{t+1}), \\ & \quad i=j=(\s_t, \a_t) \\
            1, & \hbox{if } (\s_t, \a_t) \neq (\s_{t+1}, \a_{t+1}), \\ & \quad i=(\s_t, \a_t), j=(\s_{t+1}, \a_{t+1})\\
            0, & \hbox{else}
        \end{cases} \\
        \Bar{\Av}^{\polv_t, \pv_t} & = \mathbb{E}_{\substack{\s\sim d^{\polv_t, \pv_t}(\cdot), \a\sim\polv_t(\cdot|\s),\\ \s'\sim\pv_t(\cdot|\s, \a), \a'\sim\polv_t(\cdot|\s')}}\left[\Av(\s, \a, \s', \a')\right].
\end{align*}
If $D^{\polv_t, \pv_t} = diag\left(d^{\polv_t, \pv_t}(\s)\pol_t(\a|\s)\right)$ and $\bm{1}$ is the all ones vector, then the TD limiting point satisfies 
\begin{equation}
\label{eq:TD-limit-point}
    \mathbf{D^{\polv_t, \pv_t}}\left(\rv_t - J_t^{\polv_t}\bm{1}\right) + \Bar{\Av}^{\polv_t, \pv_t}\qv_t^{\polv_t} = 0.
\end{equation}

\begin{assumption}[Uniform Ergodicity] \label{assumption:ergodic}
     A Markov chain generated by implementing policy $\polv$ and transition probabilities $\pv$ is called uniformly ergodic, if there exists $m > 0$ and $\rho \in (0, 1)$ such that 
    \begin{equation*}
        \dtv\left(\p(\s_\mt \in \cdot | \s_0 = \s), d^{\polv, \pv}\right) \leq m\rho^{\mt} \hbox{ } \forall \mt \geq 0, \s \in \mathcal{S},
    \end{equation*}
    where $d^{\polv, \pv}$ is the stationary distribution induced over the states. We assume Markov chains induced by all potential policies $\polv_t$ in all environments $\pv_t$, $t \in [T]$, are uniformly ergodic. Further, if $\polv_t^\star$ denotes the optimal policy for the environment $\mathcal{M}_t = (\mathcal{S}, \mathcal{A}, \pv_t, \rv_t)$, there exists $C > 0$ such that
     \begin{equation*}
         C = \inf_{\s, t, t', \polv} \dfrac{d^{\polv, \pv_{t'}}(\s)}{d^{\polv_{t}^\star, \pv_t}(\s)} > 0.
     \end{equation*}
\end{assumption}

\begin{lemma}[\citet{zhang2021finite}, Lemma 2] \label{assumption:maxEigValue}
    Under \Cref{assumption:ergodic}, for all potential policies $\polv_t$ in all environments $\pv_t$, $t \in [T]$, the matrix $\Bar{\Av}^{\polv_t, \pv_t}$ is negative semi-definite. Define its maximum non-zero eigenvalue as $-\lambda$.
\end{lemma}

\Cref{assumption:ergodic} is standard in literature \citep{murthy2023convergence, Wu2020A2CPG, zou2019finite}. Also note that we set the projection radius $R_Q = 2U_R\lambda^{-1}$ in line~\ref{line:criticUpdate} of \Cref{algo} because $\Vert\left(\Bar{\Av}^{\polv_t, \pv_t}\right)^{\dagger}\Vert_2 \leq \lambda^{-1}$ where $\dagger$ represents the pseudo-inverse.

%%%%%%%%%%%%%%%%%%%

\subsection{Bounds on Regret} \label{sec:regretAnalysisBounds}

\begin{theorem} \label{maintheorem:regretUpperBound}
    If \Cref{assumption:ergodic} is satisfied and the step-sizes are chosen as $0 < \clr, \alr, \rlr < 1/2$ {\color{black}and number of restarts as $0 < N < T$} in \Cref{algo}, then we have 
    \begin{equation}
        \begin{aligned}
            & \hbox{Dyn-Reg}(\mathcal{M}, T) = \mathbb{E}\left[\sum_{t=0}^{T-1} J_t^{\polv_t^\star} - \r_t(\s_t, \a_t) \right] \\ 
            & \leq \underbrace{\mathcal{\Tilde{O}}\left(\mfrac{N}{\alr} \right) + \mathcal{\Tilde{O}}\left(\sqrt{\mfrac{\color{black} NT}{\clr}} \right)}_{\text{Effect of initialization}} + \underbrace{\mathcal{\Tilde{O}}\left(\mfrac{\alr T}{\clr} \right) + \mathcal{\Tilde{O}}\left(T\sqrt{\alr} \right)}_{\substack{\text{Cumulative change} \\ \text{in policy over horizon } T}} \\
            & \quad + \underbrace{\mathcal{\Tilde{O}}\left(\mfrac{\alr T}{\rlr} \right) + \mathcal{\Tilde{O}}\left(T\sqrt{\rlr} \right) + \mathcal{\Tilde{O}}\left(\sqrt{\mfrac{\color{black} NT}{\rlr}} \right)}_{\text{Error in Average Reward Estimate at Critic}} + \underbrace{\mathcal{\Tilde{O}}\left(T\sqrt{\clr} \right)}_{\substack{\text{Cumulative change} \\ \text{in critic estimates}}} \\ 
            & \quad + \underbrace{\mathcal{\Tilde{O}}\left(\mfrac{\Delta_T T}{N} \right) + \mathcal{\Tilde{O}}\left(\mfrac{\Delta_T^{1/3}T^{2/3}}{\color{black}\sqrt{\clr}} +  \mfrac{\Delta_T^{1/3}T^{2/3}}{\color{black}\sqrt{\rlr}} \right)}_{\text{Error due to Non-Stationarity}},
        \end{aligned}
        \label{eq:thm:DynRegBound1}
    \end{equation}
    where $\Delta_T = \Delta_{R,T} + \Delta_{P,T}$, $\mathcal{\Tilde{O}}(\cdot)$ hides the constants and logarithmic dependence on the time horizon $T$. Choosing optimal $\clr^\star = \rlr^\star = \left(\frac{\Delta_T}{T}\right)^{\color{black}1/3}$, $\alr^\star = \left(\frac{\Delta_T}{T}\right)^{\color{black}1/2}$ and $N^\star = \Delta_T^{\color{black}5/6} T^{\color{black}1/6}$, the resulting regret (with explicit dependence on the size of the state-action space $|\gS|, |\gA|$) is
    \begin{align}
    \label{eq:thm:DynRegBound2}
        \hbox{Dyn-Reg}(\mathcal{M}, T) \leq \mathcal{\Tilde{O}}\left( |\gS|^{1/2}|\gA|^{1/2} {\color{black}\Delta_T^{1/6} T^{5/6}} \right).
    \end{align}
\end{theorem}
We provide a sketch of the proof in \Cref{sec:regretAnalysisProofSketch} and the full proof in \Cref{app:regret}. 

{\color{black} \textbf{Function Approximation.} While we consider the tabular formulation here for the ease of presentation, \algoName~can also be extended to the function approximation setting. Further details of regret bounds are presented in \Cref{app:functionApproximation}.}

\textbf{Effect of Non-Stationarity.} The variation budget $\Delta_T$ (\ref{eq:variationBudget}) represents the extent of non-stationarity of the environment. In \Cref{maintheorem:regretUpperBound}, as the variation budget increases, so do the optimal choice of step-sizes and number of restarts, and the regret incurred (\ref{eq:thm:DynRegBound2}). This observation is consistent with the intuition that in a rapidly changing environment, the algorithm must adapt quickly and explore more (hence, larger step-sizes {\color{black} and more restarts}). However, as a result, the algorithm cannot exploit its current policy and value-function estimates, which soon become outdated (hence, higher regret). Also, in environments with larger state/action spaces, the agent requires proportionately more samples to detect changes and learn a good policy.

\begin{theorem}[\citep{mao2021nearOptimal}, Proposition 1] \label{maintheorem:regretLowerBound}
    For any learning algorithm, there exists a non-stationary MDP such that the dynamic regret of the algorithm is at least $\Omega(\lvert\mathcal{S}\rvert^{1/3}\lvert\mathcal{A}\rvert^{1/3} \Delta_T^{1/3}T^{2/3})$.
\end{theorem}

\textbf{Gap between Bounds.} To the best of our knowledge, this is the first bound on dynamic regret for model-free policy-based algorithm in the infinite horizon average reward setting. Observe that the infinite horizon setting (only one sample per environment available) is harder than the episodic setting (environment remains stationary during the episode) and necessitates a single loop algorithm with the policy being updated at every timestep. We conjecture that the gap between the bounds results from a slack in the analysis of the underlying Natural Actor-Critic (NAC) algorithm. The best-known regret bounds for NAC for an infinite horizon \textit{stationary} MDP in the (compatible) function approximation setting with a two timescale algorithm is $\mathcal{\Tilde{O}}(T^{3/4})$ \citep{khodadadian2022finite}. {\color{black} The analysis of the actor involves the norm of the critic estimation error $\Vert \qv_t - \qv_t^{\polv_t} \Vert$ (\Cref{proposition:actor}) whereas guarantees for critic establish a bound on norm-squared of the error $\Vert \qv_t - \qv_t^{\polv_t} \Vert^2$ (\Cref{proposition:criticEstimate}). This mismatch, which underlies the sub-optimality of the current best stationary infinite horizon NAC analysis, is exacerbated in non-stationary environments resulting in the gap between the upper and the lower bounds. \footnote{The term characterizing the difference in value functions at consecutive timesteps $\Vert \qv_{t+1}^{\polv_{t+1}} - \qv_t^{\polv_t}\Vert$ is the cause for the bottleneck $\Tilde{\mathcal{O}}\left(\Delta_T^{1/3} T^{2/3}\left(\frac{1}{\color{black}\sqrt{\clr}} + \frac{1}{\color{black}\sqrt{\rlr}} \right)\right)$ term (see $I_4, I_5, I_6$ in \Cref{proposition:criticEstimate}).} Note that this mismatch of the value function estimation error between the actor and the critic doesn't occur in the analysis of the model-based methods which use a Hoeffding style high probability bounds.}

%%%%%%%%%%%%%%%%%%%
%%%%%%%%%%%%%%%%%%%

\section{Parameter-Free Algorithm: BORL-NS-NAC} \label{sec:borlAlgorithm}

{\color{black}
In this section, inspired by the bandit-over-RL (BORL) framework \citep{mao2021nearOptimal, cheung2020reinforcement}, we present the \borlAlgoName~algorithm that does not require prior knowledge of the variation budget $\Delta_T$. It works by leveraging the adversarial bandit framework to tune the variation budget dependent parameters, i.e step-sizes and number of restarts, in \algoName~and hedge against changes in rewards and transition probabilities. \borlAlgoName~runs the EXP3.P algorithm \citep{bubeck2012regret} over $\lceil T/W \rceil$ epochs with \algoName~(\Cref{algo}) as a sub-routine in each epoch. In each epoch, an arm of the bandit is pulled to choose the parameters of the sub-routine and the cumulative rewards received during the epoch are used to update the posterior. Due to paucity of space, we defer the details of the algorithm, pseudocode and analysis to \Cref{app:borl} and present the regret bound below.

\begin{theorem}
    If \Cref{assumption:ergodic} is satisfied and the time horizon $T$ is divided into epochs of length $W = \mco(T^{2/3})$ in \borlAlgoName~(\Cref{borlAlgo}), then
    \begin{align*}
        \hbox{Dyn-Reg}(\mathcal{M}, T) \leq \Tilde{\mco} \left(|\gS|^{1/2} |\gA|^{1/2} \Delta_T^{1/6} T^{5/6} \right).
    \end{align*}
\end{theorem}

}

%%%%%%%%%%%%%%%%%%%
%%%%%%%%%%%%%%%%%%%

\section{Proof Sketch of \Cref{maintheorem:regretUpperBound}} \label{sec:regretAnalysisProofSketch}

We now present a sketch of the proof of \Cref{maintheorem:regretUpperBound} that presents an upper bound on regret of \algoName~and address the following theoretical challenges that non-stationarity. (a) Stationary environment NAC analyses use the KL-divergence to the optimal policy as a Lyapunov function. What is an appropriate function for dynamic environments where the optimal policy varies with time? (b) How do the simultaneously varying environment and evolving policy affect the estimation of the average reward and state-action value function? (c) How do the time-varying transition probabilities affect the martingale-based argument used to analyze the Markovian noise?

\paragraph{Regret Decomposition.} We start by decomposing as 
\begin{align} \label{eqn:regretDecomposition}
    & \hbox{Dyn-Reg}(\mathcal{M}, T) \\ & = \sum_{t=0}^{T-1} \underbrace{\mathbb{E}\left[J_t^{\polv_t^\star} - J_t^{\polv_t} \right]}_{I_1: \substack{\text{Difference of optimal versus} \\ \text{actual average reward}}} + \underbrace{\mathbb{E}\left[J_t^{\polv_t} - \r_t(\s_t, \a_t)\right]}_{I_2: \substack{\text{Difference of actual versus} \nn \\ \text{instantaneous reward}}},
\end{align}
where $I_1$ measures the performance difference between the average reward of the actual policy $\polv_t$ at time $t$ relative to the optimal policy $\polv_t^\star$. The second term $I_2$ analyzes the gap between the average reward and the actual rewards received due to the stochasticity of the Markovian sampling process.

\paragraph{Actor (\Cref{proposition:actor}).}
We first bound $I_1$ in (\ref{eqn:regretDecomposition}) by adapting the Natural Policy Gradient analysis for average-reward stationary MDPs in \citet{murthy2023convergence} to non-stationary environments. {\color{black} NPG in the stationary case is analyzed by characterizing the drift of the policy towards the optimal policy using an appropriate Lyapunov function. In non-stationary case we innovatively separate out and analyze the change in the environment from the drift of the policy as follows.} {\color{black} We start by dividing the total horizon $T$ into $N$ restarted segments of length $H$ each} and split $I_1$ as
\begin{align*}
    I_1 & = \mathbb{E} \Bigg[ \sum_{n=0}^{N-1}\sum_{h=0}^{H - 1} \underbrace{\left(J_{nH+h}^{\polv_{nH+h}^\star} - J_{nH}^{\polv_{nH}^\star} \right)}_{I_3: \substack{\text{Optimal avg. reward} \\ \text{across two environments}}} \\ & \quad + \underbrace{\left(J_{nH}^{\polv_{nH}^\star} - J_{nH}^{\polv_{nH+h}} \right)}_{I_4: \substack{\text{Avg. reward} \\ \text{sub-optimality}}} + \underbrace{\left(J_{nH}^{\polv_{nH+h}} - J_{nH+h}^{\polv_{nH+h}}\right)}_{I_5: \substack{\text{Avg. reward with same} \\ \text{policy in two environments}}} \Bigg].
\end{align*}
We benchmark policies learned in each segment $n \in [N]$ against the optimal average reward at the initial time step $nH$ i.e. $J_{nH}^{\polv_{nH}^\star}$. We bound $I_4$ by mirror descent style analysis for each segment $n$ {\color{black} with $t = \{nH, \dots, (n+1)H-1\}$} by the Lyapunov function adapted to non-stationarity as
\begin{align*}
    W(\polv_{t}) & = \sum_{\s}d^{\polv_{nH}^\star, \pv_{nH}}(\s)\KL(\polv_{nH}^\star(\cdot|\s) \Vert \polv_{t}(\cdot|\s)).
\end{align*}
In addition, since \algoName~ does not have access to the exact value functions $\qv_{t}^{\polv_{t}}$, $I_4$ also depends on the critic estimation error $\| \qv_{t}^{\polv_{t}} - \qv_{t} \|_\infty$.

{\color{black} We analyze the change in the environment next.} We bound $I_3$, the difference in the optimal average rewards in two different environments, in terms of the corresponding changes in the environment $\|\rv_{nH+h}-\rv_{nH}\|_\infty$ and $\|\pv_{nH+h}-\pv_{nH}\|_\infty$ (\Cref{tlemma:bestAvgRewardDiff}) by a clever use of the linear programming formulation of an MDP. Similarly, we deftly bound $I_5$, the difference in average rewards when following the same policy $\polv_{nH+h}$ in two different environments, in terms of the change in the environment (\Cref{tlemma:avgRewardLipschitz}). {\color{black} Note that the number of restarts $N$ balances exploration-for-change and learning a good policy and we optimize it in \Cref{maintheorem:regretUpperBound} to minimize regret.}

\begin{figure*}[ht]
% \vskip 0.2in
\begin{center}
    % Subfigure 1
    \subfigure[$|\gS|=50$, $|\gA| = 4$, $\Delta_T \sim 300$]{
        \includegraphics[width=0.31\textwidth]{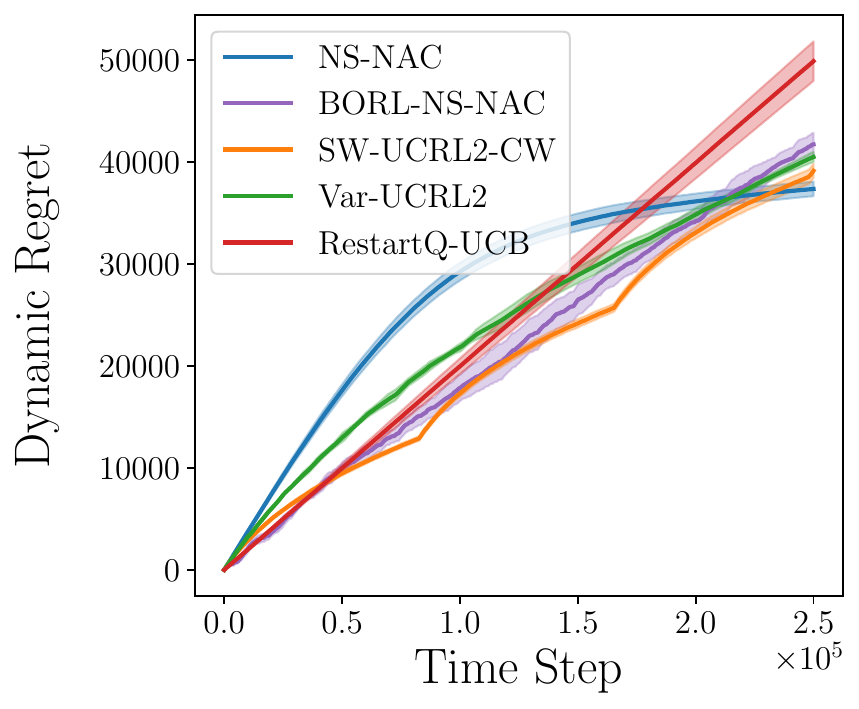}
        \label{fig:subfig1}
    }\hfill
    % Subfigure 2
    \subfigure[$|\gS| = 50$, $|\gA| = 4$, $\Delta_T \sim 300$]{
        \includegraphics[width=0.31\textwidth]{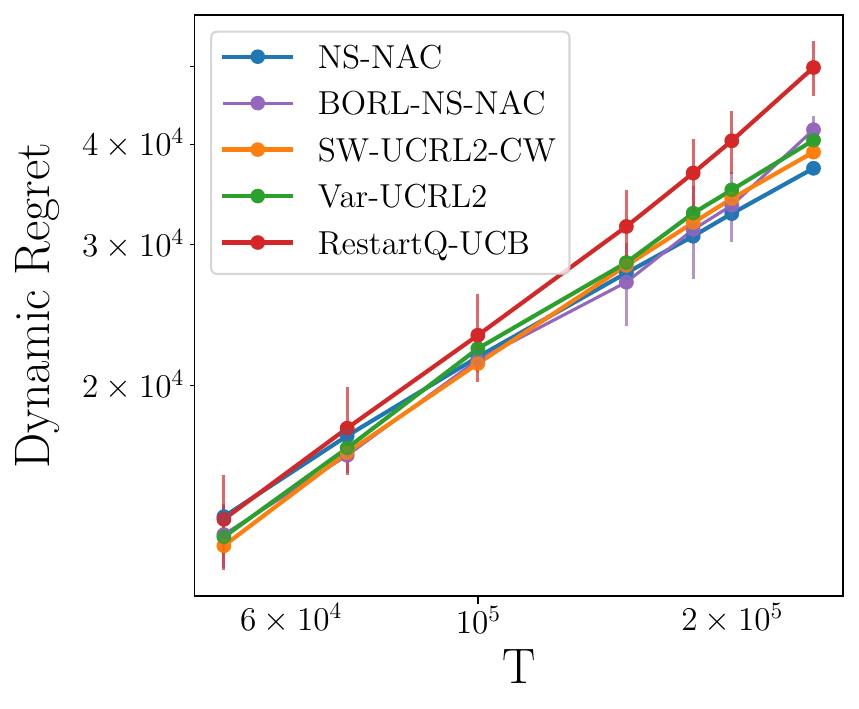}
        \label{fig:subfig2}
    }\hfill
    % Subfigure 3
    \subfigure[$|\gS| = 50$, $|\gA| = 4$, $T = 5\times10^3$]{
        \includegraphics[width=0.31\textwidth]{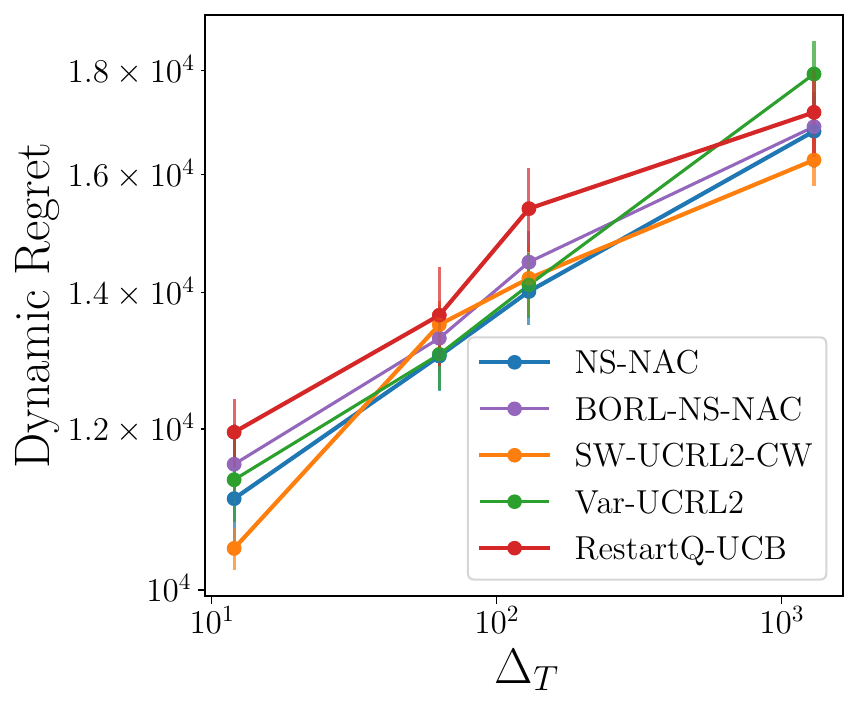}
        \label{fig:subfig3}
    }

    \caption{Performance of NS-NAC and baseline algorithms across various settings. (a) Dynamic regret for a single instance with $T = 25\times10^4$ steps. Log-log plots showing the effect of varying: (b) time horizon $T$, and (c) variation budget $\Delta_T$.}
    \label{fig:two_column_figure}
\end{center}
\vskip -0.2in
\end{figure*}

\paragraph{Critic (\Cref{proposition:criticEstimate}).}
We bound the critic estimation error $\qdv_t = {\color{black}\Pi_E}\left[\qv_t - \qv_t^{\polv_t}\right]$ \footnote{{\color{black} $\Pi_E[\mathbf{x}] = \argmin_{\mathbf{y} \in E} \Vert \mathbf{x} - \mathbf{y}\Vert_2$ is the projection to $E$, the subspace orthogonal to the all ones vector $\mathbf{1}$.}} {\color{black} for each restarted segment $n\in [N]$} {\color{black} where $t = \{nH, \dots, (n+1)H-1\}$} by adapting the critic analysis used in stationary MDPs \citep{Wu2020A2CPG, khodadadian2022finite, zhang2021finite} to non-stationary environments. We decompose the error as 
\begin{align}
    & \Vert\qdv_{t+1}\Vert_2^2 \lesssim (1-\alpha)\Vert\qdv_t \Vert_2^2 + \clr \underbrace{\Gamma(\polv_t, \pv_t, \rv_t, \qdv_t, O_t)}_{I_6: \substack{\text{Error due to Markov noise}}} \nn \\ 
    & + \clr \underbrace{(\Jv_t^{\polv_t}(O_t) - \rev_t(O_t))^2}_{I_7: \substack{\text{Avg. reward estimation error}}} + \mfrac{1}{\clr} \underbrace{\Vert{\color{black}\Pi_E}\left[\qv_t^{\polv_t} - \qv_{t+1}^{\polv_{t+1}}\right] \Vert_2^2}_{I_8: \substack{\text{Value function drift}}} \nn \\ & + \underbrace{\clr^2\Vert \rv_t(O_t) - \rev_t(O_t) + \Av(O_t)\qv_t \Vert_2^2}_{I_9: \substack{\text{Variance term}}} \label{eq:critic_error}
\end{align}
where $\Gamma_t = \qdv_t^\top \left(\rv_t(O_t) - \Jv_{t}^{\polv_t}(O_t)+\Av(O_t)\qv_{t}^{\polv_t}\right) + \qdv_t^\top\left(\Av(O_t) - \Bar{\Av}^{\polv_t, \pv_t}\right)\qdv_t$. $I_6$ is the error induced by the Markovian noise which is analyzed leveraging the auxiliary Markov chain described below. $I_7$ describes the error due to an inaccurate estimation of the average reward which is bounded below. $I_8$, the change in the true value function is caused by drifting policies and environments, and can be neatly bounded in terms of the change in policy, rewards and transition probabilities (\Cref{tlemma:consecutiveTimeQDiff}). Finally, $I_{10}$ is the variance term.

\paragraph{Bound on Markovian Noise.} 
{\color{black} For each restarted segment $n \in [N]$, consider time indices $(n+1)H > t > \mt > nH$.} Consider the \textit{auxiliary Markov chain} starting from $\s_{t-\mt}$ constructed by conditioning on $\mathcal{F}_{t-\mt} = \{\s_{t-\mt}, \polv_{t-\mt-1}, \pv_{t-\mt}\}$ and rolling out by applying $\polv_{t-\mt-1}, \pv_{t-\mt}$ as 
\begin{align*}
    \s_{t-\mt} & \xrightarrow{\polv_{t-\mt-1}} \a_{t-\mt} \xrightarrow{\pv_{t-\mt}} \Tilde{\s}_{t-\mt+1} \xrightarrow{\polv_{t-\mt-1}} \Tilde{\a}_{t-\mt+1} \overset{\dots}{\rightarrow} \\
    % \xrightarrow{\pv_{t-\mt}} 
    & \quad \Tilde{\s}_t \xrightarrow{\polv_{t-\mt-1}} \Tilde{\a}_t \xrightarrow{\pv_{t-\mt}} \Tilde{\s}_{t+1} \xrightarrow{\polv_{t-\mt-1}} \Tilde{\a}_{t+1}. \nonumber
\end{align*}
Recall that the \textit{original Markov chain} is
\begin{align*}
    \s_{t-\mt} & \xrightarrow{\polv_{t-\mt-1}} \a_{t-\mt} \xrightarrow{\pv_{t-\mt}} \s_{t-\mt+1} \xrightarrow{\polv_{t-\mt}} \a_{t-\mt+1} \overset{\dots}{\rightarrow} \\
    % \xrightarrow{\pv_{t-\mt+1}} 
    & \quad \s_t \xrightarrow{\polv_{t-1}} \a_t \xrightarrow{\pv_{t}} \s_{t+1} \xrightarrow{\polv_{t}} \a_{t+1}. \nonumber
\end{align*}
This method enables us to characterize properties of the original Markov chain in comparison to the auxiliary chain as $\dtv(\p(O_t \in \cdot | \mathcal{F}_{t-\mt}), \p(\Tilde{O}_t \in \cdot | \mathcal{F}_{t-\mt}))$. We do this by bounding the effects of drifting policies and transition probabilities in the original chain and leveraging uniform ergodicity in the auxiliary chain. While prior works use auxiliary Markov chains for stationary environments \citep{zou2019finite, Wu2020A2CPG, wang2024NonAsymptotic}, ours is the first adaptation to a non-stationary environment. Observe that the time-varying transition probabilities $\pv_t$ add an extra layer of complexity, unlike the stationary case where only the policy changes over time. 

\paragraph{Average Reward Estimation Error (\Cref{proposition:averageRewardEstimate}).} To bound
$I_7$ in (\ref{eq:critic_error}), i.e., the error in the average reward estimate $\phi_t = \re_t - J_t^{\polv_t}$, we can decompose the error as
\begin{align*}
    \phi_{t+1}^2 & \lesssim (1-\rlr)\phi_t^2 + \underbrace{\rlr (\r_t(O_t) - J_t^{\polv_t})^2}_{I_{10}: \substack{\text{Error due to Markov noise}}} \nn \\
    & \quad + \mfrac{1}{\rlr} \underbrace{(J_t^{\polv_t} - J_{t+1}^{\polv_{t+1}})^2}_{I_{11}: \substack{\text{Avg reward at consecutive} \\ \text{time steps}}} + \underbrace{\rlr^2(\r_t(O_t) - \re_t)^2}_{I_{12}: \substack{\text{Variance term}}}. \label{eq:avg_reward_error}
\end{align*}
$I_{10}$ is analyzed using the auxiliary Markov chain construction. $I_{11}$ quantifies the difference in average rewards at consecutive timesteps, and is neatly bounded in \Cref{tlemma:avgRewardLipschitz} in terms of the corresponding changes in policies, rewards, and transition probabilities. $I_{12}$ is again the variance term.

Finally, $I_2$ in (\ref{eqn:regretDecomposition}) characterizes the difference between the average reward and the instantaneous reward at any time, and is analyzed in \Cref{proposition:sampledRewardRegret} using the auxiliary Markov chain to bound the bias occurring due to Markovian sampling. This concludes the proof sketch.

%%%%%%%%%%%%%%%%%%%
%%%%%%%%%%%%%%%%%%%

\section{Simulations} \label{sec:experiments} 

We empirically evaluate the performance of our algorithms on a synthetic non-stationary MDP (see \Cref{sec:sim_setup}), comparing it with three baseline algorithms: SW-UCRL2-CW \cite{cheung2023nonstationary}, Var-UCRL2 \cite{ortner2020variational}, and RestartQ-UCB \cite{mao2021nearOptimal}. SW-UCRL2-CW is a model-based algorithm that adapts to non-stationarity by maintaining a sliding window of recent observations, applying extended value iteration, and adjusting confidence intervals to track changing dynamics. Var-UCRL2, also model-based, adjusts its confidence intervals dynamically based on the observed variations in rewards and transitions. RestartQ-UCB, a model-free approach, periodically restarts Q-learning and resets its upper confidence bounds to adapt to non-stationarity. While there is a gap between our theoretical analysis of regret and those of the baseline methods, we empirically observe in \Cref{fig:two_column_figure} that \algoName~and \borlAlgoName~strongly match their performance achieving sub-linear dynamic regret across all experimental settings.

\nocite{nachum2017bridging}

%%%%%%%%%%%%%%%%%%%
%%%%%%%%%%%%%%%%%%%

\section{Conclusion} \label{sec:conclusions} 
We consider the problem of non-stationary reinforcement learning in the infinite-horizon average-reward setting and model it as an MDP with time-varying rewards and transition probabilities. We propose and analyze the first model-free policy-based algorithm, Non-Stationary Natural Actor-Critic. A two-timescale natural policy gradient based method, \algoName~utilizes restarts to explore for change and learning rates as adapting factors to balance forgetting old and learning new environments. Further, we present a bandit-over-RL based parameter-free algorithm \borlAlgoName~that does not require prior knowledge of the variation budget and adaptively tunes step-sizes and number of restarts. Both algorithms achieve a sub-linear dynamic regret, thus, theoretically validating policy gradient methods often used in practice in continual non-stationary RL.

%%%%%%%%%%%%%%%%%%%
%%%%%%%%%%%%%%%%%%%

% \section*{Acknowledgements}

% \textbf{Do not} include acknowledgements in the initial version of
% the paper submitted for blind review.

% If a paper is accepted, the final camera-ready version can (and
% usually should) include acknowledgements.  Such acknowledgements
% should be placed at the end of the section, in an unnumbered section
% that does not count towards the paper page limit. Typically, this will 
% include thanks to reviewers who gave useful comments, to colleagues 
% who contributed to the ideas, and to funding agencies and corporate 
% sponsors that provided financial support.

\section*{Impact Statement}

This paper presents work whose goal is to advance the field of Machine Learning. There are many potential societal consequences of our work, none which we feel must be specifically highlighted here.

% Authors are \textbf{required} to include a statement of the potential 
% broader impact of their work, including its ethical aspects and future 
% societal consequences. This statement should be in an unnumbered 
% section at the end of the paper (co-located with Acknowledgements -- 
% the two may appear in either order, but both must be before References), 
% and does not count toward the paper page limit. In many cases, where 
% the ethical impacts and expected societal implications are those that 
% are well established when advancing the field of Machine Learning, 
% substantial discussion is not required, and a simple statement such 
% as the following will suffice:

% ``This paper presents work whose goal is to advance the field of 
% Machine Learning. There are many potential societal consequences 
% of our work, none which we feel must be specifically highlighted here.''

% The above statement can be used verbatim in such cases, but we 
% encourage authors to think about whether there is content which does 
% warrant further discussion, as this statement will be apparent if the 
% paper is later flagged for ethics review.

% In the unusual situation where you want a paper to appear in the
% references without citing it in the main text, use \nocite
% \nocite{langley00}

\bibliography{references}
\bibliographystyle{icml2025}

%%%%%%%%%%%%%%%%%%%%%%%%%%%%%%%%%%%%%%%%%%%%%%%%%%%%%%%%%%%%%%%%%%%%%%%%%%%%%%%
%%%%%%%%%%%%%%%%%%%%%%%%%%%%%%%%%%%%%%%%%%%%%%%%%%%%%%%%%%%%%%%%%%%%%%%%%%%%%%%
% APPENDIX
%%%%%%%%%%%%%%%%%%%%%%%%%%%%%%%%%%%%%%%%%%%%%%%%%%%%%%%%%%%%%%%%%%%%%%%%%%%%%%%
%%%%%%%%%%%%%%%%%%%%%%%%%%%%%%%%%%%%%%%%%%%%%%%%%%%%%%%%%%%%%%%%%%%%%%%%%%%%%%%
\newpage
\onecolumn
\appendix

\resumetoc

\tableofcontents

% \section*{Appendices}
% \addcontentsline{toc}{section}{Appendices}

% \localtableofcontents
% Set up the local TOC style
% \mtcsetfeature{minitoc}{open}{\section*{Appendix Contents}}
% \mtcsetfeature{minitoc}{close}{}

% % Generate the local TOC
% \minitoc

% \section*{Appendix Contents}
% \localtableofcontents

\newpage
\section{Additional Related Work} \label{app:additionalRelatedWork}

\begin{table}[H]
\centering
\begin{tabular}{|c|c|c|c|c|}
\hline
Setting & Algorithm  & Regret & \begin{tabular}[c]{@{}c@{}}Model \\ Free \end{tabular} & \begin{tabular}[c]{@{}c@{}}Policy \\ Based \end{tabular} \\ \hline
% \multirow{7}{*}{\begin{tabular}[c]{@{}c@{}}Non-Stationary\\ Infinite Horizon\\ Average Reward\end{tabular}} 
    & Lower Bound  & $\Omega\left(|\gS|^{\frac{1}{3}}|\gA|^{\frac{1}{3}}D^{\frac{2}{3}}\Delta_T^{\frac{1}{3}}T^{\frac{2}{3}}\right)$ & -  & -  \\
    & \citet{jaksch2010near} & $\mathcal{\Tilde{O}}\left(|\gS||\gA|^{\frac{1}{2}}DL^{\frac{1}{3}}T^{\frac{2}{3}}\right)$  &  $\times$ & - \\
    Non-Stationary & \citet{gajane2018sliding}  & $\mathcal{\Tilde{O}}\left(|\gS|^{\frac{2}{3}}|\gA|^{\frac{1}{3}}D^{\frac{2}{3}}L^{\frac{1}{3}}T^{\frac{2}{3}}\right)$ & $\times$ & - \\
    Infinite Horizon & \citet{ortner2020variational} & $\mathcal{\Tilde{O}}\left(|\gS||\gA|^{\frac{1}{2}}D\Delta_T^{\frac{1}{3}}T^{\frac{2}{3}}\right)$  &  $\times$  &  -  \\
    Average Reward & \citet{cheung2020reinforcement}  & $\mathcal{\Tilde{O}}\left(|\gS|^{\frac{2}{3}}|\gA|^{\frac{1}{2}}D\Delta_T^{\frac{1}{4}}T^{\frac{3}{4}}\right)$ & $\times$  & - \\
    & \citet{wei2021non}  & $\mathcal{\Tilde{O}}\left(\Delta_T^{\frac{1}{3}}T^{\frac{2}{3}}\right)$ &  $\times$  &  -  \\
    & \cellcolor{lightgray} This Work   & \cellcolor{lightgray} $\mathcal{\Tilde{O}}\left(|\gS|^{\frac{1}{2}}|\gA|^{\frac{1}{2}}\Delta_T^{\frac{1}{9}}T^{\frac{8}{9}}\right)$  & \cellcolor{lightgray} \checkmark  & \cellcolor{lightgray} \checkmark \\ \hline
    % & This Work   & $\mathcal{\Tilde{O}}\left(|\gS|^{\frac{1}{2}}|\gA|^{\frac{1}{2}}\Delta_T^{\frac{1}{9}}T^{\frac{8}{9}}\right)$  & \checkmark  & \checkmark \\ \hline
% \multirow{6}{*}{\begin{tabular}[c]{@{}c@{}}Non-Stationary\\ Episodic\end{tabular}}                     
    & Lower Bound  &  $\Omega\left(|\gS|^{\frac{1}{3}}|\gA|^{\frac{1}{3}}\Delta_T^{\frac{1}{3}}H^{\frac{2}{3}}T^{\frac{2}{3}}\right)$  & - &  -  \\
    Non-Stationary & \citet{domingues2021kernel} & $\mathcal{\Tilde{O}} \left(|\gS||\gA|^{\frac{1}{2}} \Delta_T^{\frac{1}{3}} H^{\frac{4}{3}} T^{\frac{2}{3}} \right)$ & \checkmark           &     $\times$         \\
    Episodic & \citet{wei2021non}   &  $\mathcal{\Tilde{O}}\left(\Delta_T^{\frac{1}{3}} T^{\frac{2}{3}} \right)$      &   \checkmark &  $\times$            \\
    & \citet{feng2023non}  &  $\mathcal{\Tilde{O}}\left(\Tilde{d}^{\frac{1}{2}}H^2T^{\frac{1}{2}}\right)$      & \checkmark   &   $\times$           \\
    & \citet{mao2021nearOptimal}  &  $\mathcal{\Tilde{O}}\left(|\gS|^{\frac{1}{3}}|\gA|^{\frac{1}{3}}\Delta_T^{\frac{1}{3}}HT^{\frac{2}{3}}\right)$  & \checkmark   &   $\times$   \\ \hline
% \multirow{2}{*}{\begin{tabular}[c]{@{}c@{}}Non-Stationary\\ Episodic Linear MDP \end{tabular}}         
    Non-Stationary & \citet{zhou2020nonstationary}  &  $\mathcal{\Tilde{O}}\left(d^{\frac{4}{3}}\Delta_T^{\frac{1}{3}}H^{\frac{4}{3}}T^{\frac{2}{3}}\right)$      &   \checkmark         &  $\times$            \\
    Episodic Linear MDP & \citet{touati2020efficient}  &  $\mathcal{\Tilde{O}}\left(d^{\frac{5}{4}}\Delta_T^{\frac{1}{4}}H^{\frac{5}{4}}T^{\frac{3}{4}}\right)$  &  \checkmark  &  $\times$  \\ \hline
\begin{tabular}[c]{@{}c@{}}Stationary \\ Infinite Horizon\\ Discounted Reward\end{tabular}                  & \citet{khodadadian2022finite}  &  $\mathcal{\Tilde{O}}\left(T^{\frac{5}{6}} \right)$ & \checkmark & \checkmark \\ \hline
    \begin{tabular}[c]{@{}c@{}}Stationary\\ Infinite Horizon\\ Average Reward\end{tabular}                      & \citet{wang2024NonAsymptotic}  &  $\mathcal{\Tilde{O}}\left(T^{\frac{2}{3}} \right)$  &  \checkmark & \checkmark \\ \hline
\end{tabular}
\label{table:regretComparision}
\caption{Regret comparison across Non-Stationary and Stationary RL algorithms with variation budget $\Delta_T$, time horizon $T$, episode length $H$, size of the state-action space $|\gS|$, $|\gA|$, maximum diameter of MDP $D$, dimension of feature space $d$ and dynamic Bellman Eluder dimension $\Tilde{d}$.}
\end{table}

%%%%%%%%%%%%%%%%%%%
%%%%%%%%%%%%%%%%%%%

\newpage
\section{Notation} \label{app:notation}

\paragraph{Variation Budgets}
\begin{align*}
    \Delta_{R, T} &= \sum_{t=0}^{T-1}\Vert\rv_{t+1}-\rv_t\Vert_\infty; \Delta_{R, t-\mt+1, t}  = \sum_{i=t-\mt+1}^t \Vert\rv_{i} - \rv_{i-1}\Vert_\infty, \\
    \Delta_{P, T} &= \sum_{t=0}^{T-1}\Vert\pv_{t+1}-\pv_t\Vert_\infty; \Delta_{P, t-\mt+1, t}  = \sum_{i=t-\mt+1}^t \Vert\pv_{i} - \pv_{i-1}\Vert_\infty, \\
    \Delta_T & = \Delta_{R,T} + \Delta_{P,T}.
\end{align*}
The critic update (line \ref{line:criticUpdate} in \Cref{algo}) can be defined in vector form using the following notation. Note that we use a one-to-one mapping $f: \gS \times \gA \to \{1,2,\dots,|\gS||\gA|\}$, to map state-action pairs $(s,a) \in \gS \times \gA$ to vector/matrix entries. However, for ease of notation, we denote the index of each entry by $(s,a)$, instead of the more accurate $f(s,a)$.
\begin{align*}
    O_{t} & = (\s_t, \a_t, \s_{t+1}, \a_{t+1}) \\
    \rv_t(O_t) & = \left[0; \cdots; 0; \r_t(\s_t, \a_t); 0; \cdots; 0\right]^\top \in \mathbb{R}^{\lvert\mathcal{S}\rvert\lvert\mathcal{A}\rvert}\\
    \re_t(O_t) & = \left[0; \cdots; 0; \re_t; 0; \cdots; 0\right]^\top \in \mathbb{R}^{\lvert\mathcal{S}\rvert\lvert\mathcal{A}\rvert} \\
    \Jv_t^\polv(O_t) & = \left[0; \cdots; 0; J_t^\polv; 0; \cdots; 0\right]^\top \in \mathbb{R}^{\lvert\mathcal{S}\rvert\lvert\mathcal{A}\rvert} \\
    \Av(O) & \in \mathbb{R}^{\lvert\mathcal{S}\rvert\lvert\mathcal{A}\rvert \times \lvert\mathcal{S}\rvert\lvert\mathcal{A}\rvert} \quad \text{such that} \\
    \Av(O)_{i,j} & = \Av(\s, \a, \s', \a')_{i, j} = 
    \begin{cases}
        -1 & \hbox{if } (\s, \a) \neq (\s', \a'), i=j=(\s, \a) \\
        1 & \hbox{if } (\s, \a) \neq (\s', \a'), i=(\s, \a), j=(\s', \a')\\
        0 & \hbox{else}
    \end{cases}
\end{align*}
As a result, we get the critic update
\begin{align*}
    \qv_{t+1} = \prod_{R_Q} \lbr \qv_t + \clr \lpar \rv_t(O_t) - \re_t(O_t) + \Av(O_t) \qv_t \rpar \rbr.
\end{align*}
For the purpose of analysis, we define the following quantities.
\begin{align*} 
    \Bar{\Av}^{\polv, \pv} & = \mathbb{E}_{\s\sim d^{\polv, \pv}(\cdot), \a\sim\polv(\cdot|\s), \s'\sim\pv(\cdot|\s, \a), \a'\sim\polv(\cdot|\s')}\left[\Av(\s, \a, \s', \a')\right]\\
    \qv^{\polv, \pv, \rv} & = \qv \quad \hbox{associated with} \quad \polv, \pv, \rv \\
    J^{\polv, \pv, \rv} & = \sum_{\s}d^{\polv, \pv}(\s)\sum_{\a}\polv(\a|\s) \r(\s, \a) \\
    {\color{black} \Pi_E [\mathbf{x}]} & = \argmin_{\mathbf{y} \in E} \Vert \mathbf{x} - \mathbf{y} \Vert_2 \text{ where $E$ is the subspace orthogonal to the all ones vector } \mathbf{1} \\
    \qdv_t & = {\color{black}\Pi_E}\left[\qv_t - \qv_t^{\polv_t}\right] \tag{Error in the value-function estimate} \\
    \Gamma(\polv, \pv, \rv, \qdv, O) & = \qdv^\top \left(\rv(O) - \Jv^{\polv, \pv, \rv}(O)+\Av(O)\qv^{\polv, \pv, \rv}\right) + \qdv^\top \left(\Av(O) - \Bar{\Av}^{\polv, \pv}\right)\qdv \\
    \phi_t & = \re_t - J_t^{\polv_t} \tag{Error in the average reward estimate} \\
    \Lambda(\polv, \pv, \rv, \re, O) & = (\re - J^{\polv, \pv, \rv}) (\r(\s, \a) - J^{\polv, \pv, \rv})
\end{align*}
Given time indices $t > \mt > 0$, consider the \textit{auxiliary Markov chain} starting from $\s_{t-\mt}$ constructed by conditioning on $\s_{t-\mt}, \polv_{t-\mt-1}, \pv_{t-\mt}$ and rolling out by applying $\polv_{t-\mt-1}, \pv_{t-\mt}$ as 
\begin{equation}
    \s_{t-\mt} \xrightarrow{\polv_{t-\mt-1}} \a_{t-\mt} \xrightarrow{\pv_{t-\mt}} \Tilde{\s}_{t-\mt+1} \xrightarrow{\polv_{t-\mt-1}} \Tilde{\a}_{t-\mt+1} \xrightarrow{\pv_{t-\mt}} \dots \Tilde{\s}_t \xrightarrow{\polv_{t-\mt-1}} \Tilde{\a}_t \xrightarrow{\pv_{t-\mt}} \Tilde{\s}_{t+1} \xrightarrow{\polv_{t-\mt-1}} \Tilde{\a}_{t+1}. \nonumber
\end{equation}

Recall that the \textit{original Markov chain} is
\begin{equation}
    \s_{t-\mt} \xrightarrow{\polv_{t-\mt-1}} \a_{t-\mt} \xrightarrow{\pv_{t-\mt}} \s_{t-\mt+1} \xrightarrow{\polv_{t-\mt}} \a_{t-\mt+1} \xrightarrow{\pv_{t-\mt+1}} \dots \s_t \xrightarrow{\polv_{t-1}} \a_t \xrightarrow{\pv_{t}} \s_{t+1} \xrightarrow{\polv_{t}} \a_{t+1}. \nonumber
\end{equation}

%%%%%%%%%%%%%%%%%%%
%%%%%%%%%%%%%%%%%%%

\newpage
\section{Symbol Reference} \label{app:symbolReference}

\begin{table}[H]
\label{tab:constantReference}
    \centering
    \begin{tabular}{c|c}
         Constant & First Appearance \\ \hline
         $U_R$ & \Cref{sec:problemSettingPreliminaries} \\
         $U_Q$ & \Cref{plemma:maxValues} \\
         {\color{black} $N, H$} & \Cref{algo} \\
         $C = \inf\limits_{\s, t, t', \polv} \frac{d^{\polv, \pv_{t'}}(\s)}{d^{\polv_{t}^\star, \pv_t}(\s)}$ & \Cref{assumption:ergodic} \\
         $m, \rho$ & \Cref{assumption:ergodic} \\
         $M = \ceil{\log_\rho m^{-1}} + \frac{1}{1-\rho}$ & \Cref{plemma:dTV} \\
         $\lambda$ & \Cref{assumption:maxEigValue} \\
         $W_1 = (3G_R^2)^{1/3}(4U_Q^2)^{2/3}$ & \Cref{proposition:criticEstimate} \\
         $W_2 = (3G_P^2)^{1/3}(4U_Q^2)^{2/3}$ & \Cref{proposition:criticEstimate} \\
         $D_1 = L_\polv B_2 + 4U_R\sqrt{|\gS||\gA|} B_2 + 4U_R$ & \Cref{proposition:sampledRewardRegret} \\
         $D_2 = 4U_R+L_P$ & \Cref{proposition:sampledRewardRegret} \\
         $D_3 = 4U_R F_6+8U_R^2$ & \Cref{proposition:averageRewardEstimate} \\
         $D_4 = 9L_\polv^2 B_2^2$ & \Cref{proposition:averageRewardEstimate} \\
         $W_3 = (3)^{1/3}(4U_R^2)^{2/3}$ & \Cref{proposition:averageRewardEstimate} \\
         $W_4 = (3L_P^2)^{1/3}(4U_R^2)^{2/3}$ & \Cref{proposition:averageRewardEstimate} \\
         $B_1 = 2\sqrt{\lvert\mathcal{A}\rvert}U_Q^2$ & \Cref{tlemma:actor} \\
         $B_2 = U_Q$ & \Cref{tlemma:consecutiveTimePolicyDiff} \\
         $B_3 = (F_{1\polv} + G_\polv + F_3\sqrt{|\gS||\gA|} + F_4)B_2$ & \Cref{tlemma:criticGamma} \\
         $B_4 = F_2(2U_R + 2U_Q)$ & \Cref{tlemma:criticGamma} \\
         $B_5 = F_2 G_R$ & \Cref{tlemma:criticGamma} \\
         $B_6 = F_{1\pv} + F_2G_P + F_3$ & \Cref{tlemma:criticGamma}\\
         $B_7 = (F_5L_\polv + F_7 \sqrt{|\gS||\gA|} + F_8)B_2$ & \Cref{tlemma:avgRewardLambda} \\
         $B_8 = F_7 + F_5L_P$ & \Cref{tlemma:avgRewardLambda} \\
         $L_\pol = 4U_R(M+1)\sqrt{|\gS||\gA|}$ & \Cref{tlemma:avgRewardLipschitz} \\
         $L_P = 4U_R M$ & \Cref{tlemma:avgRewardLipschitz} \\
         $G_\pol = 2U_Q\sqrt{|\gS||\gA|}$ & \Cref{tlemma:QPolicyLipschitz} \\ 
         $G_R = 2\lambda^{-1}\sqrt{\lvert\mathcal{S}\rvert\lvert\mathcal{A}\rvert}$ & \Cref{tlemma:QTimeLipschitz} \\
         $G_P = (\lambda^{-1}L_P + 4U_R\lambda^{-1}M + 4U_R\lambda^{-2}(M+1))\sqrt{|\gS||\gA|}$ & \Cref{tlemma:QTimeLipschitz} \\
         $F_{1\polv} = 2U_Q L_\polv + 4U_Q G_\polv + 8U_Q^2(M+2)\lvert\gS\rvert\lvert\mathcal{A}\rvert$ & \Cref{alemma:critic1}\\
         $F_{1P} = 2U_QL_P + 4U_Q G_P + 8U_Q^2(M+1)\sqrt{|\gS||\gA|}$ & \Cref{alemma:critic1} \\
         $F_2 = 2U_R + 18U_Q$ & \Cref{alemma:critic2} \\
         $F_3 = 16U_R U_Q + 24U_Q^2 \sqrt{|\gS||\gA|}$ & \Cref{alemma:critic3} \\
         $F_4 = 8U_R U_Q + 24U_Q^2\sqrt{|\gS||\gA|}$ & \Cref{alemma:critic4} \\
         $F_5 = 4U_R$ & \Cref{alemma:avgReward1} \\
         $F_6 = 2U_R$ & \Cref{alemma:avgReward2} \\
         $F_7 = 8U_R^2$ & \Cref{alemma:avgReward3} \\
         $F_8 = 8U_R^2$ & \Cref{alemma:avgReward4} \\
         \hline
    \end{tabular}
    % \caption{Caption}
\end{table}

%%%%%%%%%%%%%%%%%%%
%%%%%%%%%%%%%%%%%%%

\newpage

\section{Regret Analysis: NS-NAC}
\label{app:regret}

\begin{theorem} \label{theorem:regretUpperBound}
    If \Cref{assumption:ergodic} is satisfied and the step-sizes are chosen as $0 < \clr, \alr, \rlr{\color{black}, \temp} < 1/2$ {\color{black} and number of restarts as $0 < N < T$} in \Cref{algo}, then we have
    \begin{align*}
        \hbox{Dyn-Reg}(\mathcal{M}, T) & = \mathbb{E}\left[\sum_{t=0}^{T-1} J_t^{\polv_t^\star} - \r_t(\s_t, \a_t) \right] \\ 
        & \leq \underbrace{\mathcal{\Tilde{O}}\left(\mfrac{N}{\alr} \right) + \mathcal{\Tilde{O}}\left(\sqrt{\mfrac{NT}{\clr}} \right)}_{\substack{\text{ \color{black} Effect of initialization}}} + \underbrace{\mathcal{\Tilde{O}}\left(\mfrac{\alr T}{\clr} \right) + \mathcal{\Tilde{O}}\left(T\sqrt{\alr} \right)}_{\substack{\text{Cumulative change} \\ \text{in policy over horizon } T}} + \underbrace{\mathcal{\Tilde{O}}\left(\mfrac{\alr T}{\rlr} \right) + \mathcal{\Tilde{O}}\left(T\sqrt{\rlr} \right) + \mathcal{\Tilde{O}}\left(\sqrt{\mfrac{NT}{\rlr}} \right)}_{\text{Error in Average Reward Estimate at Critic}} \\ 
        & \quad + \underbrace{\mathcal{\Tilde{O}}\left(T\sqrt{\clr} \right)}_{\substack{\text{Cumulative change} \\ \text{in critic estimates}}} + \underbrace{\mathcal{\Tilde{O}}\left(\mfrac{\Delta_T T}{N} \right) + \mathcal{\Tilde{O}}\left(\Delta_T^{1/3}T^{2/3} \lpar \mfrac{1}{\color{black}\sqrt{\clr}} + \mfrac{1}{\color{black}\sqrt{\rlr}} \rpar \right)}_{\text{Error due to Non-Stationarity}},
    \end{align*}
    where $\Delta_T = \Delta_{R,T} + \Delta_{P,T}$, $\mathcal{\Tilde{O}}(\cdot)$ hides the constants and logarithmic dependence on the time horizon $T$. Choosing optimal $\clr^\star = \rlr^\star = \left(\frac{\Delta_T}{T}\right)^{\color{black}1/3}$, $\alr^\star = \left(\frac{\Delta_T}{T}\right)^{\color{black}1/2}$ and $N^\star = \Delta_T^{\color{black}5/6} T^{\color{black}1/6}$, the resulting regret (with explicit dependence on the size of the state-action space $|\gS|, |\gA|$) is
    \begin{align*}
        \hbox{Dyn-Reg}(\mathcal{M}, T) \leq \mathcal{\Tilde{O}}\left(|\gS|^{1/2}|\gA|^{1/2} {\color{black}\Delta_T^{1/6} T^{5/6}} \right).
    \end{align*}
\end{theorem}

\begin{proof}
    {\color{black} Recall that \Cref{algo} divides the total time horizon $T$ into $N$ segments of length $H = \lfloor \frac{T}{N} \rfloor$.}
    \begin{align*}
        & \mathbb{E}\left[\sum_{t=0}^{T-1} J_t^{\polv_t^\star} - \r_t(\s_t, \a_t) \right] \\
        & = \mathbb{E}\left[\sum_{t=0}^{T-1} J_t^{\polv_t^\star} - J_t^{\polv_t}\right] + \mathbb{E}\left[\sum_{t=0}^{T-1} J_t^{\polv_t} - \r_t(\s_t, \a_t) \right]\\ 
        & \overset{(a)}{\leq} \mathcal{\Tilde{O}}\left(\mfrac{\Delta_T T}{N} \right) + \mathcal{\Tilde{O}}\left(\mfrac{N}{\alr} \right) + \mathcal{\Tilde{O}}\left(\alr T \right) + 2\sum\limits_{n=0}^{N-1}\sum_{h=0}^{H-1}\mathbb{E}\Big[\Vert{\color{black}\Pi_E}\lbr\qv_{nH+h}^{\polv_{nH+h}}-\qv_{nH+h}\rbr\Vert_\infty\Big] + \mathbb{E}\left[\sum_{t=0}^{T-1} J_t^{\polv_t} - \r_t(\s_t, \a_t) \right] \\
        & \overset{}{\leq} \mathcal{\Tilde{O}}\left(\mfrac{\Delta_T T}{N} \right) + \mathcal{\Tilde{O}}\left(\mfrac{N}{\alr} \right) + \mathcal{\Tilde{O}}\left(\alr T \right) + 2\sum\limits_{n=0}^{N-1} H^{1/2} \left(\sum_{h=0}^{H-1}\mathbb{E}\Big[\Vert{\color{black}\Pi_E}\lbr\qv_{nH+h}^{\polv_{nH+h}}-\qv_{nH+h}\rbr\Vert_2^2\Big]\right)^{1/2} + \mathbb{E}\left[\sum_{t=0}^{T-1} J_t^{\polv_t} - \r_t(\s_t, \a_t) \right] \\
        & \overset{(b)}{\leq} \mathcal{\Tilde{O}}\left(\mfrac{\Delta_T T}{N} \right) + \mathcal{\Tilde{O}}\left(\mfrac{N}{\alr} \right) + \mathcal{\Tilde{O}}\left(\alr T \right) + 2\sum_{n=0}^{N-1} \Bigg[ \Tilde{\mco}\left(\sqrt{H} \right) + \Tilde{\mco}\left(\sqrt{\mfrac{H}{\clr}} \right) + \mathcal{\Tilde{O}}\left(\sqrt{\clr} H \right) + \mathcal{\Tilde{O}}\left(\mfrac{\alr H}{\clr}\right) + \mathcal{\Tilde{O}}\left(\sqrt{\alr} H \right) + \mathcal{\Tilde{O}}\left(\mfrac{\alr H}{\rlr} \right) \\ 
        & \quad  + \mathcal{\Tilde{O}}\left(\sqrt{\rlr} H \right) + \mathcal{\Tilde{O}}\left(\sqrt{\mfrac{H}{\rlr}} \right) + \mathcal{\Tilde{O}}\left(\mfrac{\Delta_{nH, (n+1)H}^{1/3}H^{2/3}}{\color{black}\sqrt{\clr}} \right) + \mathcal{\Tilde{O}}\left(\mfrac{\Delta_{nH, (n+1)H}^{1/3}H^{2/3}}{\color{black}\sqrt{\rlr}} \right) \Bigg] + \mathbb{E}\left[\sum_{t=\mt_T}^{T-1} J_t^{\polv_t} - \r_t(\s_t, \a_t) \right] \\
        & \overset{(c)}{\leq} \mathcal{\Tilde{O}}\left(\mfrac{\Delta_T T}{N} \right) + \mathcal{\Tilde{O}}\left(\mfrac{N}{\alr} \right) + \mathcal{\Tilde{O}}\left(\alr T \right) + \Tilde{\mco}\left(N\sqrt{H} \right) + \mathcal{\Tilde{O}}\left(\sqrt{\mfrac{NT}{\clr}} \right) + \mathcal{\Tilde{O}}\left(T\sqrt{\clr}ß \right) + \mathcal{\Tilde{O}}\left(\mfrac{\alr T}{\clr}\right) + \mathcal{\Tilde{O}}\left( T\sqrt{\alr} \right) + \mathcal{\Tilde{O}}\left(\mfrac{\alr T}{\rlr} \right) \\ 
        & \quad + \mathcal{\Tilde{O}}\left(T \sqrt{\rlr}\right) + \mathcal{\Tilde{O}}\left(\sqrt{\mfrac{NT}{\rlr}} \right) + \mathcal{\Tilde{O}}\left(\mfrac{\Delta_T^{1/3}T^{2/3}}{\color{black}\sqrt{\clr}} \right) + \mathcal{\Tilde{O}}\left(\mfrac{\Delta_T^{1/3}T^{2/3}}{\color{black}\sqrt{\rlr}} \right) + \mathcal{\Tilde{O}}\left(1 \right) + \mathcal{\Tilde{O}}\left(\alr T \right) + \mathcal{\Tilde{O}}\left(\Delta_{P, T} \right),
    \end{align*}
% }%%%%%%%%%%%%%
    where $(a)$ is due to \Cref{proposition:actor}, $(b)$ is by \Cref{proposition:criticEstimate} and and $\Delta_{nH, (n+1)H} = \Delta_{R, nH, (n+1)H} + \Delta_{P, nH, (n+1)H}$, $(c)$ is by Jensen's inequality, $\Delta_T = \Delta_{R, T}+\Delta_{P, T}$ and \Cref{proposition:sampledRewardRegret}. We further have $\mt_H = \mathcal{O}(\log T)$. Note that $\mathcal{\Tilde{O}}(\cdot)$ hides constants and logarithmic terms.
\end{proof}

\newpage
\subsection{\textbf{Actor}} \label{app:actor}
The next result bounds the performance difference, measured by the average reward, between the optimal policies $\polv_t^\star$ and the current policy $\polv_t$.

\begin{proposition} \label{proposition:actor}
    If Assumption \ref{assumption:ergodic} holds, we have 
    \begin{align*}
        & \mathbb{E}\left[\sum_{t=0}^{T-1} J_t^{\polv_t^\star} - J_t^{\polv_t}\right]  \leq \underbrace{\left( 2 + 2 G_R + \frac{1}{\concr} \right) \frac{T \Delta_{R, T}}{N} + \left( U_Q + L_P + 2 G_P + \frac{L_P}{\concr} \right) \frac{T \Delta_{P,T}}{N}}_{\text{Error due to Non-Stationarity}} \\ 
        & \quad + \underbrace{2\sum\limits_{n=0}^{N-1}\sum\limits_{h=0}^{H-1}\mathbb{E}\Big[\Vert{\color{black}\Pi_E}\left[\qv_{nH+h}^{\polv_{nH+h}}-\qv_{nH+h}\right]\Vert_\infty\Big]}_{\text{Critic Estimation Error}} + \underbrace{N \cdot\frac{\color{black} \log |\gA|}{\alr}}_{\substack{N \times \text{Bias of} \\ \text{Initialization}}} + \underbrace{\frac{B_1 \alr T}{\concr}}_{\substack{\text{Bounds cumulative} \\ \text{change in policy}}} + \underbrace{\frac{U_R}{\concr}}_{\text{constant}},
    \end{align*}
    where $\Delta_{R,T} = \sum\limits_{t=0}^{T-1} \Vert \rv_{r+1} - \rv_t \Vert_\infty$, $\Delta_{P, T} = \sum\limits_{t=0}^{T-1} \Vert \pv_{t+1} - \pv_t \Vert_\infty$, $C$ is defined in \Cref{assumption:ergodic}, $T$ is the total time horizon and $N$ is the number of restarts. The remaining constants are defined in \Cref{tab:constantReference}.
\end{proposition}

\begin{proof}
    {\color{black} Recall that \Cref{algo} divides the total time horizon $T$ into $N$ segments of length $H = \lfloor \frac{T}{N} \rfloor$.} In each segment (indexed by $n \in [N]$), we use $J_{nH}^{\polv_{nH}^\star}$ as an anchor against which to compare the performance of the learned policies. 
    \begin{align}
        &\mathbb{E}\left[\sum_{t=0}^{T-1} J_t^{\polv_t^\star} - J_t^{\polv_t}\right] \leq \mathbb{E}\left[\sum_{n=0}^{N-1}\sum_{h=0}^{H -1} \left(J_{nH+h}^{\polv_{nH+h}^\star} - J_{nH}^{\polv_{nH}^\star}\right) + \left(J_{nH}^{\polv_{nH}^\star} - J_{nH}^{\polv_{nH+h}} \right) + \left(J_{nH}^{\polv_{nH+h}} - J_{nH+h}^{\polv_{nH+h}}\right) \right] \nonumber \\
        & \overset{(a)}{\leq} \mathbb{E}\left[ \sum_{n=0}^{N-1}\sum_{h=0}^{H -1} \left(2\Vert\rv_{nH+h} - \rv_{nH}\Vert_\infty + (U_Q+L_P)\Vert\pv_{nH+h} - \pv_{nH}\Vert_\infty \right) + \left(J_{nH}^{\polv_{nH}^\star} - J_{nH}^{\polv_{nH+h}}\right) \right] \nonumber
        \\
        & \overset{(b)}{\leq} \sum_{n=0}^{N-1}\sum_{h=0}^{H-1}  (H-1)\left(2\Vert\rv_{nH+h} - \rv_{nH+h-1}\Vert_\infty + (U_Q+L_P)\Vert\pv_{nH+h} - \pv_{nH+h-1}\Vert_\infty \right) + \left(J_{nH}^{\polv_{nH}^\star} - J_{nH}^{\polv_{nH+h}}\right) \nn\\
        & \leq (H-1) \left(2\Delta_{R, T} + (U_Q+L_P) \Delta_{P, T}\right) + \mathbb{E}\left[ \sum_{n=0}^{N-1}\sum_{h=0}^{H-1} J_{nH}^{\polv_{nH}^\star} - J_{nH}^{\polv_{nH+h}}\right], \label{eqn:propositionActorStart}
    \end{align}
    where $(a)$ is by \Cref{tlemma:bestAvgRewardDiff} and \Cref{tlemma:avgRewardLipschitz} and $(b)$ is by triangle inequality. We now bound the last term as
    \begin{align}
        & \sum_{n=0}^{N-1}\sum_{h=0}^{H-1} J_{nH}^{\polv_{nH}^\star} - J_{nH}^{\polv_{nH+h}} \nonumber\\
        & \overset{(c)}{=} \sum_{n=0}^{N-1}\sum_{h=0}^{H-1} \frac{1}{\alr} \sum_{\s}\sum_{\a} d^{\polv_{nH}^\star, \pv_{nH}}(\s) \polv_{nH}^\star(\a|\s) \left[\alr\q_{nH}^{\polv_{nH+h}}(\s, \a) - \alr\v_{nH}^{\polv_{nH+h}}(\s)  \right] \nonumber\\
        & = \sum_{n}\sum_{h} \frac{1}{\alr} \sum_{\s}\sum_{\a} d^{\polv_{nH}^\star, \pv_{nH}}(\s) \polv_{nH}^\star(\a|\s) \left[\alr\q_{nH+h}^{\polv_{nH+h}}(\s, \a) - \alr\v_{nH+h}^{\polv_{nH+h}}(\s) + \alr\q_{nH+h}(\s, \a) - \alr\q_{nH+h}(\s, \a)  \right] \nonumber \\ & \quad + \sum_{n}\sum_{j} \frac{1}{\alr} \sum_{\s}\sum_{\a} d^{\polv_{nH}^\star, \pv_{nH}}(\s) \polv_{nH}^\star(\a|\s) \left[\alr \q_{nH}^{\polv_{nH+h}}(\s, \a) - \alr\q_{nH+h}^{\polv_{nH+h}}(\s, \a) + \alr \v_{nH+h}^{\polv_{nH+h}}(\s) - \alr\v_{nH}^{\polv_{nH+h}}(\s) \right] \nonumber\\
        & = \sum_{n}\sum_{h} \frac{1}{\alr} \sum_{\s}\sum_{\a} d^{\polv_{nH}^\star, \pv_{nH}}(\s) \polv_{nH}^\star(\a|\s) \left[\alr\q_{nH+h}^{\polv_{nH+h}}(\s, \a) - \alr\v_{nH+h}^{\polv_{nH+h}}(\s) + \alr\q_{nH+h}(\s, \a) - \alr\q_{nH+h}(\s, \a)  \right] \nonumber \\ & \quad + \sum\limits_n \sum\limits_h 2 \Vert \qv_{nH}^{\polv_{nH+h}} - \qv_{nH+h}^{\polv_{nH+h}} \Vert_\infty \nonumber\\
        & \overset{(d)}{\leq} \sum_{n}\sum_{h} \frac{1}{\alr} \sum_{\s}\sum_{\a} d^{\polv_{nH}^\star, \pv_{nH}}(\s) \polv_{nH}^\star(\a|\s) \left[\alr\q_{nH+h}^{\polv_{nH+h}}(\s, \a) - \alr\v_{nH+h}^{\polv_{nH+h}}(\s) + \alr\q_{nH+h}(\s, \a) - \alr\q_{nH+h}(\s, \a)  \right] \nonumber \\ & \quad + \sum_{n=0}^{N-1}\sum_{h=0}^{H - 1} 2G_R\Vert\rv_{nH} - \rv_{nH+h}\Vert_\infty + 2G_P \Vert\pv_{nH}-\pv_{nH+h}\Vert_\infty \nonumber\\
        & \overset{(e)}{=} \sum_{n}\sum_{h} \frac{1}{\alr} \sum_{\s}\sum_{\a} d^{\polv_{nH}^\star, \pv_{nH}}(\s) \polv_{nH}^\star(\a|\s) \Bigg[ \underbrace{\log Z_{nH+h}(\s) - \alr\v_{nH+h}^{\polv_{nH+h}}(\s)}_{I_1} \Bigg] \nonumber\\ 
        & \quad + \sum_{n}\sum_{h} \frac{1}{\alr} \sum_{\s}\sum_{\a} d^{\polv_{nH}^\star, \pv_{nH}}(\s) \polv_{nH}^\star(\a|\s) \Bigg[ \underbrace{\log\frac{\pol_{nH+h+1}(\a|\s)}{\pol_{nH+h}(\a|\s)}}_{I_2} + \underbrace{\alr\q_{nH+h}^{\polv_{nH+h}}(\s, \a) - \alr\q_{nH+h}(\s, \a)}_{I_3} \Bigg] \nonumber\\ 
        & \quad + (H-1)(2G_R\Delta_{R, T} + 2G_P\Delta_{P, T})\label{eq_proof:proposition:actor}
    \end{align}
    where $(c)$ follows from the Performance Difference Lemma~\ref{tlemma:perfDiff}, $(d)$ follows from \Cref{tlemma:QTimeLipschitz} and $(e)$ from the actor update equation (line~\ref{line:actorUpdate} in \Cref{algo}) and $Z_t(\s) = \sum_{\a' \in \mathcal{A}} \pol_t(\a'|\s) \exp(\alr\q_t(\s, \a'))$. Next, we bound each of $I_1, I_2, I_3$. Using \Cref{tlemma:actor}, we have
    \begin{align}
        & I_1 = \sum_{n}\sum_{h} \sum_{\s} d^{\polv_{nH}^\star, \pv_{nH}}(\s) \left[\frac{\log Z_{nH+h}(\s)}{\alr} - \v_{nH+h}^{\polv_{nH+h}}(\s)\right] \underbrace{\sum_{\a}\pol_{nH}^\star(\a|\s)}_{=1} \nonumber\\
        & \leq \sum_n \sum_h \lbr \mfrac{J_{nH+h+1}^{\polv_{nH+h+1}} - J_{nH+h}^{\polv_{nH+h}}}{\concr} + \Vert \qv_{nH+h}^{\polv_{nH+h}} - \qv_{nH+h} \Vert_\infty + \mfrac{B_1 \alr}{\concr} + \mfrac{\Vert \rv_{nH+h+1} - \rv_{nH+h} \Vert_\infty}{\concr} + \mfrac{L_P \Vert \pv_{nH+h+1} - \pv_{nH+h} \Vert_\infty}{\concr} \rbr. \label{eq:proposition:actor_I1}
    \end{align}

    Next, we establish a bound on $I_2$ as 
    \begin{align}
        & I_2 = \frac{1}{\alr} \sum_{n}\sum_{h} \sum_{\s} \sum_{\a} d^{\polv_{nH}^\star, \pv_{nH}}(\s) \polv_{nH}^\star(\a|\s) \log\frac{\pol_{nH+h+1}(\a|\s)}{\pol_{nH+h}(\a|\s)} \nonumber\\
        & \overset{}{\leq} \frac{1}{\alr} \sum_{n=0}\sum_{h} \sum_{\s} d^{\polv_{nH}^\star, \pv_{nH}}(\s) \left[\KL\left(\polv_{nH}^\star(\cdot |\s) \Vert \polv_{nH+h}(\cdot|\s)\right) - \KL\left(\polv_{nH}^\star(\cdot |\s) \Vert \polv_{nH+h+1}(\cdot |\s)\right) \right] \nonumber\\
        & = \frac{1}{\alr} \sum_{n} \sum_{\s} d^{\polv_{nH}^\star, \pv_{nH}}(\s) \left[\KL\left(\polv_{nH}^\star(\cdot |\s) \Vert \polv_{nH}(\cdot |\s)\right) - \KL\left(\polv_{nH}^\star(\cdot|\s) \Vert \polv_{(n+1)H}\right) \right] \nonumber\\
        & \overset{(f)}{\leq} \frac{1}{\alr} \sum_{n} \sum_{\s} d^{\polv_{nH}^\star, \pv_{nH}}(\s) \KL\left(\polv_{nH}^\star(\cdot |\s) \Vert \polv_{nH}(\cdot |\s)\right) \nonumber \\
        & \overset{(g)}{\leq} \frac{1}{\alr} \sum_n \sum_s d^{\polv_{nH}^\star, \pv_{nH}}(\s) \log \frac{|\gA|}{1} \leq \frac{N\log |\gA|}{\alr} 
    \end{align}
    where $(f)$ is because of non-negativity of KL-divergence and {\color{black} $(g)$ is due to the restart in line~\ref{line:restart} of \Cref{algo}.} Lastly, $I_3$ can be bounded as 
    \begin{align}
        I_3 &= \sum_{n}\sum_{h} \sum_{\s}\sum_{\a} d^{\polv_{nH}^\star, \pv_{nH}}(\s) \polv_{nH}^\star(\a|\s) \left[\q_{nH+h}^{\polv_{nH+h}}(\s, \a) - \q_{nH+h}(\s, \a) \right] \nn \\
        & \leq \sum_n\sum_h \Vert \qv_{nH+h}^{\polv_{nH+h}} - \qv_{nH+h} \Vert_\infty. \label{eqn:propositionActorEnd}
    \end{align}
    We substitute the bounds on $I_1, I_2, I_3$ from (\ref{eq:proposition:actor_I1})-(\ref{eqn:propositionActorEnd}) in (\ref{eq_proof:proposition:actor}) and then combine with (\ref{eqn:propositionActorStart}). {\color{black} Recall that the set of solutions to the Bellman equations is $\qv_t^{\polv_t} = \{\qv_{t, E}^{\polv_t} + c\mathbf{1} | \qv_{t, E}^{\polv_t} \in E, c \in \mathbb{R}\}$ where $E$ is the subspace orthogonal to the all ones vector and $\qv_{t, E}^{\polv_t}$ is the unique solution in $E$ \citep{zhang2021finite}. Finally, we use the equivalence $\Vert \qv_t^{\polv_t} - \qv_t \Vert_\infty = \Vert \Pi_E\left[\qv_t^{\polv_t} - \qv_t\right] \Vert_\infty$ to get the result. }
\end{proof}

%%%%%%%%%%%%%%%%%%%
%%%%%%%%%%%%%%%%%%%

\newpage
\subsection{\textbf{Critic}} \label{app:critic}
In this section, we characterize the error in the critic estimation.

\begin{proposition} \label{proposition:criticEstimate}
    {\color{black} For any $n \in [N]$}, if \Cref{assumption:ergodic} is satisfied and $0 < \rlr < 1/2$, then we have
        \begin{align*}
            & \sum_{t=nH+\mt_H}^{nH+H-1} \mathbb{E}\left[\Vert{\color{black}\Pi_E}\left[\qv_t-\qv_t^{\polv_t}\right]\Vert_2^2\right] \\
            & \leq \underbrace{\mathcal{\Tilde{O}}\left(\frac{1}{\clr} \right)}_{\substack{\text{Effect of} \\ \text{initialization}}} + \underbrace{\mathcal{\Tilde{O}}\left(\frac{\Delta_{R, nH,(n+1)H}^{2/3}H^{1/3}}{\color{black} \clr} \right) + \mathcal{\Tilde{O}}\left(\frac{\Delta_{P, nH,(n+1)H}^{2/3}H^{1/3}}{\color{black} \clr} \right)}_{\text{Error due to Non-Stationarity}} \\
            & \quad + \underbrace{\mathcal{\Tilde{O}} \left( \rlr H \right) + \mathcal{\Tilde{O}}\left( \frac{1}{\rlr} \right) + \mathcal{\Tilde{O}}\left( \frac{\alr^2 H}{\rlr^2} \right) + \mathcal{\Tilde{O}}\left( \frac{\Delta_{R, nH,(n+1)H}^{2/3}H^{1/3}}{\color{black} \rlr} \right) + \mathcal{\Tilde{O}}\left( \frac{\Delta_{P, nH,(n+1)H}^{2/3}H^{1/3}}{\color{black} \rlr} \right)}_{\text{Error in Average Reward Estimate } (\re_t) \text{ at Critic}} \\
            & \quad + \underbrace{\mathcal{\Tilde{O}}\left(\alr T \right) + \mathcal{\Tilde{O}}\left(\frac{\alr^2 H}{\clr^2} \right)}_{\substack{\text{Bounds cumulative change} \\ \text{in policy over horizon } T}} + \underbrace{\mathcal{\Tilde{O}}\left(\clr H \right)}_{\substack{\text{Bounds cumulative} \\ \text{change in critic estimates}}}, \nonumber
        \end{align*}
        where $\mathcal{\Tilde{O}}(\cdot)$ hides constants and logarithmic terms which can be found in \Cref{eqn:exactCriticEstimationExpression} and $\Delta_{R, nH,(n+1)H} = \sum_{t=nH}^{(n+1)H} \Vert\rv_{t+1}-\rv_t\Vert_\infty$, and $\Delta_{P, nH,(n+1)H} = \sum_{t=nH}^{(n+1)H} \Vert\pv_{t+1}-\pv_t\Vert_\infty$.
\end{proposition}
\begin{proof}
    Recall that $\qdv_t = {\color{black}\Pi_E}\left[\qv_t - \qv_t^{\polv_t}\right]$, {\color{black} $E$ is the subspace orthogonal to the all ones vector $\mathbf{1}$} and the critic update equation (line~\ref{line:criticUpdate} in \Cref{algo}) can be expressed in vector form as $\qv_{t+1} = \Pi_{R_Q}\left[\qv_t + \clr\left(\rv_t(O_t) - \rev_t(O_t) + \Av(O_t)\qv_t\right) \right]$. Recall the notations $\rv_t, \re_t, \Av(O_t), \Bar{\Av}^{\polv_t, \pv_t}, \Jv_t(O_t), \Gamma(\cdot), \phi_t$ from \Cref{app:notation}. We therefore have
    \begin{align*}
        & \Vert\qdv_{t+1}\Vert_2^2 = \Vert {\color{black}\Pi_E}\left[\qv_{t+1} - \qv_{t+1}^{\polv_{t+1}}\right]\Vert_2^2 \\
        & \overset{}{\leq} \Vert{\color{black}\Pi_E}\left[\qv_t + \clr \left(\rv_t(O_t) - \rev_t(O_t) + \Av(O_t)\qv_t\right) - \qv_{t+1}^{\polv_{t+1}}\right] \Vert_2^2 \\
        & = \Vert{\color{black}\Pi_E}\left[\qdv_t + \clr\left(\rv_t(O_t) - \rev_t(O_t) + \Av(O_t)\qv_t\right) + \qv_t^{\polv_t} - \qv_{t+1}^{\polv_{t+1}}\right] \Vert_2^2 \\
        & \overset{}{\leq} \Vert \qdv_t \Vert_2^2 + 2\clr\qdv_t^\top \left(\rv_t(O_t) - \rev_t(O_t)\right) + \Av(O_t)\qv_t \\ 
        & \quad + 2\qdv_t^\top {\color{black}\Pi_E}\left[\qv_t^{\polv_t} - \qv_{t+1}^{\polv_{t+1}} \right]  + 2\clr^2\Vert \rv_t(O_t) - \rev_t(O_t) + \Av(O_t)\qv_t \Vert_2^2 + 2\Vert {\color{black}\Pi_E}\left[\qv_t^{\polv_t} - \qv_{t+1}^{\polv_{t+1}}\right] \Vert_2^2 \\
        & \overset{}{\leq} \Vert\qdv_t \Vert_2^2 + 2\clr\qdv_t^\top \left(\rv_t(O_t) - \rev_t(O_t) + \Av(O_t)\qv_t - \Bar{\Av}^{\polv_t, \pv_t}\qdv_t \right) + 2\clr\qdv_t^\top \Bar{\Av}^{\polv_t, \pv_t}\qdv_t \\ 
        & \quad + 2\qdv_t^\top {\color{black}\Pi_E}\left[\qv_t^{\polv_t} - \qv_{t+1}^{\polv_{t+1}} \right]  + 2\clr^2\Vert \rv_t(O_t) - \rev_t(O_t) + \Av(O_t)\qv_t \Vert_2^2 + 2\Vert {\color{black}\Pi_E}\left[\qv_t^{\polv_t} - \qv_{t+1}^{\polv_{t+1}}\right] \Vert_2^2 \\
        & \overset{}{\leq} \Vert\qdv_t \Vert_2^2 + 2\clr\qdv_t^\top \left(\rv_t(O_t) - \rev_t(O_t) + \Av(O_t)\qv_t^{\polv_t}\right)  + 2\clr\qdv_t^\top \left(\Av(O_t) - \Bar{\Av}^{\polv_t, \pv_t}\right)\qdv_t + 2\clr\qdv_t^\top \Bar{\Av}^{\polv_t, \pv_t}\qdv_t \\ 
        & \quad + 2\qdv_t^\top {\color{black}\Pi_E}\left[\qv_t^{\polv_t} - \qv_{t+1}^{\polv_{t+1}} \right]  + 2\clr^2\Vert \rv_t(O_t) - \rev_t(O_t) + \Av(O_t)\qv_t \Vert_2^2 + 2\Vert {\color{black}\Pi_E}\left[\qv_t^{\polv_t} - \qv_{t+1}^{\polv_{t+1}}\right] \Vert_2^2 \\
        & \overset{}{\leq} \Vert\qdv_t \Vert_2^2 + 2\clr\Gamma(\polv_t, \pv_t, \rv_t, \qdv_t, O_t) + 2\clr\qdv_t^\top (\Jv_t^{\polv_t}(O_t) - \re_t(O_t)) + 2\clr\qdv_t^\top \Bar{\Av}^{\polv_t, \pv_t}\qdv_t \\ 
        & \quad + 2\qdv_t^\top {\color{black}\Pi_E}\left[\qv_t^{\polv_t} - \qv_{t+1}^{\polv_{t+1}} \right]  + 2\clr^2\Vert \rv_t(O_t) - \rev_t(O_t) + \Av(O_t)\qv_t \Vert_2^2 + 2\Vert {\color{black}\Pi_E}\left[\qv_t^{\polv_t} - \qv_{t+1}^{\polv_{t+1}}\right] \Vert_2^2 \\
        & \overset{(a)}{\leq} \Vert\qdv_t \Vert_2^2 + 2\clr\Gamma(\polv_t, \pv_t, \rv_t, \qdv_t, O_t) + 2\clr\Vert\qdv_t\Vert_2 \norm{\Jv_t^{\polv_t}(O_t) - \re_t(O_t)}_2 + 2\clr\qdv_t^\top \Bar{\Av}^{\polv_t, \pv_t}\qdv_t \\ 
        & \quad + 2\Vert\qdv_t\Vert_2\Vert{\color{black}\Pi_E}\left[\qv_t^{\polv_t} - \qv_{t+1}^{\polv_{t+1}}\right] \Vert_2  + 2\clr^2\Vert \rv_t(O_t) - \rev_t(O_t) + \Av(O_t)\qv_t \Vert_2^2 + 2\Vert {\color{black}\Pi_E}\left[\qv_t^{\polv_t} - \qv_{t+1}^{\polv_{t+1}}\right] \Vert_2^2 \\
        & \overset{(b)}{\leq} \Vert\qdv_t \Vert_2^2 + 2\clr\Gamma(\polv_t, \pv_t, \rv_t, \qdv_t, O_t) + 2\clr\Vert\qdv_t\Vert_2\lvert J_t^{\polv_t} - \re_t \rvert - 2\clr\lambda\Vert\qdv_t\Vert_2^2 \\ 
        & \quad + 2\Vert\qdv_t\Vert_2\Vert{\color{black}\Pi_E}\left[\qv_t^{\polv_t} - \qv_{t+1}^{\polv_{t+1}}\right] \Vert_2  + 2\clr^2\Vert \rv_t(O_t) - \rev_t(O_t) + \Av(O_t)\qv_t \Vert_2^2 + 2\Vert {\color{black}\Pi_E}\left[\qv_t^{\polv_t} - \qv_{t+1}^{\polv_{t+1}}\right] \Vert_2^2 \\
        & \overset{}{\leq} (1-2\clr\lambda)\Vert\qdv_t \Vert_2^2 + 2\clr\Gamma(\polv_t, \pv_t, \rv_t, \qdv_t, O_t) + 2\clr\Vert\qdv_t\Vert_2\lvert J_t^{\polv_t} - \re_t\rvert \\ 
        & \quad + 2\Vert\qdv_t\Vert_2\Vert{\color{black}\Pi_E}\left[\qv_t^{\polv_t} - \qv_{t+1}^{\polv_{t+1}}\right] \Vert_2  + 2\clr^2(2U_R + 2U_Q)^2 + 2\Vert{\color{black}\Pi_E}\left[ \qv_t^{\polv_t} - \qv_{t+1}^{\polv_{t+1}}\right] \Vert_2^2 \\
        & \overset{}{\leq} (1-2\clr\lambda)\Vert\qdv_t \Vert_2^2 + 2\clr\Gamma(\polv_t, \pv_t, \rv_t, \qdv_t, O_t) + 2\clr\Vert\qdv_t\Vert_2\lvert J_t^{\polv_t} - \re_t\rvert \\ 
        & \quad + {\color{black} 2\Vert\qdv_t\Vert_2\Vert{\color{black}\Pi_E}\left[\qv_t^{\polv_t} - \qv_{t+1}^{\polv_{t}}\right] \Vert_2 + 2\Vert\qdv_t\Vert_2\Vert{\color{black}\Pi_E}\left[\qv_{t+1}^{\polv_t} - \qv_{t+1}^{\polv_{t+1}}\right] \Vert_2}  + 2\clr^2(2U_R + 2U_Q)^2 + 2\Vert{\color{black}\Pi_E}\left[ \qv_t^{\polv_t} - \qv_{t+1}^{\polv_{t+1}}\right] \Vert_2^2, 
    \end{align*}
    where $(a)$ is due to Cauchy-Schwarz inequality, $(b)$ follows from {\color{black} $\qdv_t \in E$ and \Cref{assumption:maxEigValue}}. 
    
    Taking expectation, rearranging the terms, setting $\mt = \mt_H = \min\{i \geq 0 | m \rho^{i-1} \leq \min\{\alr, \clr\} \}$ and summing over time, we have 
    \begin{align}
        & \sum_{t=nH+\mt_H}^{nH+H-1} \lambda\mathbb{E}\left[\Vert\qdv_t\Vert_2^2 \right] \nn \\
        & \leq \underbrace{\sum_{t=nH+\mt_H}^{nH+H-1} \frac{\mathbb{E}[\Vert\qdv_t\Vert_2^2 - \Vert\qdv_{t+1}\Vert_2^2]}{2\clr}}_{I_1} + \underbrace{\sum_{t=nH+\mt_H}^{nH+H-1} \mathbb{E}\left[\Gamma(\polv_t, \pv_t, \rv_t, \qdv_t, O_t)\right]}_{I_2} + \underbrace{\sum_{t=nH+\mt_H}^{nH+H-1} \mathbb{E}\left[\lvert\phi_t\rvert \Vert\qdv_t\Vert_2 \right]}_{I_3} \nn \\ 
        & \quad {\color{black} + \underbrace{\sum_{t=nH+\mt_H}^{nH+H-1} \frac{\mathbb{E}\left[\Vert\qdv_t\Vert_2 \Vert{\color{black}\Pi_E}\left[\qv_t^{\polv_t} - \qv_{t+1}^{\polv_{t}}\right]\Vert_2 \right]}{\clr}}_{I_4}} {\color{black} + \underbrace{\sum_{tnH+\mt_H}^{nH+H-1} \frac{\mathbb{E}\left[\Vert\qdv_t\Vert_2 \Vert{\color{black}\Pi_E}\left[\qv_{t+1}^{\polv_t} - \qv_{t+1}^{\polv_{t+1}}\right]\Vert_2 \right]}{\clr}}_{I_5}} \nn \\ 
        & \quad + \clr(2U_R+2U_Q)^2(T-\mt_T) + \underbrace{\sum_{t=nH+\mt_H}^{nH+H-1} \frac{\mathbb{E}\left[\Vert{\color{black}\Pi_E}\left[\qv_t^{\polv_t}-\qv_{t+1}^{\polv_{t+1}}\right]\Vert_2^2 \right]}{\clr}}_{\color{black} I_6}. \label{eq:proof:proposition:criticEstimate1}
    \end{align}

    We now bound each of the terms starting with the first term as
    \begin{align*}
        I_1 = \frac{\mathbb{E}[\Vert\qdv_{nH+\mt_H}\Vert_2^2 - \Vert\qdv_{nH+H}\Vert_2^2]}{2\clr} \leq \frac{2U_Q^2}{\clr}.
    \end{align*}

    By \Cref{tlemma:criticGamma}, we have
    \begin{align*}
        I_2 & \leq \sum_{t=nH+\mt_H}^{nH+H-1 }B_3 \alr (\mt_H+1)^2 + B_4 \clr \mt_H + B_5 \Delta_{R, t-nH-\mt_H+1, t} + B_6 \mt_H\Delta_{P, t-nH-\mt_H+1, t} \\
        & \leq B_3\alr(\mt_H+1)^2(H-\mt_H) + B_4\clr\mt_H(H-\mt_H) + B_5 \mt_H\Delta_{R, nH,(n+1)H} + B_6\mt_H^2\Delta_{P, nH,(n+1)H}.
    \end{align*}

    By the Cauchy-Schwarz inequality, we have
    \begin{align*}
        I_3 & \leq \sum_{t=nH+\mt_H}^{nH+H-1} \sqrt{\mathbb{E}[\phi_t^2]} \sqrt{\mathbb{E}[\Vert\qdv_t\Vert_2^2]} \leq \left(\sum_{t=nH+\mt_H}^{nH+H-1} \mathbb{E}[\phi_t^2] \right)^{1/2} \left(\sum_{t=nH+\mt_H}^{nH+H-1}\mathbb{E}[\Vert\qdv_t\Vert_2^2] \right)^{1/2},
    \end{align*}
    where $\sum_{t=nH+\mt_H}^{nH+H-1} \mathbb{E}[\phi_t^2]$ can be further bounded using \Cref{proposition:averageRewardEstimate}.

    {\color{black}
    Using \Cref{tlemma:QTimeLipschitz}, we have
    \begin{align*}
        I_4 \leq & \frac{2U_Q}{\clr} \sum_{t=nH+\mt_H}^{nH+H-1} \Vert \qv_t^{\polv_t} - \qv_{t+1}^{\polv_t} \Vert_2 \leq \frac{2U_Q}{\clr} \sum_{t=nH+\mt_H}^{nH+H-1} G_R \Vert \rv_{t+1} - \rv_t \Vert_\infty + G_P\Vert\pv_{t+1}-\pv_t\Vert_\infty \\
        & \leq \frac{2U_Q}{\clr} (G_R \Delta_{R, nH, (n+1)H} + G_P\Delta_{P, nH, (n+1)H}). 
    \end{align*}

    Using the Cauchy-Schwarz inequality and \Cref{tlemma:QPolicyLipschitz}, we have
    \begin{align*}
        I_5 & \leq \left(\sum_{t=nH+\mt_H}^{nH+H-1} \frac{\mathbb{E}[\Vert{\color{black}\Pi_E}\left[\qv_{t+1}^{\polv_t} - \qv_{t+1}^{\polv_{t+1}}\right]\Vert_2^2]}{\clr^2} \right)^{1/2} \left(\sum_{t=nH+\mt_H}^{nH+H-1}\mathbb{E}[\Vert\qdv_t\Vert_2^2] \right)^{1/2} \\
        & \leq \left(\frac{G_\polv^2B_2^2\alr^2 H}{\clr^2} \right)^{1/2} \left(\sum_{t=nH+\mt_H}^{nH+H-1} \mathbb{E}[\Vert\qdv_t\Vert_2^2] \right)^{1/2}.
    \end{align*}

    We now the final term $I_6$ as follows. For timesteps with small changes in the environment, we use \Cref{tlemma:consecutiveTimeQDiff}, and for timesteps with large changes in the environment, we use a naive upper bound. Define the set of timesteps $\mathcal{T}_Q := \{t : \Vert \rv_{t+1}-\rv_t\Vert_\infty \leq \delta_R, \Vert \pv_{t+1}-\pv_t\Vert_\infty \leq \delta_P \}$.

    \begin{align}
        I_6 & = \sum_{t=nH+\mt_H}^{nH+H-1} \frac{\mathbb{E}\left[\Vert{\color{black}\Pi_E}\left[\qv_t^{\polv_t} - \qv_{t+1}^{\polv_{t+1}}\right]\Vert_2^2 \right]}{\clr} \overset{(c)}{\leq} \sum_{t \in \mathcal{T}_Q} \frac{\mathbb{E}\left[\Vert{\color{black}\Pi_E}\left[\qv_t^{\polv_t} - \qv_{t+1}^{\polv_{t+1}}\right]\Vert_2^2\right]}{\clr} + \sum_{t \notin \mathcal{T}_Q} \frac{4U_Q^2}{\clr} \nonumber\\
        & \overset{(d)}{\leq} \sum_{t \in \mathcal{T}_Q} \frac{3G_R^2 \delta_R^2}{\clr} + \frac{3G_P^2 \delta_P^2}{\clr} + \frac{3G_\polv^2 B_2^2 \alr^2}{\clr} + \sum_{t\notin \mathcal{T}_Q} \frac{4U_Q^2}{\clr} \nonumber\\
        & \overset{(e)}{\leq} \frac{3G_R^2 \delta_R^2 H}{\clr} + \frac{3G_P^2 \delta_P^2 H}{\clr} + \frac{3G_\polv^2 B_2^2 \alr^2 H}{\clr} + \frac{4U_Q^2 \Delta_{R,nH,(n+1)H}}{\clr\delta_R} + \frac{4U_Q^2 \Delta_{P,nH,(n+1)H}}{\clr\delta_P} \nonumber\\
        & \overset{(f)}{\leq} \frac{W_1 \Delta_{R,nH,(n+1)H}^{2/3} H^{1/3}}{\clr}  + \frac{W_2 \Delta_{P, nH,(n+1)H}^{2/3} H^{1/3}}{\clr} + \frac{3G_\polv^2 B_2^2 \alr^2 H}{\clr} \label{eq:instanceIndependentQTimeDiff}
    \end{align}
    where $(c)$ follows from the \Cref{plemma:maxValues}, $(d)$ follows from \Cref{tlemma:consecutiveTimeQDiff} and $(e)$ is obtained by choosing $\delta_R = \left(\frac{4U_Q^2 \Delta_{R, nH,(n+1)H}}{3G_R^2 H}\right)^{1/3}$ and $\delta_P = \left(\frac{4U_Q^2 \Delta_{P, nH,(n+1)H}}{3G_P^2 H}\right)^{1/3}$ and defining $W_1 = (3G_R^2)^{1/3}(4U_Q^2)^{2/3}$, $W_2 = (3G_P^2)^{1/3} (4U_Q^2)^{2/3}$.
    }

    We substitute the bounds on {\color{black} $I_1, \dots, I_6$} (using \Cref{proposition:averageRewardEstimate}) into (\ref{eq:proof:proposition:criticEstimate1}) and use the squaring trick from Section C.3 in \citet{Wu2020A2CPG}. The above equation is of the form, $X \leq Y + Z\sqrt{X}$. Completing the squares and rearranging, we get $X \leq 2Y  + Z^2$. Hence, we get the final result as 
    \begin{align}
        & \sum_{t=nH+\mt_H}^{nH+H-1} \mathbb{E}\left[\Vert\qdv_t\Vert_2^2 \right] \nn \\
        & \leq \frac{4U_Q^2}{\clr\lambda} + \frac{2B_3\alr(\mt_H+1)^2 H}{\lambda} + \frac{2\clr(B_4 + 8U_R^2 + 8U_Q^2)\mt_H H}{\lambda} + \frac{2B_5 \mt_H\Delta_{R, nH, (n+1)H}}{\lambda} + \frac{2B_6\mt_H^2\Delta_{P, nH, (n+1)H}}{\lambda} \nonumber\\ 
        & \quad + \frac{8U_R^2}{\rlr\lambda^2} + \frac{4B_7\alr(\mt_H+1)^2 H}{\lambda^2} + \frac{2D_3 \rlr \mt_H H}{\lambda^2} + \frac{4B_8 (\mt_H+1)^2 \Delta_{P, nH, (n+1)H}}{\lambda^2} \nonumber\\ 
        & \quad + \frac{2D_4 \alr^2 H}{\rlr^2\lambda^2} + \frac{{\color{black}8}W_3 \Delta_{R, nH, (n+1)H}^{2/3}H^{1/3}}{{\color{black}\rlr}\lambda^2} + \frac{{\color{black}8}W_4 \Delta_{P, nH, (n+1)H}^{2/3}H^{1/3}}{{\color{black}\rlr}\lambda^2} \nn \\
        & \quad {\color{black} + \frac{4U_Q G_R \Delta_{R,nH, (n+1)H}}{\clr\lambda} + \frac{4U_Q G_P \Delta_{P, nH, (n+1)H}}{\clr\lambda} + \frac{G_{\polv}^2 B_2^2 \alr^2 H}{\clr^2\lambda^2}} \nn \\
        & \quad {\color{black} + \frac{2W_1 \Delta_{R,nH, (n+1)H}^{2/3} H^{1/3}}{\clr\lambda} + \frac{2W_2 \Delta_{P, nH, (n+1)H}^{2/3} H^{1/3}}{\clr\lambda} + \frac{6G_\polv^2B_2^2\alr^2 H}{\clr\lambda}} \label{eqn:exactCriticEstimationExpression} \\
        & \leq \mathcal{\Tilde{O}}\left(\frac{1}{\clr} \right) + \mathcal{\Tilde{O}}\left(\alr H \right) + \mathcal{\Tilde{O}}\left(\clr H \right) + \mathcal{\Tilde{O}}\left(\Delta_{R, nH, (n+1)H} \right) + \mathcal{\Tilde{O}}\left(\Delta_{P, nH, (n+1)H} \right) + \mathcal{\Tilde{O}}\left(\rlr H \right) + \mathcal{\Tilde{O}}\left(\frac{1}{\rlr} \right) \nonumber \\ 
        & \quad + \mathcal{\Tilde{O}}\left(\frac{\alr^2 H}{\rlr^2} \right) + \mathcal{\Tilde{O}}\left(\frac{\Delta_{R, nH, (n+1)H}^{2/3}H^{1/3}}{\color{black}\rlr} \right) + \mathcal{\Tilde{O}}\left(\frac{\Delta_{P, nH, (n+1)H}^{2/3}H^{1/3}}{\color{black}\rlr} \right) \nn \\
        & \quad + \mathcal{\Tilde{O}}\left(\frac{\Delta_{R,nH, (n+1)H}^{2/3} H^{1/3}}{\color{black}\clr} \right) + \mathcal{\Tilde{O}}\left(\frac{\Delta_{P,nH, (n+1)H}^{2/3} H^{1/3}}{\color{black}\clr} \right) + \mathcal{\Tilde{O}}\left(\frac{\alr^2 H}{\clr^2} \right), \nonumber
    \end{align}
    where $\mathcal{\Tilde{O}}(\cdot)$ hides constants and logarithmic terms.

\end{proof}

%%%%%%%%%%%%%%%%%%%
%%%%%%%%%%%%%%%%%%%

\newpage
\subsection{\textbf{Average Reward Estimation}} \label{app:avgReward}
In this section, we first analyze the gap between the average rewards and the rewards accumulate by \algoName~in \Cref{proposition:sampledRewardRegret}. We then characterize the error in the average reward estimation in \Cref{proposition:averageRewardEstimate}.

\begin{proposition} \label{proposition:sampledRewardRegret}
    {\color{black} For any $n \in [N]$}, if \Cref{assumption:ergodic} is satisfied, then the following holds true
    \begin{align*}
        \sum_{t=nH+\mt_H}^{\color{black} nH+H-1} \mathbb{E}\left[  J_t^{\polv_t} - \r_t(\s_t, \a_t) \right] \leq D_1 \alr (\mt_H+1)^2({\color{black}H-\mt_H}) + D_2 (\mt_H+1)^2\Delta_{\color{black}P,nH,(n+1)H}
    \end{align*}
    where $D_1 = L_\polv B_2 + 4U_R\sqrt{|\gS||\gA|} B_2 + 4U_R$, $D_2 = 4U_R+L_P$ and $\Delta_{\color{black} P, nH,(n+1)H} = \sum_{t=nH}^{(n+1)H} \Vert\pv_{t+1}-\pv_t\Vert_\infty$.
\end{proposition}
\begin{proof}
    Given time indices $t > \mt > 0$, recall the auxiliary Markov chain starting from $\s_{t-\mt}$ constructed by conditioning on $\s_{t-\mt}, \polv_{t-\mt-1}, \pv_{t-\mt}$ and rolling out by applying $\polv_{t-\mt-1}, \pv_{t-\mt}$ as 
    \begin{equation}
        \s_{t-\mt} \xrightarrow{\polv_{t-\mt-1}} \a_{t-\mt} \xrightarrow{\pv_{t-\mt}} \Tilde{\s}_{t-\mt+1} \xrightarrow{\polv_{t-\mt-1}} \Tilde{\a}_{t-\mt+1} \xrightarrow{\pv_{t-\mt}} \dots \Tilde{\s}_t \xrightarrow{\polv_{t-\mt-1}} \Tilde{\a}_t \xrightarrow{\pv_{t-\mt}} \Tilde{\s}_{t+1} \xrightarrow{\polv_{t-\mt-1}} \Tilde{\a}_{t+1}. \nonumber
    \end{equation}
    
    Also, recall that the original Markov chain is
    \begin{equation}
        \s_{t-\mt} \xrightarrow{\polv_{t-\mt-1}} \a_{t-\mt} \xrightarrow{\pv_{t-\mt}} \s_{t-\mt+1} \xrightarrow{\polv_{t-\mt}} \a_{t-\mt+1} \xrightarrow{\pv_{t-\mt+1}} \dots \s_t \xrightarrow{\polv_{t-1}} \a_t \xrightarrow{\pv_{t}} \s_{t+1} \xrightarrow{\polv_{t}} \a_{t+1}. \nonumber
    \end{equation}

    Further, recall $J^{\polv_{t-\mt-1}, \pv_{t-\mt}, \rv_t}:= \sum_{\s, \a} d^{\polv_{t-\mt-1}, \pv_{t-\mt}}(\s)\polv_{t-\mt-1}(\a|\s)\r_t(\s, \a)$. 
    
    We start by decomposing the term as
    \begin{align}
        \mathbb{E}\left[ J_t^{\polv_t} - \r_t(\s_t, \a_t)\right] & = \underbrace{\mathbb{E}\left[ J_t^{\polv_t} - J^{\polv_{t-\mt-1}, \pv_{t-\mt}, \rv_t} \right]}_{I_1} + \underbrace{\mathbb{E}\left[\r_t(\Tilde{\s}_t, \Tilde{\a}_t) - \r_t(\s_t, \a_t)\right]}_{I_2} + \underbrace{\mathbb{E}\left[ J^{\polv_{t-\mt-1}, \pv_{t-\mt}, \rv_t} - \r_t(\Tilde{\s}_t, \Tilde{\a}_t)\right]}_{I_3} \label{eqn:propositionSampledRewardRegret}. \nonumber\\
    \end{align}
    Note that $I_1$ is the difference in the average rewards between the two policies $\polv_t, \polv_{t-\mt-1}$ in two different environments $(\pv_{t}, \rv_t)$ and $(\pv_{t-\mt}, \rv_t)$ that share the same reward function. Hence, using \Cref{tlemma:avgRewardLipschitz} and \Cref{tlemma:consecutiveTimePolicyDiff} successively, we get
    \begin{align}
        I_1 & \leq \mathbb{E}\left[L_{\polv}\Vert\polv_t - \polv_{t-\mt-1}\Vert_2 + L_P\Vert\pv_t-\pv_{t-\mt}\Vert_\infty \right] \nn \\
        & \leq \mathbb{E}\left[L_\polv\sum_{i=t-\mt}^t \Vert\polv_i - \polv_{i-1}\Vert_2 + L_P\sum_{i=t-\mt+1}^t\Vert\pv_i-\pv_{i-1}\Vert_\infty\right] \nn \\
        & \leq L_\polv B_2 \alr (\mt+1) + L_P\Delta_{P, t-\mt+1, t},
    \end{align}
    where $\Delta_{P, t-\mt+1, t} = \sum_{i=t-\mt+1}^t \Vert\pv_i-\pv_{i-1}\Vert_\infty$.

    For $I_2$, by \Cref{plemma:dTVauxO} and \Cref{alemma:policySum} successively, we get 
    \begin{align}
        I_2 & \leq 2U_R \cdot 2\dtv\left(\p((\s_t, \a_t) \in \cdot | \mathcal{F}_{t-\mt}), \p((\Tilde{\s}_t, \Tilde{\a}_t) \in \cdot | \mathcal{F}_{t-\mt}) \right) \nonumber\\
        & \leq 4U_R\sqrt{|\gS||\gA|}\rvert\mathbb{E}\left[\sum_{i=t-\mt}^t \Vert\polv_i-\polv_{t-\mt-1}\Vert_2 \Big\lvert \mathcal{F}_{t-\mt} \right] + 4U_R\sum_{i=t-\mt}^t \Vert\pv_i-\pv_{t-\mt}\Vert_\infty \nonumber\\
        & \leq 4U_R\sqrt{|\gS||\gA|} B_2 \alr (\mt+1)^2 + 4U_R\mt\Delta_{P, t-\mt+1, t}.
    \end{align}

    Finally, we bound $I_3$ using \Cref{alemma:sampleAvgReward4} as
    \begin{align}
        I_3 \leq 4U_R m \rho^\mt.
    \end{align}

    Plugging the bounds on $I_1, I_2, I_3$ into \Cref{eqn:propositionSampledRewardRegret} and setting $\mt = \mt_H = \min\{i \geq 0 | m \rho^{i-1} \leq \min\{\alr, \clr\} \}$, 
    \begin{align*}
        & \sum_{t=nH+\mt_H}^{nH+H-1} \mathbb{E}\left[  J_t^{\polv_t} - \r_t(\s_t, \a_t)  \right] \\ 
        & \leq \sum_{t=nH+\mt_H}^{nH+H-1} L_\polv B_2 \alr (\mt_H+1) + L_P\Delta_{P, t-nH-\mt_H+1, t} + 4U_R\sqrt{|\gS||\gA|} B_2 \alr (\mt_H+1)^2 + 4U_R\mt_H\Delta_{P, t-nH-\mt_H+1, t} + 4U_R m \rho^{\mt_H} \\
        & \leq (L_\polv + 4U_R\sqrt{|\gS||\gA|})B_2 \alr (\mt_H+1)^2 (H-\mt_H) + (4U_R+L_P)(\mt_H+1)^2\Delta_{P, nH,(n+1)H} + 4U_R\alr(H-\mt_H).
    \end{align*}
\end{proof}

%%%%%%%%%%%%%%%%%%%

\begin{proposition} \label{proposition:averageRewardEstimate}
    {\color{black} For any $n \in [N]$}, if \Cref{assumption:ergodic} holds and $0 < \rlr < 1/2$, then we have the following 
        \begin{align*}
            \sum_{t=nH+\mt_H}^{nH+H-1} \mathbb{E}\left[(J_t^{\polv_t} - \re_t)^2 \right] & \leq \frac{4U_R^2}{\rlr} + 2B_7\alr(\mt_H+1)^2{\color{black}H} + D_3 \rlr \mt_H {\color{black}H} + 2B_8 (\mt_H+1)^2 \Delta_{\color{black} P, nH, (n+1)H} \\ & \quad + \frac{D_4 \alr^2 {\color{black}H}}{\rlr^2} + {\color{black} \frac{4W_3\Delta_{R, nH, (n+1)H}^{2/3}H^{1/3}}{\rlr} + \frac{4W_4 \Delta_{P, nH, (n+1)H}^{2/3} H^{1/3}}{\rlr}} 
        \end{align*}
        where $D_3 = 4U_R F_6+8U_R^2$, $D_4 = 9L_\polv^2 B_2^2$, $\Delta_{\color{black}R, nH, (n+1)H} = \sum_{t=nH}^{(n+1)H} \Vert\rv_{t+1}-\rv_t\Vert_\infty$, $\Delta_{\color{black}P, nH, (n+1)H} = \sum_{t=nH}^{(n+1)H} \Vert\pv_{t+1}-\pv_t\Vert_\infty$, $W_3 = (3)^{1/3}(4U_R^2)^{2/3}$ and $W_4 = (3L_P^2)^{1/3}(4U_R^2)^{2/3}$.
\end{proposition}
\begin{proof}
    Recall that $\phi_t := \re_t - J_t^{\polv_t}$. Using the average reward update equation (line~\ref{line:avgRewardUpdate} in \Cref{algo}), we have
    \begin{align*}
        \phi_{t+1}^2 & = \left(\re_t + \rlr(\r_t(\s_t, \a_t)-\re_t) - J_{t+1}^{\polv_{t+1}} \right)^2 \\
        & = \left(\phi_t + J_t^{\polv_t} - J_{t+1}^{\polv_{t+1}} + \rlr(\r_t(\s_t, \a_t)-\re_t) \right)^2 \\
        & \leq \phi_t^2 + 2\rlr\phi_t(\r_t(\s_t, \a_t) - \re_t) + 2\phi_t(J_t^{\polv_t} - J_{t+1}^{\polv_{t+1}}) + 2(J_t^{\polv} - J_{t+1}^{\polv_{t+1}})^2 + 2\rlr^2(\r_t(\s_t, \a_t) - \re_t)^2 \\
        & = (1-2\rlr)\phi_t^2 + 2\rlr\phi_t(\r_t(\s_t, \a_t) - J_t^{\polv_t}) + 2\phi_t(J_t^{\polv_t} - J_{t+1}^{\polv_{t+1}}) \\ & \quad + 2(J_t^{\polv} - J_{t+1}^{\polv_{t+1}})^2 + 2\rlr^2(\r_t(\s_t, \a_t) - \re_t)^2 \\
        & = (1-2\rlr)\phi_t^2 + 2\rlr\Lambda(\polv_t, \pv_t, \rv_t, \re_t, O_t) + 2\phi_t(J_t^{\polv_t} - J_{t+1}^{\polv_{t+1}}) \\ & \quad + 2(J_t^{\polv} - J_{t+1}^{\polv_{t+1}})^2 + 2\rlr^2(\r_t(\s_t, \a_t) - \re_t)^2 \\
        & = (1-2\rlr)\phi_t^2 + 2\rlr\Lambda(\polv_t, \pv_t, \rv_t, \re_t, O_t) {\color{black} + 2\phi_t(J_t^{\polv_t} - J_{t+1}^{\polv_t}) + 2\phi_t(J_{t+1}^{\polv_t}- J_{t+1}^{\polv_{t+1}})} \\ & \quad + 2(J_t^{\polv} - J_{t+1}^{\polv_{t+1}})^2 + 2\rlr^2(\r_t(\s_t, \a_t) - \re_t)^2.
    \end{align*}

    Rearranging and setting $\mt = \mt_H = \min\{i \geq 0 | m \rho^{i-1} \leq \min\{\alr, \clr\} \}$, we have 
    \begin{align*}
        \sum_{t=nH+\mt_H}^{nH+H-1} \mathbb{E}[\phi_t^2] & \leq \underbrace{\sum_{t=nH+\mt_H}^{nH+H-1} \frac{\mathbb{E}[\phi_t^2 - \phi_{t+1}^2]}{2\rlr}}_{I_1} + \underbrace{\sum_{t=nH+\mt_H}^{nH+H-1} \mathbb{E}[\Lambda(\polv_t, \pv_t, \rv_t, \re_t, O_t)]}_{I_2} + {\color{black} \underbrace{\sum_{t=nH+\mt_H}^{nH+H-1} \frac{\mathbb{E}[\phi_t (J_t^{\polv_t} - J_{t+1}^{\polv_{t}})]}{\rlr}}_{I_3}} \\ & \quad {\color{black} \underbrace{\sum_{t=nH+\mt_H}^{nH+H-1} \frac{\mathbb{E}[\phi_t (J_{t+1}^{\polv_t} - J_{t+1}^{\polv_{t+1}})]}{\rlr}}_{I_4}}  + \underbrace{\sum_{t=nH+\mt_H}^{nH+H-1} \frac{\mathbb{E}[(J_t^{\polv_t} - J_{t+1}^{\polv_{t+1}})^2]}{\rlr}}_{{\color{black} I_5}} + \underbrace{\sum_{t=nH+\mt_H}^{nH+H-1} \rlr \mathbb{E}[(\r_t(\s_t, \a_t) - \re_t)^2]}_{{\color{black} I_6}}. 
    \end{align*}

    We now analyze each of these terms starting with the first term as 
    \begin{align*}
        I_1 = \frac{\mathbb{E}[\phi_{nH+\mt_H}^2 - \phi_{nH+H}^2]}{2\rlr} \leq \frac{2U_R^2}{\rlr}.
    \end{align*}

    By \Cref{tlemma:avgRewardLambda} and the average reward update equation, we have
    \begin{align*}
        I_2 & \leq \sum_{t=nH+\mt_H}^{nH+H-1} B_7 \alr (\mt_H+1)^2 + F_6 \lvert \re_t - \re_{t-nH-\mt_H}\rvert + F_7 \mt \Delta_{P, t-nH-\mt_H+1, t} \\
        & \leq B_7\alr(\mt_H+1)^2(H-\mt_H) + 2U_R F_6 \rlr \mt_H (H-\mt_H) + B_8 (\mt_H+1)^2 \Delta_{P, nH, (n+1)H}.
    \end{align*}

    {\color{black} 
    By \Cref{tlemma:avgRewardLipschitz}, we have 
    \begin{align*}
        I_3 \leq \frac{2U_R}{\rlr} \left(\sum_{t=nH+\mt_H}^{nH+H-1} \Vert\rv_{t+1}-\rv_t\Vert_\infty + L_P\Vert\pv_{t+1}-\pv_t\Vert_\infty\right) \leq \frac{2U_R \Delta_{R, nH, (n+1)H}}{\rlr} + \frac{2U_R \Delta_{P, (n+1)H}}{\rlr}.
    \end{align*}
    
    By \Cref{tlemma:avgRewardLipschitz} and Cauchy-Schwartz inequality, we have
    \begin{align*}
        I_4 & \leq \left(\sum_{t=nH+\mt_H}^{nH+H-1} \mathbb{E}[\phi_t^2] \right)^{1/2} \left(\sum_{t=nH+\mt_H}^{nH+H-1} \frac{\mathbb{E}[(J_{t+1}^{\polv_{t+1}} - J_t^{\polv_t})^2]}{\rlr^2} \right)^{1/2} \\
        & \leq \left(\sum_{t=nH+\mt_H}^{nH+H-1} \mathbb{E}[\phi_t^2] \right)^{1/2} \left(\frac{L_{\polv}^2 B_2^2 \alr^2 H}{\rlr^2} \right)^{1/2}.
    \end{align*}
    }

    {\color{black}
    We now bound $I_5$ as follows. For timesteps with small changes in the environment, we use \Cref{tlemma:avgRewardLipschitz}, and for timesteps with large changes in the environment, we use a naive upper bound. Define the set of timesteps $\mathcal{T}_J := \{t : \Vert \rv_{t+1}-\rv_t\Vert_\infty \leq \delta_R, \Vert \pv_{t+1}-\pv_t\Vert_\infty \leq \delta_P \}$.

    \begin{align}
        I_5 = \sum_{t=nH+\mt_H}^{nH+H-1} \frac{\mathbb{E}\left[(J_{t+1}^{\polv_{t+1}} - J_t^{\polv_t})^2 \right]}{\rlr} & \overset{}{\leq} \sum_{t \in \mathcal{T}_J} \frac{\mathbb{E}\left[(J_{t+1}^{\polv_{t+1}} - J_t^{\polv_t})^2 \right]}{\rlr} + \sum_{t \notin \mathcal{T}_J} \frac{4U_R^2}{\rlr} \nonumber\\
        & \overset{(a)}{\leq} \sum_{t \in \mathcal{T}_J} \frac{3 \delta_R^2}{\rlr} + \frac{3L_P^2 \delta_P^2}{\rlr} + \frac{3L_\polv^2 B_2^2 \alr^2}{\rlr} + \sum_{t\notin \mathcal{T}_J} \frac{4U_R^2}{\rlr} \nonumber\\
        & \overset{(b)}{\leq} \frac{3 \delta_R^2 H}{\rlr} + \frac{3L_P^2 \delta_P^2 H}{\rlr} + \frac{3L_\polv^2 B_2^2 \alr^2 H}{\rlr} + \frac{4U_R^2 \Delta_{R,nH,(n+1)H}}{\rlr\delta_R} + \frac{4U_R^2 \Delta_{P,nH, (n+1)H}}{\rlr\delta_P} \nonumber\\
        & \overset{}{\leq} \frac{W_3 \Delta_{R,nH,(n+1)H}^{2/3} T^{1/3}}{\rlr}  + \frac{W_4 \Delta_{P, nH,(n+1)H}^{2/3} T^{1/3}}{\rlr} + \frac{3L_\polv^2 B_2^2 \alr^2 T}{\rlr} \label{eq:instanceIndependentJTimeDiff}
    \end{align}
    where $(a)$ follows from \Cref{tlemma:avgRewardLipschitz} and $(b)$ is obtained by choosing $\delta_R = \left(\frac{4U_R^2 \Delta_{R, nH, (n+1)H}}{3 T}\right)^{1/3}$ and $\delta_P = \left(\frac{4U_R^2 \Delta_{P, nH,(n+1)H}}{3L_P^2 T}\right)^{1/3}$ and defining $W_3 = (3)^{1/3}(4U_R^2)^{2/3}$, $W_2 = (3L_P^2)^{1/3} (4U_R^2)^{2/3}$.
    }

    For the final term, we have
    \begin{align*}
        {\color{black} I_6} \leq 4U_R^2 \rlr (H-\mt_H).
    \end{align*}

    Putting everything together, we have 
    \begin{align*}
        \sum_{t=nH+\mt_H}^{nH+H-1} \mathbb{E}[\phi_t^2] & \leq \frac{2U_R^2}{\rlr} + B_7\alr(\mt_H+1)^2(H-\mt_H) + 2U_R F_6 \rlr \mt_H (H-\mt_H) + B_8 (\mt_H+1)^2 \Delta_{P, nH,(n+1)H} \\ & \quad {\color{black} + \frac{2U_R \Delta_{R, nH,(n+1)H}}{\rlr} + \frac{2U_R \Delta_{P, nH,(n+1)H}}{\rlr} + \left(\sum_{t=nH+\mt_H}^{nH+H-1} \mathbb{E}[\phi_t^2] \right)^{1/2} \left( \frac{L_\polv^2 B_2^2 \alr^2 H}{\rlr^2} \right)^{1/2}} \\ & \quad + \frac{W_3 \Delta_{R,nH,(n+1)H}^{2/3} H^{1/3}}{\rlr}  + \frac{W_4 \Delta_{P, nH,(n+1)H}^{2/3} H^{1/3}}{\rlr} + \frac{3L_\polv^2 B_2^2 \alr^2 H}{\rlr} + 4U_R^2 \rlr (H-\mt_H).
    \end{align*}

    Now, we use the squaring trick from Section C.3, \citet{Wu2020A2CPG}. The above equation is of the form, $X \leq Y + Z\sqrt{X}$. Completing the squares and rearranging, we get $X \leq 2Y  + Z^2$. Hence,
    \begin{align*}
        \sum_{t=nH+\mt_H}^{nH+H-1} \mathbb{E}[\phi_t^2] & \leq \frac{4U_R^2}{\rlr} + 2B_7\alr(\mt_H+1)^2(H-\mt_H) + 4U_R F_6 \rlr \mt_H (H-\mt_H) + 2B_8 (\mt_H+1)^2 \Delta_{P, nH,(n+1)H} \\ & \quad {\color{black} + \frac{4U_R \Delta_{R, nH,(n+1)H}}{\rlr} + \frac{4U_R \Delta_{P, nH, (n+1)H}}{\rlr} + \frac{9L_\polv^2 B_2^2 \alr^2 H}{\rlr^2}} \\ & \quad {\color{black} + \frac{2W_3\Delta_{R, nH,(n+1)H}^{2/3}H^{1/3}}{\rlr} + \frac{2W_4\Delta_{P, nH,(n+1)H}^{2/3}H^{1/3}}{\rlr}} + 8U_R^2 \rlr (H-\mt_H) \\ 
        & \leq \frac{4U_R^2}{\rlr} + 2B_7\alr(\mt_H+1)^2(H-\mt_H) + 4U_R F_6 \rlr \mt_H (H-\mt_H) + 2B_8 (\mt_H+1)^2 \Delta_{P, nH,(n+1)H} \\ & \quad {\color{black} + \frac{9L_\polv^2 B_2^2 \alr^2 H}{\rlr^2} + \frac{4W_3\Delta_{R, nH,(n+1)H}^{2/3}H^{1/3}}{\rlr} + \frac{4W_4\Delta_{P, nH,(n+1)H}^{2/3}H^{1/3}}{\rlr}} + 8U_R^2 \rlr (H-\mt_H).
    \end{align*}

\end{proof}

%%%%%%%%%%%%%%%%%%%
%%%%%%%%%%%%%%%%%%%

\subsection{\textbf{Technical Lemmas}} \label{app:tLemmas}

\subsubsection{\textbf{Actor}}
\begin{lemma} \label{tlemma:actor}
    If \Cref{assumption:ergodic} holds, for any $t, t' \geq 0$, we have 
    \begin{align*}
        \sum_s d^{\polv_{t'}^\star, \pv_{t'}}(\s) \left[\frac{\log Z_t(\s)}{\alr} - \v_t^{\polv_t}(\s) \right] & \leq \frac{J_{t+1}^{\polv_{t+1}} - J_t^{\polv_t}}{\concr} + \Vert\qv_t^{\polv_t} - \qv_t \Vert_\infty + \frac{B_1 \alr}{\concr} + \frac{\Vert \rv_{t+1}-\rv_{t}\Vert_\infty}{\concr} + \frac{L_P \Vert\pv_{t+1}-\pv_t\Vert_\infty}{\concr}
    \end{align*}
    where $Z_t(\s) = \sum_{\a' \in \mathcal{A}} \pol_t(\a'|\s) \exp(\alr\q_t(\s, \a'))$, $C$ is defined in \Cref{assumption:ergodic} and other constants in \Cref{app:symbolReference}. 
\end{lemma}
\begin{proof}
    We have
    \begin{align}
        & J_{t+1}^{\polv_{t+1}} - J_t^{\polv_t} \nn \\
        & = J_{t+1}^{\polv_{t+1}} - J_t^{\polv_{t+1}} + J_t^{\polv_{t+1}} - J_t^{\polv_t} \nonumber\\
        & \overset{(a)}{=} J_{t+1}^{\polv_{t+1}} - J_t^{\polv_{t+1}} + \sum_{\s, \a} d^{\polv_{t+1}, \pv_t}(\s) \pol_{t+1}(\a|\s) \left[\q_t^{\polv_t}(\s, \a) - \v_t^{\polv_t}(\s) + \q_t(\s, \a) - \q_t(\s, \a) \right] \nonumber \\ 
        & \overset{(b)}{=} J_{t+1}^{\polv_{t+1}} - J_t^{\polv_{t+1}} + \sum_{\s, \a} d^{\polv_{t+1}, \pv_t}(\s) \pol_{t+1}(\a|\s) \Bigg[\q_t^{\pol_t}(\s, \a) - \v_t^{\polv_t}(\s) \nonumber \\ & \quad + \frac{\log Z_t(\s)}{\alr} + \frac{1}{\alr}\log\frac{\pol_{t+1}(\a|\s)}\pol_t(\a|\s) - \q_t(\s, \a) \Bigg] \nonumber \\
        & \overset{(c)}{\geq} J_{t+1}^{\polv_{t+1}} - J_t^{\polv_{t+1}} + \sum_{\s, \a} d^{\polv_{t+1}, \pv_t}(\s) \pol_{t+1}(\a|\s) \Bigg[\q_t^{\pol_t}(\s, \a) - \v_t^{\polv_t}(\s) + \frac{\log Z_t(\s)}{\alr} - \q_t(\s, \a) \Bigg] \nonumber \\
        & \geq J_{t+1}^{\polv_{t+1}} - J_t^{\polv_{t+1}} + \underbrace{\sum_{\s, \a} d^{\polv_{t+1}, \pv_t}(\s) \pol_{t}(\a|\s)\Bigg[\frac{\log Z_t(\s)}{\alr} - \q_t(\s, \a)\Bigg] }_{I_1} \nonumber\\ 
        & \quad + \underbrace{\sum_{\s, \a} d^{\polv_{t+1}, \pv_{t}}(\s) (\pol_{t+1}(\a|\s) - \pol_t(\a|\s))\left[\q_t^{\polv_t}(\s, \a) - \q_t(\s, \a) \right]}_{I_2} \label{eqn:tlemmaActor1}
    \end{align}
    where $(a)$ follows from \Cref{tlemma:perfDiff}, $(b)$ follows from the actor update (line~\ref{line:actorUpdate} in \Cref{algo}), and $(c)$ is due to the non-negativity of KL-Divergence. 
    
    Next, we bound the last two terms in (\ref{eqn:tlemmaActor1}). Under \Cref{assumption:ergodic}, we have 
    \begin{align}
        I_1 & = \sum_{\s, \a} d^{\polv_{t'}^\star, \pv_{t'}}(\s) \left(\frac{d^{\polv_{t+1}, \pv_t}(\s)}{d^{\polv_{t'}^\star, \pv_{t'}}(\s)}\right) \pol_{t}(\a|\s)\Bigg[\frac{\log Z_t(\s)}{\alr} - \q_t(\s, \a)\Bigg] \nonumber\\
        & \overset{}{\geq} \concr \sum_{\s, \a} d^{\polv_{t'}^\star, \pv_{t'}}(\s) \pol_{t}(\a|\s)\Bigg[\frac{\log Z_t(\s)}{\alr} - \q_t(\s, \a)\Bigg] \nonumber\\
        & \overset{}{\geq} \concr\sum_{\s} d^{\polv_{t'}^\star, \pv_{t'}}(\s) \left[\frac{\log Z_t(\s)}{\alr} - \v_t^{\polv_t}(\s) \right] + \concr\sum_{\s, \a} d^{\polv_{t'}^\star, \pv_{t'}}(\s) \pol_t(\a|\s)\left[\q_t^{\polv_t}(\s, \a) - \q_t(\s, \a) \right] \label{eqn:tlemmaActor2}
    \end{align}

    Further, we have by 1-Lipschitzness of the tabular softmax policy
    \begin{align}
        I_2 \geq -2U_Q \sum_{\s, \a} d^{\polv_{t+1}, \pv_{t}}(\s) (\pol_{t+1}(\a|\s) - \pol_t(\a|\s)) \geq -2U_Q \cdot \alr U_Q \sqrt{\lvert\mathcal{A}\rvert} \geq - B_1 \alr \label{eqn:tlemmaActor3}.
    \end{align}
    
    Plugging the bounds from (\ref{eqn:tlemmaActor2}) and (\ref{eqn:tlemmaActor3}) into (\ref{eqn:tlemmaActor1}) and rearranging, we have
    \begin{align*}
        & \sum_s d^{\polv_{t'}^\star, \pv_{t'}}(\s) \left[\frac{\log Z_t(\s)}{\alr} - \v_t^{\polv_t}(\s) \right] \\ & \leq \frac{J_{t+1}^{\polv_{t+1}} - J_t^{\polv_t}}{C} + \sum_{\s, \a} d^{\polv_{t'}^\star, \pv_{t'}}(\s) \pol_t(\a|\s)\left[\q_t(\s, \a) - \q_t^{\polv_t}(\s, \a) \right] + \frac{B_1 \alr}{\concr} + \frac{J_t^{\polv_{t+1}} - J_{t+1}^{\polv_{t+1}}}{\concr} \\
        & \leq \frac{J_{t+1}^{\polv_{t+1}} - J_t^{\polv_t}}{C} + \Vert \qv_t^{\polv_t} - \qv_t\Vert_\infty + \frac{B_1 \alr}{\concr} + \frac{\Vert \rv_{t+1} - \rv_t \Vert_\infty}{\concr} + \frac{L_P \Vert \pv_{t+1} - \pv_t \Vert_\infty}{\concr}
    \end{align*}
    where the last inequality follows from \Cref{tlemma:avgRewardLipschitz}.
\end{proof}

%%%%%%%%%%%%%%%%%%%

\begin{lemma} \label{tlemma:consecutiveTimePolicyDiff}
    For $t \geq 0$, policy $\polv_t$ satisfies
    \begin{align}
        \Vert\polv_{t+1} - \polv_t\Vert_2 \leq B_2\alr \nonumber
    \end{align}
    where $B_2 = U_Q$. 
    % \neharika{DOUBLE CHECK} \ps{Looks good. }
\end{lemma}
\begin{proof}

    By 1-Lipschitzness of the softmax parameterization of the actor \citep{beck2017first} and \Cref{plemma:maxValues}, we have 
    \begin{align}
        \Vert \polv_{t+1} - \polv_t \Vert_2 \leq  \Vert \alr \qv_t\Vert_2 \leq \alr U_Q. \nonumber
    \end{align}
\end{proof}

%%%%%%%%%%%%%%%%%%%

\begin{lemma}[Average-Reward Performance Difference Lemma \citep{murthy2023convergence}]
\label{tlemma:perfDiff}
    The average rewards for any two policies $\polv, \polv'$ at time $t$ satisfy
    \begin{align*}
        J_t^{\polv} - J_t^{\polv'} = \sum_{\s \in \mathcal{S}} d^{\polv, \pv_t}(\s) \sum_{\a \in \mathcal{A}} \polv(\a|\s) \left[\q_t^{\polv'}(\s, \a) - \v_t^{\polv'}(\s) \right].
    \end{align*}
\end{lemma}

%%%%%%%%%%%%%%%%%%%

\begin{lemma} \label{tlemma:bestAvgRewardDiff}
    For any $t, t' \geq 0$, it holds that 
    \begin{align*}
        J_t^{\polv_t^\star} - J_{t'}^{\polv_{t'}^\star} \leq \Vert\rv_t-\rv_{t'}\Vert_\infty + U_Q\Vert\pv_t-\pv_{t'}\Vert_\infty.
    \end{align*}
    where $\polv_t^\star$ represents the optimal policy for MDP $\mathcal{M}_t(\mathcal{S}, \mathcal{A}, \pv_t, \rv_t)$. 
\end{lemma}
\begin{proof}
    Consider the linear programming formulation of an MDP $\mathcal{M}(\mathcal{S}, \mathcal{A}, \pv, \rv)$ \citep{Puterman}
    \begin{equation*}
        \min\limits_{J, V(\s)} J
    \end{equation*}
    \begin{equation} \label{eqn:optimalJLP}
        \hbox{such that } J + \v(\s) \geq r(\s, \a) + \sum_{\s'}\p(\s'|\s, \a)\v(\s') \hbox{ } \forall \s \in \mathcal{S}, \a \in \mathcal{A}.
    \end{equation}

    If the optimal solution for $\mathcal{M}_{t'}(\mathcal{S}, \mathcal{A}, \pv_{t'}, \rv_{t'})$ is $J_{t'}^\star, \vv_{t'}^\star$, we have
    \begin{equation*}
        J_{t'}^\star \bm{1} \geq \rv_{t'} + (\pv_{t'} - \mathbf{I})\vv_{t'}^\star.
    \end{equation*}

    Now for $\mathcal{M}_{t}(\mathcal{S}, \mathcal{A}, \pv_{t}, \rv_{t})$, we know
    \begin{align}
        J_{t}^\star & \leq \Vert\rv_{t} + (\pv_{t} - \mathbf{I})\vv_{t'}^\star\Vert_\infty \nonumber \\
        & \leq \Vert\rv_{t'} + (\pv_{t'} - \mathbf{I})\vv_{t'}^\star + (\rv_{t} - \rv_{t'}) + (\pv_{t}-\pv_{t'})\vv_{t'}^\star\Vert_\infty \nonumber\\
        & \leq \Vert J_{t'}^\star\bm{1} \Vert_\infty + \Vert \rv_{t} - \rv_{t'} \Vert_\infty + \Vert (\pv_{t}-\pv_{t'})\vv_{t'}^\star\Vert_\infty. \nonumber
    \end{align}
    Hence, we have
    \begin{align}
        J_{t}^\star - J_{t'}^\star & \leq \Vert \rv_{t} - \rv_{t'}\Vert_\infty + \Vert\left(\pv_{t}-\pv_{t'}\right)\vv_{t'}^\star\Vert_\infty \nonumber\\
        J_t^{\polv_t^\star} - J_{t'}^{\polv_{t'}^\star} & \leq \Vert \rv_{t} - \rv_{t'}\Vert_\infty + U_Q\Vert\pv_{t}-\pv_{t'}\Vert_\infty \nonumber.
    \end{align}
\end{proof}

%%%%%%%%%%%%%%%%%%%

\begin{lemma} \label{tlemma:avgRewardLipschitz}
    There exist constants $L_\polv = 4U_R(M+1)\sqrt{|\gS||\gA|}$ and $L_P = 4U_R M$ such that for all policies $\polv, \polv'$ and timesteps $t, t'$, it holds that 
    \begin{align*}
        J_{t}^{\polv} - J_{t'}^{\polv'} & \leq L_\pol\Vert\polv-\polv'\Vert_2 + \Vert\rv_{t}-\rv_{t'}\Vert_\infty + L_P\Vert\pv_t-\pv_{t'}\Vert_\infty.
    \end{align*}
\end{lemma}
\begin{proof}
    \begin{align}
        J_t^\polv - J_{t'}^{\polv'} & = \underbrace{J_t^\polv - J_t^{\polv'}}_{T_1} + \underbrace{J_t^{\polv'} - J_{t'}^{\polv'}}_{T_2}, \label{eq:tlemma:avgReward:diffPolicyAndDynamics}
    \end{align}
    where $T_1$ is the difference in the average rewards between two policies $\polv, \polv'$ under the same environments $(\rv_{t}, \pv_t)$, while $T_2$ is the difference in the average rewards with the same policy $\polv'$, but under two different environments $(\rv_{t}, \pv_t)$ and $(\rv_{t'}, \pv_{t'})$.
    \begin{align}
        T_1 & = J_t^\polv - J_t^{\polv'} = \mathbb{E}_{\s \sim d^{\polv, \pv_t}, \a \sim \polv, \s'\sim d^{\polv', \pv_t}, \a'\sim\polv'}[\r_t(\s, \a) - \r_t(\s', \a')] \nonumber\\
        & = 4U_R\dtv\left(d^{\polv, \pv_t}\otimes\polv, d^{\polv', \pv_t}\otimes\polv'\right) \nonumber\\
        & \overset{(a)}{\leq}  L_\polv\Vert\polv-\polv'\Vert_2, \label{eq:tlemma:avgReward:sameDynamics}
    \end{align}
    where $(a)$ follows from \Cref{plemma:dTV}, where $\otimes$ denotes the Kronecker product. Next, we bound $T_2$.
    \begin{align}
        T_2 & = J_t^{\polv'} - J_{t'}^{\polv'} = \sum_{\s, \a} d^{\polv', \pv_t}(\s)\polv'(\a|\s)\r_t(\s, \a) - d^{\polv', \pv_{t'}}(\s)\polv'(\a|\s)\r_{t'}(\s, \a) \nonumber\\
        & \leq \sum_{\s, \a} \left\lvert d^{\polv', \pv_t}(\s)\polv'(\a|\s)\r_t(\s, \a) - d^{\polv', \pv_t}(\s)\polv'(\a|\s)\r_{t'}(\s, \a) \right\rvert \nonumber\\ 
        & \quad + \sum_{\s, \a} \left\lvert d^{\polv', \pv_t}(\s)\polv'(\a|\s)\r_{t'}(\s, \a) - d^{\polv', \pv_{t'}}(\s)\polv'(\a|\s)\r_{t'}(\s, \a) \right\rvert \nonumber\\
        & \leq \Vert\rv_t - \rv_{t'}\Vert_\infty + 4U_R\dtv(d^{\polv', \pv_t}\otimes\polv', d^{\polv', \pv_{t'}}\otimes\polv') \nonumber\\
        & \overset{(b)}{\leq} \Vert\rv_{t}-\rv_{t'}\Vert_\infty + L_P\Vert\pv_t-\pv_{t'}\Vert_\infty \label{eq:tlemma:avgReward:samePolicy}
    \end{align}
    where $(b)$ also follows from \Cref{plemma:dTV}. Substituting the bounds from (\ref{eq:tlemma:avgReward:sameDynamics}) and (\ref{eq:tlemma:avgReward:samePolicy}) into (\ref{eq:tlemma:avgReward:diffPolicyAndDynamics}), we get the result.
\end{proof}

%%%%%%%%%%%%%%%%%%%

% \newpage
\subsubsection{\textbf{Critic}}

\begin{lemma} \label{tlemma:QPolicyLipschitz}
    For any policies $\polv, \polv'$, we have 
    \begin{align}
        \Vert \qv_t^\polv - \qv_t^{\polv'} \Vert_2 \leq G_\pol \Vert \polv - \polv' \Vert_2 \nonumber
    \end{align}
    where $G_\pol = 2U_Q\sqrt{|\gS||\gA|}$. 
\end{lemma}
\begin{proof}
    \begin{align}
        \q_t^{\polv}(\s, \a) & \overset{(a)}{=} \r_t(\s, \a) - J_t^\polv + \mathbb{E}_{\s'\sim\p_t(\cdot|\s,\a)}\left[\v_t^\polv(\s') \right] \nonumber\\
        \Rightarrow \frac{\partial \q_t^\polv(\s, \a)}{\partial \polv} & = \frac{-\partial J_t^\polv}{\partial\polv} + \sum_{\s'\in\mathcal{S}}\p_t(\s'|\s, \a)\frac{\partial\v_t^\polv(\s')}{\partial\polv} \nonumber\\
        \left\Vert \frac{\partial \q_t^\polv(\s, \a)}{\partial \polv} \right\Vert_2 & \leq 2\left\Vert \frac{\partial J_t^\polv}{\partial \polv}\right\Vert_2 \\
        % \tag{\ps{Explain more}}
        \left\Vert \frac{\partial \q_t^\polv(\s, \a)}{\partial \polv} \right\Vert_2 & \overset{(b)}{\leq} 2\left\Vert d^{\polv, \pv_t}(\s)\q_t^\polv(\s, \a)\right\Vert_2 \leq 2U_Q 
    \end{align}
    It follows from mean-value theorem that
    \begin{align*}
        \lvert \q_t^\polv(\s, \a) - \q_t^{\polv'}(\s, \a)\rvert & \leq 2U_Q \Vert\polv - \polv'\Vert_2, \text{ for all } \s, \a \\
        \Rightarrow \Vert \qv_t^\polv - \qv_t^{\polv'} \Vert_2 & \leq G_\polv \Vert\polv-\polv'\Vert_2,
    \end{align*}
    where $(a)$ is by using the Bellman equation,
    and $(b)$ follows from Policy Gradient Theorem \citep{SuttonBartoBook} and \Cref{plemma:maxValues}. 
\end{proof}

%%%%%%%%%%%%%%%%%%%

\begin{lemma} \label{tlemma:QTimeLipschitz}
    For any timesteps $t, t' \geq 0$, we have 
    \begin{align}
        \Vert{\color{black}\Pi_E}\left[\qv_{t}^\polv - \qv_{t'}^\polv\right]\Vert_2 \leq G_R \Vert\rv_t-\rv_{t'}\Vert_\infty + G_P\Vert\pv_t-\pv_{t'}\Vert_\infty \nonumber
    \end{align}
    where $G_R = 2\lambda^{-1}\sqrt{|\gS||\gA|}$ and $G_P = (\lambda^{-1}L_P + 4U_R\lambda^{-1}M + 4U_R\lambda^{-2}(M+1))\sqrt{|\gS||\gA|}$.
\end{lemma}
\begin{proof}
    Recall the diagonal matrix $D^{\polv, \pv_t} = diag\left(d^{\polv, \pv_t}(\s)\pol(\a|\s)\right)$, where $d^{\polv, \pv_t}(\cdot)$ denotes the stationary distribution induced over the states, while $\bm{1}$ denotes the all ones vector. {\color{black} $E$ denotes the subspace orthogonal to the all ones vector. Pseudo-inverse of a matrix is represented by $\X^{\dagger}$.}
    Now, we have
    \begin{align}
        & \Vert{\color{black}\Pi_E}\left[\qv_t^\polv - \qv_{t'}^\polv\right]\Vert_2 \overset{(a)}{\leq} \Vert (\Bar{\Av}^{\polv, \pv_{t}})^{\color{black}\dagger}D^{\polv, \pv_t}(J_{t}^\polv\bm{1} - \rv_t) - (\Bar{\Av}^{\polv, \pv_{t'}})^{\color{black}\dagger}D^{\polv, \pv_{t'}}(J_{t'}^\polv\bm{1} - \rv_{t'})\Vert_2 \nonumber\\
        & \leq \Vert(\Bar{\Av}^{\polv, \pv_{t}})^{\color{black}\dagger}D^{\polv, \pv_t}(J_{t}^\polv\bm{1} - \rv_t) - (\Bar{\Av}^{\polv, \pv_{t}})^{\color{black}\dagger}D^{\polv, \pv_{t'}}(J_{t'}^\polv\bm{1} - \rv_{t'})\Vert_2 \nonumber\\ & \quad + \Vert(\Bar{\Av}^{\polv, \pv_{t}})^{\color{black}\dagger}D^{\polv, \pv_{t'}}(J_{t'}^\polv\bm{1} - \rv_{t'}) - (\Bar{\Av}^{\polv, \pv_{t'}})^{\color{black}\dagger}D^{\polv, \pv_{t'}}(J_{t'}^\polv\bm{1} - \rv_{t'})\Vert_2 \nonumber\\
        & \leq \Vert(\Bar{\Av}^{\polv, \pv_{t}})^{\color{black}\dagger}\Vert_2 \left(\Vert D^{\polv,\pv_t}J_t^\polv\bm{1} - D^{\polv, \pv_{t'}}J_{t'}^\polv\bm{1} \Vert_2 + \Vert D^{\polv, \pv_t}\rv_t - D^{\polv, \pv_{t'}}\rv_{t'}\Vert_2\right) \nonumber\\
        & \quad + \Vert(\Bar{\Av}^{\polv, \pv_{t}})^{\color{black}\dagger}D^{\polv, \pv_{t'}}(J_{t'}^\polv\bm{1} - \rv_{t'}) - (\Bar{\Av}^{\polv, \pv_{t'}})^{\color{black}\dagger}D^{\polv, \pv_{t'}}(J_{t'}^\polv\bm{1} - \rv_{t'})\Vert_2 \nonumber \\
        & \overset{(b)}{\leq} \lambda^{-1}\left(\Vert D^{\polv, \pv_t}(J_t^\polv - J_{t'}^\polv)\bm{1} \Vert_2 + \Vert(D^{\polv, \pv_t} - D^{\polv, \pv_{t'}})J_{t'}^\polv \bm{1} \Vert_2  + \Vert D^{\polv, \pv_t}\rv_t - D^{\polv, \pv_{t'}}\rv_{t'} \Vert_2 \right) \nonumber\\
        & \quad + \Vert(\Bar{\Av}^{\polv, \pv_{t}})^{\color{black}\dagger}D^{\polv, \pv_{t'}}(J_{t'}^\polv\bm{1} - \rv_{t'}) - (\Bar{\Av}^{\polv, \pv_{t'}})^{\color{black}\dagger}D^{\polv, \pv_{t'}}(J_{t'}^\polv\bm{1} - \rv_{t'})\Vert_2 \nonumber \\
        & \overset{}{\leq} \lambda^{-1}\left(\sqrt{|\gS||\gA|}\Vert D^{\polv, \pv_t}\Vert_2 |J_t^\polv - J_{t'}^\polv| + \Vert D^{\polv, \pv_t} - D^{\polv, \pv_{t'}}\Vert_2 \cdot U_R\sqrt{|\gS||\gA|} + \Vert D^{\polv, \pv_t}\rv_t - D^{\polv, \pv_{t'}}\rv_{t'} \Vert_2 \right) \nonumber\\
        & \quad + \Vert(\Bar{\Av}^{\polv, \pv_{t}})^{\color{black}\dagger}D^{\polv, \pv_{t'}}(J_{t'}^\polv\bm{1} - \rv_{t'}) - (\Bar{\Av}^{\polv, \pv_{t'}})^{\color{black}\dagger}D^{\polv, \pv_{t'}}(J_{t'}^\polv\bm{1} - \rv_{t'})\Vert_2 \nonumber \\
        & \overset{(c)}{\leq} \lambda^{-1} \sqrt{|\gS||\gA|} \left( \Vert\rv_t - \rv_{t'}\Vert_\infty + L_P\Vert\pv_t-\pv_{t'}\Vert_\infty + 2U_R \dtv(d^{\polv, \pv_{t}}\otimes\polv, d^{\polv, \pv_{t'}}\otimes\polv )\right) \nonumber\\
        & \quad + \lambda^{-1}\Vert D^{\polv, \pv_t}\rv_t - D^{\polv, \pv_{t'}}\rv_{t'} \Vert_2 \\ & \quad + \Vert(\Bar{\Av}^{\polv, \pv_{t}})^{\color{black}\dagger}D^{\polv, \pv_{t'}}(J_{t'}^\polv\bm{1} - \rv_{t'}) - (\Bar{\Av}^{\polv, \pv_{t'}})^{\color{black}\dagger}D^{\polv, \pv_{t'}}(J_{t'}^\polv\bm{1} - \rv_{t'})\Vert_2 \nonumber \\
        & \overset{(d)}{\leq} \lambda^{-1} \sqrt{|\gS||\gA|} \left( \Vert\rv_t - \rv_{t'}\Vert_\infty + L_P\Vert\pv_t-\pv_{t'}\Vert_\infty + 2U_RM\Vert\pv_t - \pv_{t'}\Vert_\infty\right) \nonumber\\
        & \quad + \lambda^{-1}\Vert D^{\polv, \pv_t}\rv_t - D^{\polv, \pv_{t'}}\rv_{t'} \Vert_2 \\ & \quad + \Vert(\Bar{\Av}^{\polv, \pv_{t}})^{\color{black}\dagger}D^{\polv, \pv_{t'}}(J_{t'}^\polv\bm{1} - \rv_{t'}) - (\Bar{\Av}^{\polv, \pv_{t'}})^{\color{black}\dagger}D^{\polv, \pv_{t'}}(J_{t'}^\polv\bm{1} - \rv_{t'})\Vert_2 \nonumber \\
        & \overset{(e)}{\leq} \lambda^{-1} \sqrt{|\gS||\gA|} \left( 2\Vert\rv_t - \rv_{t'}\Vert_\infty + L_P\Vert\pv_t-\pv_{t'}\Vert_\infty + 4U_RM\Vert\pv_t - \pv_{t'}\Vert_\infty\right) \nonumber\\
        & \quad + \Vert(\Bar{\Av}^{\polv, \pv_{t}})^{\color{black}\dagger}D^{\polv, \pv_{t'}}(J_{t'}^\polv\bm{1} - \rv_{t'}) - (\Bar{\Av}^{\polv, \pv_{t'}})^{\color{black}\dagger}D^{\polv, \pv_{t'}}(J_{t'}^\polv\bm{1} - \rv_{t'})\Vert_2 \nonumber \\
        & \overset{}{\leq} \lambda^{-1} \sqrt{|\gS||\gA|} \left( 2\Vert\rv_t - \rv_{t'}\Vert_\infty + L_P\Vert\pv_t-\pv_{t'}\Vert_\infty + 4U_RM\Vert\pv_t - \pv_{t'}\Vert_\infty\right) \nonumber\\
        & \quad + \Vert(\Bar{\Av}^{\polv, \pv_{t}})^{\color{black}\dagger} - (\Bar{\Av}^{\polv, \pv_{t'}})^{\color{black}\dagger}\Vert_2 \cdot 2U_R \nonumber \\
        & \overset{(f)}{\leq} \lambda^{-1} \sqrt{|\gS||\gA|} \left( 2\Vert\rv_t - \rv_{t'}\Vert_\infty + L_P\Vert\pv_t-\pv_{t'}\Vert_\infty + 4U_RM\Vert\pv_t - \pv_{t'}\Vert_\infty\right) \nonumber\\
        & \quad + 2U_R \lambda^{-2}\Vert\Bar{\Av}^{\polv, \pv_{t}} - \Bar{\Av}^{\polv, \pv_{t'}}\Vert_2  \nonumber \\
        & \overset{(g)}{\leq} \lambda^{-1} \sqrt{|\gS||\gA|} \left( 2\Vert\rv_t - \rv_{t'}\Vert_\infty + L_P\Vert\pv_t-\pv_{t'}\Vert_\infty + 4U_RM\Vert\pv_t - \pv_{t'}\Vert_\infty\right) \nonumber\\
        & \quad + 2U_R \lambda^{-2} \cdot 2(M+1)\sqrt{|\gS||\gA|} \Vert\pv_t - \pv_{t'}\Vert_\infty \nonumber \\
        & \leq G_R \Vert\rv_t-\rv_{t'}\Vert_\infty + G_P\Vert\pv_t-\pv_{t'}\Vert_\infty \nonumber
    \end{align}
    where $(a)$ is because $\mathbb{E}\left[\rv(O) - \Jv(O) + \Av(O)\qv^{\polv}\right] = 0$ (see TD limiting point (\ref{eq:TD-limit-point}) in \Cref{sec:regretAnalysisAssumptions}) $(b)$ is from \Cref{assumption:maxEigValue}, $(c)$ is by \Cref{tlemma:avgRewardLipschitz}, $(d)$ is due to \cref{plemma:dTV}, $(e)$ is using the same process as the last step for the second term, $(f)$ is because $\Vert \X^{\color{black}\dagger} - \Y^{\color{black}\dagger}\Vert_2 \leq \Vert \X^{\color{black}\dagger}(\X-\Y)\Y^{\color{black}\dagger}\Vert_2 \leq \Vert \X^{\color{black}\dagger}\Vert_2 \Vert \X-\Y\Vert_2 \Vert \Y^{\color{black}\dagger}\Vert_2$ and $(g)$ is by \Cref{plemma:maxValues} and \Cref{plemma:dTV}. 
\end{proof}

%%%%%%%%%%%%%%%%%%%

\begin{lemma} \label{tlemma:consecutiveTimeQDiff}
    For any $t \geq 0$, we have 
    \begin{align}
        \Vert{\color{black}\Pi_E}\left[\qv_{t+1}^{\polv_{t+1}} - \qv_{t}^{\polv_t}\right]\Vert_2 \leq G_R\Vert\rv_{t+1}-\rv_t\Vert_\infty + G_P\Vert\pv_{t+1} - \pv_t\Vert_\infty + G_\polv B_2 \alr. \nonumber 
    \end{align}
    See \Cref{app:symbolReference} for constants.
\end{lemma}
\begin{proof}
    \begin{align}
        \Vert{\color{black}\Pi_E}\left[\qv_{t+1}^{\polv_{t+1}} - \qv_{t}^{\polv_t}\right]\Vert_2 & \leq \Vert{\color{black}\Pi_E}\left[\qv_{t+1}^{\polv_{t+1}} - \qv_{t}^{\polv_{t+1}}\right]\Vert_2 + \Vert{\color{black}\Pi_E}\left[\qv_{t}^{\polv_{t+1}} - \qv_{t}^{\polv_t}\right]\Vert_2 \nonumber\\
        & \overset{(a)}{\leq} G_R\Vert\rv_{t+1}-\rv_t\Vert_\infty + G_P\Vert\pv_{t+1} - \pv_t\Vert_\infty + G_\polv\Vert\polv_{t+1}-\polv_t\Vert_2 \nonumber\\
        & \overset{(b)}{\leq}  G_R\Vert\rv_{t+1}-\rv_t\Vert_\infty + G_P\Vert\pv_{t+1} - \pv_t\Vert_\infty + G_\polv B_2 \alr \nonumber
    \end{align}
    where $(a)$ is by \Cref{tlemma:QTimeLipschitz} and \Cref{tlemma:QPolicyLipschitz} and $(b)$ is from \Cref{tlemma:consecutiveTimePolicyDiff}.
\end{proof}

%%%%%%%%%%%%%%%%%%%

\begin{lemma} \label{tlemma:criticGamma}
    If \Cref{assumption:ergodic} holds, for any $t > \mt$, we have
    \begin{align*}
        \mathbb{E}\left[\Gamma(\polv_t, \pv_t, \rv_t, \qdv_t, O_t)\right] \leq B_3 \alr (\mt+1)^2 + B_4 \clr \mt + B_5 \Delta_{R, t-\mt+1, t} + B_6 \mt\Delta_{P, t-\mt+1, t}
    \end{align*}
    where $B_3 = (F_{1\polv} + F_2 G_\polv + F_3\sqrt{|\gS||\gA|} + F_4)B_2$, $B_4 = F_2(2U_R + 2U_Q)$, $B_5 = F_2 G_R$ and $B_6 = F_{1\pv} + F_2G_P + F_3$, $\Delta_{R, t-\mt+1, t} = \sum_{i=t-\mt+1}^t \Vert\rv_i-\rv_{i-1}\Vert_\infty$ and $\Delta_{P, t-\mt+1, t} = \sum_{i=t-\mt+1}^t \Vert\pv_i-\pv_{i-1}\Vert_\infty$.
\end{lemma}
\begin{proof}
    Recall from \Cref{app:notation}, the definition
    \begin{align*}
        \Gamma(\polv, \pv, \rv, \qdv, O) & = \qdv^\top \left(\rv(O) - \Jv^{\polv, \pv, \rv}(O)+\Av(O)\qv^{\polv, \pv, \rv}\right) + \qdv^\top \left(\Av(O) - \Bar{\Av}^{\polv, \pv}\right)\qdv.
    \end{align*}
    We first decompose $\Gamma(\cdot)$ into the following four terms
    \begin{align*}
        \mathbb{E}\left[\Gamma(\polv_t,\pv_t, \rv_t, \qdv_t, O_t)\right] & \leq \underbrace{\mathbb{E}\left[\Gamma(\polv_t, \pv_t, \rv_t, \qdv_t, O_t) - \Gamma(\polv_{t-\mt-1}, \pv_{t-\mt}, \rv_t,\qdv_t, O_t)\right]}_{I_1} \\ 
        & + \underbrace{\mathbb{E}\left[\Gamma(\polv_{t-\mt-1}, \pv_{t-\mt}, \rv_t, \qdv_t, O_t) - \Gamma(\polv_{t-\mt-1}, \pv_{t-\mt}, \rv_t, \qdv_{t-\mt}, O_t)\right]}_{I_2} \\ 
        & + \underbrace{\mathbb{E}\left[\Gamma(\polv_{t-\mt-1}, \pv_{t-\mt}, \rv_t,\qdv_{t-\mt}, O_t) - \Gamma(\polv_{t-\mt-1},\pv_{t-\mt}, \rv_t, \qdv_{t-\mt}, \Tilde{O}_t)\right]}_{I_3} \\ 
        &  + \underbrace{\mathbb{E}\left[\Gamma(\polv_{t-\mt-1}, \pv_{t-\mt}, \rv_t,\qdv_{t-\mt}, \Tilde{O}_t)\right]}_{I_4}.
    \end{align*}
    We now bound each term as follows.
    \begin{align*}
        I_1 & \overset{(a)}{\leq} F_{1\polv} \mathbb{E}\left[\Vert\polv_t-\polv_{t-\mt-1}\Vert_2\right]  + F_{1\pv}\Vert\pv_t - \pv_{t-\mt}\Vert_\infty  \\
        & \leq F_{1\polv} \mathbb{E}\left[\sum_{i=t-\mt}^t\Vert\polv_i-\polv_{i-1}\Vert_2\right] + F_{1\pv} \sum_{i=t-\mt+1}^t \Vert\pv_i - \pv_{i-1}\Vert_\infty \\
        & \overset{(b)}{\leq} F_{1\polv} B_2 \alr (\mt+1) + F_{1\pv}\Delta_{P, t-\mt+1, t}
    \end{align*}
    where $(a)$ is by \Cref{alemma:critic1} and $(b)$ is due to \Cref{alemma:policySum}. For the second term, we have
    \begin{align*}
        I_2 & \overset{(c)}{\leq} F_2 \mathbb{E}\left[\Vert\qdv_t - \qdv_{t-\mt}\right]\Vert_2 \leq F_2 \mathbb{E}\left[\sum_{i=t-\mt+1}^t\Vert\qdv_i - \qdv_{i-1}\Vert_2\right] \\
        & \overset{(d)}{\leq} F_2 \lbr \sum_{i=t-\mt+1}^{t} (2U_R+2U_Q)\clr + G_R \Vert\rv_{i}-\rv_{i-1}\Vert_\infty + G_P\Vert\pv_{i}-\pv_{i-1}\Vert_\infty + G_\polv B_2\alr \rbr \\
        & \leq F_2 (2U_R+2U_Q)\clr\mt + F_2 G_R\Delta_{R, t-\mt+1, t} + F_2 G_P \Delta_{P, t-\mt+1, t} + F_2 G_\polv B_2 \alr \mt
    \end{align*}
    where $(c)$ is by \Cref{alemma:critic2}, $(d)$ follows from \Cref{plemma:maxValues}, $\norm{\qv_{t+1}-\qv_t}_2 \leq \alr (U_R + U_Q)$ by the critic update equation (\ref{line:criticUpdate}), \cref{tlemma:consecutiveTimeQDiff} and \Cref{alemma:policySum}. We also define $\Delta_{R, t-\mt+1, t} = \sum_{i=t-\mt+1}^t \Vert\rv_i-\rv_{i-1}\Vert_\infty$ and $\Delta_{P, t-\mt+1, t} = \sum_{i=t-\mt+1}^t \Vert\pv_i-\pv_{i-1}\Vert_\infty$.

    For the third term, we have
    \begin{align*}
        I_3 & \overset{(e)}{\leq} F_3\sqrt{|\gS||\gA|} \mathbb{E}\left[\sum_{i=t-\mt}^t \Vert\polv_i-\polv_{t-\mt-1}\Vert_2 \Big\lvert\mathcal{F}_{t-\mt} \right] + F_3\sum_{i=t-\mt}^t\Vert\pv_{i}-\pv_{t-\mt}\Vert_\infty \\
        & \overset{(f)}{\leq} F_3\sqrt{|\gS||\gA|} B_2 \alr (\mt+1)^2 + F_3\mt\Delta_{P, t-\mt+1, t}.
    \end{align*}
    where $(e)$ is due to \Cref{alemma:critic3} and $(f)$ follows from \cref{alemma:policySum}. For the last term, by \Cref{alemma:critic4}, we have
    \begin{align*}
        I_4 \leq F_4 m \rho^\mt.
    \end{align*}

    We get the final result by putting all the four terms together.
\end{proof}

%%%%%%%%%%%%%%%%%%%

% \newpage
\subsubsection{\textbf{Average Reward Estimation}}
\begin{lemma} \label{tlemma:avgRewardLambda}
    If \Cref{assumption:ergodic} holds, for any $t > \mt$, we have
    \begin{align*}
        \mathbb{E}[\Lambda(\polv_t, \pv_t, \rv_t, \re_t, O_t)] \leq B_7 \alr (\mt+1)^2 + F_6 \lvert \re_t - \re_{t-\mt}\rvert + B_8 \mt \Delta_{P, t-\mt+1, t}
    \end{align*}
    where $B_7 = (F_5 L_\polv + F_7 \sqrt{|\gS||\gA|} + F_8)B_2$, $B_8 = F_7 + F_5L_P$ and $\Delta_{P, t-\mt+1, t} = \sum_{i=t-\mt+1}^t \Vert\pv_i - \pv_{i-1}\Vert_\infty$.
\end{lemma}
\begin{proof}
    Recall from \Cref{app:notation}, the definition
    \begin{align*}
        \Lambda(\polv, \pv, \rv, \re, O) & = (\re - J^{\polv, \pv, \rv}) (\r(\s, \a) - J^{\polv, \pv, \rv})
    \end{align*}
    We first decompose $\Lambda(\polv_t, \pv_t, \rv_t, \re_t, O_t)$ into the following four terms
    \begin{align*}
        \mathbb{E}[\Lambda(\polv_t, \pv_t, \rv_t, \re_t, O_t)] & = \underbrace{\mathbb{E}[\Lambda(\polv_t, \pv_t, \rv_t, \re_t, O_t) - \Lambda(\polv_{t-\mt-1}, \pv_{t-\mt}, \rv_t, \re_t, O_t)]}_{I_1} \\
        & \quad + \underbrace{\mathbb{E}[\Lambda(\polv_{t-\mt-1}, \pv_{t-\mt}, \rv_t, \re_t, O_t) - \Lambda(\polv_{t-\mt-1}, \pv_{t-\mt}, \rv_t, \re_{t-\mt}, O_t)]}_{I_2} \\
        & \quad + \underbrace{\mathbb{E}[\Lambda(\polv_{t-\mt-1}, \pv_{t-\mt}, \rv_t, \re_{t-\mt}, O_t) - \Lambda(\polv_{t-\mt-1}, \pv_{t-\mt}, \rv_t, \re_{t-\mt}, \Tilde{O}_t)]}_{I_3} \\
        & \quad + \underbrace{\mathbb{E}[\Lambda(\polv_{t-\mt-1}, \pv_{t-\mt}, \rv_t, \re_{t-\mt}, \Tilde{O}_t)]}_{I_4}.
    \end{align*}

    We now bound each term as follows.
    \begin{align*}
        I_1 & \overset{(a)}{\leq} F_5 L_\polv \mathbb{E}\left[\Vert\polv_t-\polv_{t-\mt-1}\Vert_2\right] + F_5 L_P \Vert\pv_{t} - \pv_{t-\mt}\Vert_\infty \\
        & \leq F_5 L_\polv \mathbb{E}\left[\sum_{i=t-\mt}^t\Vert\polv_i-\polv_{i-1}\Vert_2\right] + F_5L_P \sum_{i=t-\mt+1}^{t} \Vert\pv_i-\pv_{i-1}\Vert_\infty\\
        & \overset{(b)}{\leq} F_5 L_\polv B_2 \alr (\mt+1) + F_5 L_P \Delta_{P, t-\mt+1, t}
    \end{align*}
    where $(a)$ follows from \Cref{alemma:avgReward1}, and $(b)$ is due to \Cref{alemma:policySum}. For the second term $I_2$, we have
    \begin{align*}
        I_2 \overset{(c)}{\leq} F_6 \lvert \re_t - \re_{t-\mt}\rvert
    \end{align*}
    where $(c)$ is by \Cref{alemma:avgReward2}. For the third term $I_3$, we have 
    \begin{align*}
        I_3 & \overset{(d)}{\leq} F_7\sqrt{|\gS||\gA|} \mathbb{E}\left[\sum_{i=t-\mt}^t \Vert\polv_i-\polv_{t-\mt-1}\Vert_2 \Big\lvert\mathcal{F}_{t-\mt} \right] + F_7\sum_{i=t-\mt}^t\Vert\pv_{i}-\pv_{t-\mt}\Vert_\infty \\
        & \overset{(e)}{\leq} F_7\sqrt{|\gS||\gA|} B_2 \alr (\mt+1)^2 + F_7 \Delta_{P, t-\mt+1, t}.
    \end{align*}
    where $(d)$ is due to \Cref{alemma:avgReward3} and $(e)$ follows from \cref{alemma:policySum}.
    For the last term, by \Cref{alemma:avgReward4}, we have
    \begin{align*}
        I_4 \leq F_8 m \rho^\mt.
    \end{align*}

    We get the final result by putting all the four terms together.
\end{proof}

%%%%%%%%%%%%%%%%%%%
%%%%%%%%%%%%%%%%%%%

% \newpage
\subsection{\textbf{Auxiliary Lemmas}} \label{app:auxLemmas}

\subsubsection{\textbf{Actor}}

%%%%%%%%%%%%%%%%%%%
%%%%%%%%%%%%%%%%%%%

\begin{lemma} \label{alemma:policySum}
    For any timesteps $t > \mt > 0$, the policies generated by \Cref{algo} satisfy
    \begin{align}
        \sum_{i=t-\mt}^t \Vert\polv_i - \polv_{t-\mt-1}\Vert_2 \leq B_2 \alr (\mt+1)^2 \nonumber
    \end{align}
    % where $B_2 = U_Q$.
    and reward and transition probability matrices satisfy 
    \begin{align*}
        \sum_{i=t-\mt}^t \Vert \rv_i - \rv_{t-\mt}\Vert_\infty & \leq \mt \sum_{i=t-\mt+1}^t \Vert \rv_i - \rv_{i-1}\Vert_\infty \\
        \sum_{i=t-\mt}^t \Vert \pv_i - \pv_{t-\mt}\Vert_\infty & \leq \mt \sum_{i=t-\mt+1}^t \Vert \pv_i - \pv_{i-1}\Vert_\infty.
    \end{align*}
\end{lemma}
\begin{proof}
    By triangle inequality, we have
    \begin{align*}
        \sum_{i=t-\mt}^t \Vert\polv_i - \polv_{t-\mt-1}\Vert_2 & \leq \sum_{i=t-\mt}^t \Vert \sum_{j=t-\mt}^i \polv_j - \polv_{j-1} \Vert_2\\
        & \leq \sum_{i=t-\mt}^t \sum_{j=t-\mt}^i \Vert \polv_j - \polv_{j-1}\Vert_2 \\
        & \overset{(a)}{\leq} B_2 \alr (\mt+1)^2
    \end{align*}
    where $(a)$ is by \Cref{tlemma:consecutiveTimePolicyDiff}.
    The rest follow similarly using triangle inequality.
\end{proof}

%%%%%%%%%%%%%%%%%%%

% \newpage
\subsubsection{\textbf{Critic}}

\begin{lemma} \label{alemma:critic1}
    For any $\polv, \polv', \pv, \pv', \rv, \qdv$ and $O=(\s, \a, \s', \a')$, we have
    \begin{align*}
        \lvert\Gamma(\polv,\pv, \rv, \qdv, O) - \Gamma(\polv',\pv', \rv, \qdv, O)\rvert \leq F_{1\polv} \Vert\polv-\polv'\Vert_2 + F_{1\pv} \Vert\pv-\pv'\Vert_\infty
    \end{align*}
    where $F_{1\polv} = 2U_Q L_\polv + 4U_Q G_\polv + 8U_Q^2(M+2)\lvert\gS\rvert\lvert\mathcal{A}\rvert$, $F_{1\pv} = 2U_QL_P + 4U_Q G_P + 8U_Q^2(M+1)\sqrt{|\gS||\gA|}$. 
\end{lemma}
\begin{proof}
    \begin{align*}
        & \lvert\Gamma(\polv, \pv, \rv,\qdv, O) - \Gamma(\polv', \pv', \rv, \qdv, O)\rvert \\
        & = \lvert\qdv^\top (\Jv^{\polv', \pv', \rv}(O) - \Jv^{\polv, \pv, \rv}(O)) + \qdv^\top \Av(O)\left(\qv^{\polv, \pv, \rv} - \qv^{\polv', \pv', \rv}\right) + \qdv^\top \left(\Bar{\Av}^{\polv', \pv'} - \Bar{\Av}^{\polv, \pv}\right)\qdv\rvert \\
        & \overset{(a)}{\leq} \Vert\qdv\Vert_\infty\lvert J^{\polv', \pv', \rv} - J^{\polv, \pv, \rv} \rvert + \Vert\qdv\Vert_2\left\Vert \Av(O)\right\Vert_2\left\Vert\qv^{\polv, \pv, \rv} - \qv^{\polv', \pv', \rv}\right\Vert_2 \\ & \quad + \Vert\qdv\Vert_\infty\left\Vert \Bar{\Av}^{\polv', \pv'} - \Bar{\Av}^{\polv, \pv}\right\Vert_\infty\left\Vert\qdv\right\Vert_1 \\
        & \overset{(b)}{\leq} 2U_Q L_\polv \Vert\polv - \polv'\Vert_2 + 2U_QL_P\Vert\pv-\pv'\Vert_\infty + \Vert\qdv\Vert_2\left\Vert \Av(O)\right\Vert_2\left\Vert\qv^{\polv, \pv, \rv} - \qv^{\polv', \pv', \rv}\right\Vert_2 \\ 
        & \quad + \Vert\qdv\Vert_\infty\left\Vert \Bar{\Av}^{\polv', \pv'} - \Bar{\Av}^{\polv, \pv}\right\Vert_\infty\left\Vert\qdv\right\Vert_1 \\
        & \overset{(c)}{\leq} 2U_Q L_\polv \Vert\polv - \polv'\Vert_2 + 2U_QL_P\Vert\pv-\pv'\Vert_\infty + 4U_Q \cdot G_{\polv} \Vert \polv - \polv' \Vert_2 + 4U_Q G_P \Vert\pv - \pv'\Vert_\infty \\
        & \quad + \Vert\qdv\Vert_\infty\left\Vert \Bar{\Av}^{\polv', \pv'} - \Bar{\Av}^{\polv, \pv}\right\Vert_\infty\left\Vert\qdv\right\Vert_1 \\
        & \overset{(d)}{\leq} 2U_Q L_\polv \Vert\polv - \polv'\Vert_2 + 2U_QL_P\Vert\pv-\pv'\Vert_\infty + 4U_Q G_\polv\Vert\polv-\polv'\Vert_2 + 4U_Q G_P \Vert\pv - \pv'\Vert_\infty \\ 
        & \quad + 2U_Q \cdot 2\dtv\left(d^{\polv', \pv'}\otimes\polv'\otimes\pv'\otimes\polv', d^{\polv, \pv}\otimes\polv\otimes\pv\otimes\polv\right)\cdot 2U_Q\sqrt{|\gS||\gA|} \\
        & \overset{(e)}{\leq} 2U_Q L_\polv \Vert\polv - \polv'\Vert_2 + 2U_QL_P\Vert\pv-\pv'\Vert_\infty + 4 U_Q G_\polv\Vert\polv-\polv'\Vert_2 + 4U_Q G_P \Vert\pv - \pv'\Vert_\infty \\ 
        & \quad + 8U_Q^2(M+2)\lvert\gS\rvert\lvert\mathcal{A}\rvert\Vert\polv-\polv'\Vert_2 + 8U_Q^2(M+1)\sqrt{|\gS||\gA|}\Vert\polv-\polv'\Vert_\infty
    \end{align*}
    where $(a)$ follows from Holder's inequality; $(b)$ is due to \Cref{tlemma:avgRewardLipschitz}; $(c)$ is by \Cref{tlemma:consecutiveTimeQDiff} and \Cref{plemma:maxValues} ($\Vert \Av (O) \Vert_1 \leq 1$); $(d)$ is by \Cref{plemma:maxValues} and $(e)$ uses \Cref{plemma:dTV}.
\end{proof}

%%%%%%%%%%%%%%%%%%%

\begin{lemma} \label{alemma:critic2}
    For any $\polv, \pv, \rv, \qdv, \qdv'$ and $O=(\s, \a, \s', \a')$, we have
    \begin{align*}
        \lvert\Gamma(\polv, \pv, \rv, \qdv, O) - \Gamma(\polv, \pv, \rv, \qdv', O)\lvert \leq F_2 \Vert \qdv-\qdv'\Vert_2
    \end{align*}
    where $F_2 = 2U_R + 18U_Q$.
\end{lemma}
\begin{proof}
    \begin{align*}
        & \lvert\Gamma(\polv, \pv, \rv, \qdv, O) - \Gamma(\polv, \pv, \rv, \qdv', O)\lvert\\ 
        & \leq \left(\Vert \rv(O)\Vert_2 + \Vert\Jv^{\polv, \pv, \rv}(O)\Vert_2 + \Vert\Av(O)\Vert_2 \Vert\qv^{\polv, \pv, \rv}\Vert_2\right)\Vert \qdv - \qdv'\Vert_2 \\ & \quad +  \Vert \Av(O) - \Bar{\Av}^{\polv, \pv}\Vert_2 \Vert\qdv-\qdv'\Vert_2 \left(\Vert\qdv\Vert_2 + \Vert\qdv'\Vert_2\right)\\
        & \leq (2U_R + 18U_Q) \Vert \qdv-\qdv'\Vert_2.
    \end{align*}
\end{proof}

%%%%%%%%%%%%%%%%%%%

\begin{lemma} \label{alemma:critic3}
    Consider an observation from the original Markov chain by $O_t = (\s_t, \a_t, \s_{t+1}, \a_{t+1})$ and auxiliary Markov chain by $\Tilde{O}_t = (\Tilde{\s}_t, \Tilde{\a}_t, \Tilde{\s}_{t+1}, \Tilde{\a}_{t+1})$. Conditioned on $\mathcal{F}_{t-\mt} = \{\s_{t-\mt}, \polv_{t-\mt-1}, \pv_{t-\mt}\}$, we have
    \begin{align*}
        & \mathbb{E}\left[\Gamma(\polv_{t-\mt-1}, \pv_{t-\mt}, \rv_t, \qdv_{t-\mt}, O_t) - \Gamma(\polv_{t-\mt-1}, \pv_{t-\mt}, \rv_t, \qdv_{t-\mt}, \Tilde{O}_t) \big\lvert \mathcal{F}_{t-\mt}\right] \\ & \quad \leq F_3\sqrt{|\gS||\gA|} \mathbb{E}\left[\sum_{i=t-\mt}^t \Vert\polv_i-\polv_{t-\mt-1}\Vert_2 \Big\lvert\mathcal{F}_{t-\mt} \right] + F_3\sum_{i=t-\mt}^t\Vert\pv_{i}-\pv_{t-\mt}\Vert_\infty
    \end{align*}
    where $F_3 = 16U_R U_Q + 24U_Q^2 \sqrt{|\gS||\gA|}$. 
\end{lemma}
\begin{proof}
    Consider the original and auxiliary Markov chains whose construction is described in \Cref{app:notation}.
    \begin{align*}
        & \mathbb{E}\left[\Gamma(\polv_{t-\mt-1}, \pv_{t-\mt}, \rv_t, \qdv_{t-\mt}, O_t) - \Gamma(\polv_{t-\mt-1}, \pv_{t-\mt}, \rv_t, \qdv_{t-\mt}, \Tilde{O}_t) \big\lvert \mathcal{F}_{t-\mt}\right] \\
        & = \qdv_{t-\mt}^\top \mathbb{E}\left[\rv_t(O_t) - \rv_t(\Tilde{O}_t) + \Jv^{\polv_{t-\mt-1}, \pv_{t-\mt}, \rv_t}(\Tilde{O}_t) - \Jv^{\polv_{t-\mt-1}, \pv_{t-\mt}, \rv_t,}(O_t) \big\lvert \mathcal{F}_{t-\mt}\right] \\ & \quad + \qdv_{t-\mt}^\top \mathbb{E}\left[\left(\Av(O_t)-\Av(\Tilde{O}_t)\right)\qv^{\polv_{t-\mt-1}, \pv_{t-\mt}, \rv_t} \big\lvert \mathcal{F}_{t-\mt}\right] \\ 
        & \quad + \qdv_{t-\mt}^\top \mathbb{E}\left[\left(\Av(O_t)-\Av(\Tilde{O}_{t})\right) \big\lvert\mathcal{F}_{t-\mt}\right]\qdv_{t-\mt} \\
        & \leq \Vert\qdv_{t-\mt}\Vert_\infty\left\Vert\mathbb{E}\left[\rv_t(O_t) - \rv_t(\Tilde{O}_t) + \Jv_t^{\polv_{t-\mt-1}}(\Tilde{O}_t) - \Jv_t^{\polv_{t-\mt-1}}(O_t) \big\lvert \mathcal{F}_{t-\mt}\right]\right\Vert_1  \\ 
        & \quad  + \Vert\qdv_{t-\mt}\Vert_\infty\left\Vert\mathbb{E}\left[\Av(O_t)-\Av(\Tilde{O}_t) \big\lvert \mathcal{F}_{t-\mt}\right]\right\Vert_1 \Vert\qv^{\polv_{t-\mt-1}, \pv_{t-\mt}, \rv_t}\Vert_1\\ 
        & \quad + \Vert\qdv_{t-\mt}\Vert_\infty\left\Vert\mathbb{E}\left[\Av(O_t)-\Av(\Tilde{O}_{t}) \big\lvert\mathcal{F}_{t-\mt}\right]\right\Vert_1 \Vert\qdv_{t-\mt}\Vert_1 \\ 
        & \overset{}{\leq} 2U_Q \cdot 4U_R \cdot 2\dtv\left(\p(O_t \in \cdot |\mathcal{F}_{t-\mt}), \p(\Tilde{O}_t \in \cdot |\mathcal{F}_{t-\mt})\right) \\ 
        & \quad + 2U_Q \cdot 4\dtv\left(\p(O_t \in \cdot |\mathcal{F}_{t-\mt}), \p(\Tilde{O}_t \in \cdot |\mathcal{F}_{t-\mt})\right)\cdot U_Q\sqrt{|\gS||\gA|} \\ 
        & \quad + 2U_Q \cdot 4\dtv\left(\p(O_t \in \cdot |\mathcal{F}_{t-\mt}), \p(\Tilde{O}_t \in \cdot |\mathcal{F}_{t-\mt})\right)\cdot 2U_Q\sqrt{|\gS||\gA|} \\
        & \overset{}{\leq} (16U_R U_Q + 24U_Q^2 \sqrt{|\gS||\gA|})\left(\sqrt{|\gS||\gA|} \mathbb{E}\left[\sum_{i=t-\mt}^t \Vert\polv_i-\polv_{t-\mt-1}\Vert_2 \Big\lvert\mathcal{F}_{t-\mt} \right] + \sum_{i=t-\mt}^t\Vert\pv_{i}-\pv_{t-\mt}\Vert_\infty \right)
    \end{align*}
    where the last inequality is from \Cref{plemma:dTVauxO}.
\end{proof}

%%%%%%%%%%%%%%%%%%%

% \subsection{Average Reward}
\begin{lemma} \label{alemma:critic4}
    Consider an observation from the original Markov chain by $O_t = (\s_t, \a_t, \s_{t+1}, \a_{t+1})$ and auxiliary Markov chain by $\Tilde{O}_t = (\Tilde{\s}_t, \Tilde{\a}_t, \Tilde{\s}_{t+1}, \Tilde{\a}_{t+1})$. Conditioned on $\mathcal{F}_{t-\mt} = \{\s_{t-\mt}, \polv_{t-\mt-1}, \pv_{t-\mt}\}$, we have
    \begin{align*}
        \mathbb{E}\left[\Gamma(\polv_{t-\mt-1}, \pv_{t-\mt}, \rv_t, \qdv_{t-\mt}, \Tilde{O}_t) \big\lvert \mathcal{F}_{t-\mt} \right] \leq F_4 m \rho^\mt
    \end{align*}
    where $F_4 = 8U_R U_Q + 24U_Q^2\sqrt{|\gS||\gA|}$. 
\end{lemma}
\begin{proof}
    Consider the original and auxiliary Markov chains whose construction is described in \Cref{app:notation}. Also, consider the observation tuple $O'_t = (\s'_t, \a'_t, \s'_{t+1}, \a'_{t+1})$ where $\s'_t \sim d^{\polv_{t-\mt-1}, \pv_{t-\mt}}(\cdot)$, $\a'_t \sim \polv_{t-\mt-1}(\cdot|\s'_t)$, $\s'_{t+1}\sim\pv_{t-\mt}(\cdot|\s'_t, \a'_t)$ and $\a'_{t+1}\sim\polv_{t-\mt-1}(\cdot|\s'_{t+1})$. From the definition of $\Gamma(\cdot)$ and the TD limit point equation (\ref{eq:TD-limit-point}), it follows that
    \begin{align*}
        \mathbb{E}\left[\Gamma(\polv_{t-\mt-1}, \pv_{t-\mt}, \rv_t, \qdv_{t-\mt}, O'_t) \big\lvert \mathcal{F}_{t-\mt} \right] = 0
    \end{align*}

    Hence, we have
    \begin{align*}
        & \mathbb{E}\left[\Gamma(\polv_{t-\mt-1}, \pv_{t-\mt}, \rv_t, \qdv_{t-\mt}, \Tilde{O}_t) \big\lvert \mathcal{F}_{t-\mt} \right] \\ & \quad \leq \mathbb{E}\left[\Gamma(\polv_{t-\mt-1}, \pv_{t-\mt}, \rv_t, \qdv_{t-\mt}, \Tilde{O}_t) - \Gamma(\polv_{t-\mt-1}, \pv_{t-\mt}, \rv_t, \qdv_{t-\mt}, O'_t)  \big\lvert \mathcal{F}_{t-\mt} \right]\\ 
        & \quad \leq \Vert\qdv_{t-\mt}\Vert_\infty \left\Vert\mathbb{E}\left[\rv_t(\Tilde{O}_t) - \Jv^{\polv_{t-\mt-1}, \pv_{t-\mt}, \rv_t}(\Tilde{O}_t) - \rv_t(O'_t) + \Jv^{\polv_{t-\mt-1}, \pv_{t-\mt}, \rv_t}(O'_t) \big\lvert\mathcal{F}_{t-\mt} \right]\right\Vert_1 \\ & \quad + \Vert\qdv_{t-\mt}\Vert_\infty\left\Vert\mathbb{E}\left[\left(\Av(\Tilde{O}_t)-\Av(O'_t)\right)\qv^{\polv_{t-\mt-1}, \pv_{t-\mt}, \rv_t} \big\lvert \mathcal{F}_{t-\mt}\right]\right\Vert_1 \\& \quad + \Vert\qdv_{t-\mt}\Vert_\infty\left\Vert\mathbb{E}\left[\left(\Av(\Tilde{O}_t)-\Av(O'_t)\right)\qdv_{t-\mt} \big\lvert \mathcal{F}_{t-\mt}\right]\right\Vert_1 \\
        & \quad \leq 2U_Q \cdot 4U_R \cdot 2\dtv\left(\p(\Tilde{O}_t \in \cdot |\mathcal{F}_{t-\mt}), \p(O'_t \in \cdot |\mathcal{F}_{t-\mt})\right) \\ & \quad + 2U_Q \cdot 4\dtv\left(\p(\Tilde{O}_t \in \cdot |\mathcal{F}_{t-\mt}), \p(O'_t \in \cdot |\mathcal{F}_{t-\mt})\right)\cdot U_Q\sqrt{|\gS||\gA|} \\ 
        & \quad + 2U_Q \cdot 4\dtv\left(\p(\Tilde{O}_t \in \cdot |\mathcal{F}_{t-\mt}), \p(O'_t \in \cdot |\mathcal{F}_{t-\mt})\right)\cdot 2U_Q\sqrt{|\gS||\gA|} \\
        & \quad = F_4 \sum_{\s,\a,\s',\a'} \lvert\p(\Tilde{\s}_t=s|\mathcal{F}_{t-\mt})\pol_{t-\mt-1}(\a|\s)\p_{t-\mt}(\s'|\s, \a)\pol_{t-\mt-1}(\a'|\s') \\ & \quad- \p(\s'_t=\s|\mathcal{F}_{t-\mt})\pol_{t-\mt-1}(\a|\s)\p_{t-\mt}(\s'|\s,\a)\pol_{t-\mt-1}(\a'|\s') \rvert \\
        & \quad = F_4 \sum_{\s,\a,\s',\a'} \pol_{t-\mt-1}(\a|\s)\p(\s'|\s,\a)\pol_{t-\mt-1}(\a'|\s')\lvert\p(\Tilde{\s}_t=\s|\mathcal{F}_{t-\mt})-\p(\s'_t=\s|\mathcal{F}_{t-\mt})\rvert \\
        & \quad = F_4 \sum_{\s} \lvert \p(\Tilde{\s}_t=\s|\mathcal{F}_{t-\mt})-\p(\s'_t=\s|\mathcal{F}_{t-\mt})\rvert \\
        & \quad \leq F_4 m \rho^\mt
    \end{align*}
    where the last inequality follows from \Cref{assumption:ergodic}.
\end{proof}

%%%%%%%%%%%%%%%%%%%
\subsubsection{\textbf{Average Reward Estimation}}

\begin{lemma} \label{alemma:sampleAvgReward4}
    Consider an observation from the original Markov chain by $O_t = (\s_t, \a_t, \s'_t, \a'_t)$ and auxiliary Markov chain by $\Tilde{O}_t = (\Tilde{\s}_t, \Tilde{\a}_t, \Tilde{\s}_{t+1}, \Tilde{\a}_{t+1})$. Conditioned on $\mathcal{F}_{t-\mt} = \{\s_{t-\mt}, \polv_{t-\mt-1}, \pv_{t-\mt}\}$, we have
    \begin{align*}
        \mathbb{E}\left[J^{\polv_{t-\mt-1}, \pv_{t-\mt}, \rv_{t}} - \r_t(\Tilde{\s}_t, \Tilde{\a}_t) | \mathcal{F}_{t-\mt}\right] \leq 4U_R m \rho^{\mt}
    \end{align*}
    where $J^{\polv_{t-\mt-1}, \pv_{t-\mt}, \rv_t}= \sum_{\s, \a} d^{\polv_{t-\mt-1}, \pv_{t-\mt}}(\s)\polv_{t-\mt-1}(\a|\s)\r_t(\s, \a)$.
\end{lemma}
\begin{proof}
    Consider the observation tuple $O'_t = (\s'_t, \a'_t, \s'_{t+1}, \a'_{t+1})$ where $\s'_t \sim d^{\polv_{t-\mt-1}, \pv_{t-\mt}}(\cdot)$, $\a'_t \sim \polv_{t-\mt-1}(\cdot|\s'_t)$, $\s'_{t+1}\sim\pv_{t-\mt}(\cdot|\s'_t, \a'_t)$ and $\a'_{t+1}\sim\polv_{t-\mt-1}(\cdot|\s'_{t+1})$. Then, by definition of $J^{\polv_{t-\mt-1}, \pv_{t-\mt}, \rv_t}$, we have
    \begin{align*}
        \mathbb{E}\left[J^{\polv_{t-\mt-1}, \pv_{t-\mt}, \rv_t} - \r_t(\s'_t, \a'_t) | \mathcal{F}_{t-\mt} \right] = 0.
    \end{align*}
    Hence, we have
    \begin{align*}
        & \mathbb{E}\left[J^{\polv_{t-\mt-1}, \pv_{t-\mt}, \rv_t} - \r_t(\Tilde{\s}_t, \Tilde{\a}_t) | \mathcal{F}_{t-\mt}\right] \\ & \quad = \mathbb{E}\left[J^{\polv_{t-\mt-1}, \pv_{t-\mt}, \rv_t} - \r_{t}(\s'_t, \a'_t) - \r_t(\Tilde{\s}_t, \Tilde{\a}_t) + \r_{t}(\s'_t, \a'_t)|\mathcal{F}_{t-\mt}\right] \\
        & \quad = \mathbb{E}\left[\r_t(\s'_t, \a'_t) - \r_t(\Tilde{\s}_t, \Tilde{\a}_t)|\mathcal{F}_{t-\mt}\right] \\
        & \quad \leq 2U_R \cdot 2 \dtv\left(d^{\polv_{t-\mt-1}, \pv_{t-\mt}} \otimes \polv_{t-\mt-1}, \p((\Tilde{\s}_t, \Tilde{\a}_t) \in \cdot |\mathcal{F}_{t-\mt})\right) \\
        & \quad \overset{(a)}{\leq} 4U_R \dtv\left(d^{\polv_{t-\mt-1}, \pv_{t-\mt}}, \p(\Tilde{\s}_t \in \cdot | \mathcal{F}_{t-\mt})\right) \\
        & \quad \overset{(b)}{\leq} 4U_R m \rho^\mt
    \end{align*}
    where $(a)$ follows from Lemma B.1 in \citep{Wu2020A2CPG} and $(b)$ is by \Cref{assumption:ergodic}.
\end{proof}

%%%%%%%%%%%%%%%%%%%

\begin{lemma} \label{alemma:avgReward1}
    For any $ \polv, \polv', \pv, \pv', \rv, \re$, and $O=(\s, \a, \s', \a')$, we have
    \begin{align*}
        \lvert \Lambda(\polv, \pv, \rv, \re, O) - \Lambda(\polv', \pv', \rv, \re, O) \rvert \leq F_5L_\polv \Vert \polv - \polv' \Vert_2 + F_5 L_P \Vert\pv-\pv'\Vert_\infty,
    \end{align*}
    where $F_5 = 4U_R$.
\end{lemma}
\begin{proof}
    \begin{align*}
        & \lvert \Lambda(\polv, \pv, \rv, \re, O) - \Lambda(\polv', \pv', \rv, \re, O) \rvert \\ \quad & \leq \lvert (\re-J^{\polv, \pv, \rv})(\r(\s, \a) - J^{\polv, \pv, \rv}) - (\re - J^{\polv', \pv', \rv})(\r(\s, \a) - J^{\polv', \pv', \rv}) \rvert \\
        & \leq \lvert(\re-J^{\polv, \pv, \rv})(\r(\s, \a)-J^{\polv, \pv, \rv}) - (\re-J^{\polv, \pv, \rv})(\r(\s, \a) - J^{\polv', \pv', \rv}) \rvert \\ & \quad + \lvert(\re-J^{\polv, \pv, \rv})(\r(\s, \a)-J^{\polv', \pv', \rv}) - (\re-J^{\polv', \pv', \rv})(\r(\s, \a) - J^{\polv', \pv', \rv}) \rvert \\
        & \quad \leq 4U_R\lvert J^{\polv, \pv, \rv} - J^{\polv', \pv', \rv}\rvert \overset{(a)}{\leq} 4U_R L_\polv \Vert \polv - \polv' \Vert_2 + 4U_RL_P\Vert \pv - \pv'\Vert_\infty
    \end{align*}
    where $(a)$ follows from \Cref{tlemma:avgRewardLipschitz}.
\end{proof}

%%%%%%%%%%%%%%%%%%%

\begin{lemma} \label{alemma:avgReward2}
    For any $\polv, \pv, \rv, \re, \re'$ and $O=(\s, \a, \s', \a')$, we have
    \begin{align*}
        \lvert \Lambda(\polv, \pv, \rv, \re, O) - \Lambda(\polv, \pv, \rv, \re', O) \rvert \leq F_6 \lvert \re-\re'\rvert
    \end{align*}
    where $F_6 = 2U_R$.
\end{lemma}
\begin{proof}
Recall the definition of $\Lambda(\cdot)$ in \Cref{app:notation}. It is straightforward to see that
    \begin{align*}
        \lvert \Lambda(\polv, \pv, \rv, \re, O) - \Lambda(\polv, \pv, \rv, \re', O) \rvert  \leq 2U_R \lvert \re-\re'\rvert
    \end{align*}
\end{proof}

%%%%%%%%%%%%%%%%%%%

\begin{lemma} \label{alemma:avgReward3}
    Consider an observation from the original Markov chain by $O_t = (\s_t, \a_t, \s_{t+1}, \a_{t+1})$ and auxiliary Markov chain by $\Tilde{O}_t = (\Tilde{\s}_t, \Tilde{\a}_t, \Tilde{\s}_{t+1}, \Tilde{\a}_{t+1})$. Conditioned on $\mathcal{F}_{t-\mt} = \{\s_{t-\mt}, \polv_{t-\mt-1}, \pv_{t-\mt} \}$, we have
    \begin{align*}
        & \mathbb{E}\left[\Lambda(\polv_{t-\mt-1}, \pv_{t-\mt}, \rv_t, \re_{t-\mt}, O_t) - \Lambda(\polv_{t-\mt-1}, \pv_{t-\mt}, \rv_t, \re_{t-\mt}, \Tilde{O}_t) \big\lvert \mathcal{F}_{t-\mt} \right] \\ & \quad \leq F_7\sqrt{|\gS||\gA|} \mathbb{E}\left[\sum_{i=t-\mt}^t \Vert\polv_i-\polv_{t-\mt-1}\Vert_2 \Big\lvert\mathcal{F}_{t-\mt} \right] + F_7\sum_{i=t-\mt}^t\Vert\pv_{i}-\pv_{t-\mt}\Vert_\infty 
    \end{align*}
    where $F_7 = 8U_R^2$.
\end{lemma}
\begin{proof}
    \begin{align*}
        & \mathbb{E}\left[\Lambda(\polv_{t-\mt-1}, \pv_{t-\mt}, \rv_t, \re_{t-\mt}, O_t) - \Lambda(\polv_{t-\mt-1},\pv_{t-\mt}, \rv_t, \re_{t-\mt}, \Tilde{O}_t) \big\lvert \mathcal{F}_{t-\mt} \right] \\ & \quad = (\re_{t-\mt} - J^{\polv_{t-\mt-1}, \pv_{t-\mt}, \rv_t}) \mathbb{E}\left[\r_t(\s_t, \a_t) - \r_t(\Tilde{\s}_t, \Tilde{\a}_{t}) \big\lvert \mathcal{F}_{t-\mt} \right] \\
        & \quad \overset{}{\leq} 2U_R \cdot 4U_R \dtv\left(\p(O_t \in \cdot |\mathcal{F}_{t-\mt}), \p(\Tilde{O}_t \in \cdot |\mathcal{F}_{t-\mt})\right) \\
        & \quad \overset{(a)}{\leq} F_7\sqrt{|\gS||\gA|} \mathbb{E}\left[\sum_{i=t-\mt}^t \Vert\polv_i-\polv_{t-\mt-1}\Vert_2 \Big\lvert\mathcal{F}_{t-\mt} \right] + F_7\sum_{i=t-\mt}^t\Vert\pv_{i}-\pv_{t-\mt}\Vert_\infty
    \end{align*}
    where $(a)$ follows from \Cref{plemma:dTVauxO}.
\end{proof}

%%%%%%%%%%%%%%%%%%%

\begin{lemma} \label{alemma:avgReward4}
    Consider an observation from the original Markov chain by $O_t = (\s_t, \a_t, \s_{t+1}, \a_{t+1})$ and auxiliary Markov chain by $\Tilde{O}_t = (\Tilde{\s}_t, \Tilde{\a}_t, \Tilde{\s}_{t+1}, \Tilde{\a}_{t+1})$. Conditioned on $\mathcal{F}_{t-\mt} = \{\s_{t-\mt}, \polv_{t-\mt-1}, \pv_{t-\mt} \}$, we have
    \begin{align*}
        \mathbb{E}\left[\Lambda(\polv_{t-\mt-1}, \pv_{t-\mt}, \rv_t, \re_{t-\mt}, \Tilde{O}_t) \big\lvert \mathcal{F}_{t-\mt}\right] \leq F_8 m \rho^\mt
    \end{align*}
    where $F_8 = 8U_R^2$.
\end{lemma}
\begin{proof}
    Consider the observation tuple $O'_t = (\s'_t, \a'_t, \s'_{t+1}, \a'_{t+1})$ where $\s'_t \sim d^{\polv_{t-\mt-1}, \pv_{t-\mt}}(\cdot)$, $\a'_t \sim \polv_{t-\mt-1}(\cdot|\s'_t)$, $\s'_{t+1}\sim\pv_{t-\mt}(\cdot|\s'_t, \a'_t)$ and $\a'_{t+1}\sim\polv_{t-\mt-1}(\cdot|\s'_{t+1})$.

    We know 
    \begin{align*}
        \mathbb{E}\left[\Lambda(\polv_{t-\mt-1}, \pv_{t-\mt}, \rv_t, \re_{t-\mt}, O'_t) \big\lvert \mathcal{F}_{t-\mt} \right] = 0.
    \end{align*}

    Hence, we have
    \begin{align*}
        & \mathbb{E}\left[\Lambda(\polv_{t-\mt-1},\pv_{t-\mt}, \rv_t, \re_{t-\mt}, \Tilde{O}_t) \big\lvert \mathcal{F}_{t-\mt} \right] \\ & \quad = \mathbb{E}\left[\Lambda(\polv_{t-\mt-1}, \pv_{t-\mt}, \rv_t, \re_{t-\mt}, \Tilde{O}_t) \big\lvert - \Lambda(\polv_{t-\mt-1}, \pv_{t-\mt}, \rv_t, \re_{t-\mt}, O'_t) \big\lvert \mathcal{F}_{t-\mt} \right] \\
        & \quad = \mathbb{E}\left[(\re_{t-\mt}-J^{\polv_{t-\mt-1}, \pv_{t-\mt}, \rv_t})(\r_t(\Tilde{\s}_t, \Tilde{\a}_t) - \r_t(\s'_t, \a'_t)) \big\vert \mathcal{F}_{t-\mt} \right] \\
        & \quad \leq 2U_R \cdot 4U_R \dtv\left(d^{\polv_{t-\mt-1}, \pv_{t-\mt}} \otimes \polv_{t-\mt-1}, \p((\Tilde{\s}_t, \Tilde{\a}_t) \in \cdot |\mathcal{F}_{t-\mt})\right) \\
        & \quad \overset{(a)}{\leq} 2U_R \cdot 4U_R \dtv\left(d^{\polv_{t-\mt-1}, \pv_{t-\mt}}, \p(\Tilde{\s}_t \in \cdot | \mathcal{F}_{t-\mt})\right) \\
        & \quad \overset{(b)}{\leq} 8U_R^2 m \rho^\mt
    \end{align*}
    where $(a)$ follows from Lemma B.1 in \citep{Wu2020A2CPG} and $(b)$ is by \Cref{assumption:ergodic}.
\end{proof}

%%%%%%%%%%%%%%%%%%%
%%%%%%%%%%%%%%%%%%%

\subsection{\textbf{Preliminary Lemmas}} \label{app:plemmas}

\begin{lemma} \label{plemma:dTV}
    For any policies $\polv, \polv'$ and transition probabilities matrices $\pv, \pv'$, it holds that 
    % \ps{Check the $|\gA|$ term in first term}
    \begin{align*}
        \dtv\left(d^{\polv, \pv}, d^{\polv', \pv'}\right) & \leq  M\sqrt{|\gS||\gA|} \lvert\lvert\polv - \polv'\rvert\rvert_2 + M ||\pv - \pv'||_\infty , \\
        \dtv\left(d^{\polv, \pv} \otimes \polv, d^{\polv', \pv'} \otimes \polv'\right) & \leq (M+1)\sqrt{|\gS||\gA|} \Vert\polv - \polv'\Vert_2 + M\Vert\pv - \pv'\Vert_\infty, \\
        \dtv\left(d^{\polv, \pv} \otimes \polv \otimes \pv, d^{\polv', \pv'} \otimes \polv' \otimes \pv'\right) & \leq (M+1)\sqrt{|\gS||\gA|}\Vert\polv -\polv'\Vert_2 + (M+1)\Vert\pv-\pv'\Vert_\infty, \\
        \dtv\left(d^{\polv, \pv} \otimes \polv \otimes \pv \otimes \polv, d^{\polv', \pv'}\otimes\polv'\otimes\pv'\otimes\polv'\right) & \leq (M+2)\sqrt{|\gS||\gA|}\Vert\polv-\polv'\Vert_2 + (M+1)\Vert\pv - \pv'\Vert_\infty
    \end{align*}
    where $\otimes$ denotes the Kronecker product, and $M := \left(\ceil{\log_\rho m^{-1}} + \frac{1}{1-\rho}\right)$. 
\end{lemma}
\begin{proof}
    Recall that $d^{\polv, \pv}(\cdot)$ is the stationary distribution induced over the states by a Markov chain with transition probabilities $\pv$ following policy $\polv$. Define the matrices $\mathbf{K}, \mathbf{K'} \in \mathbb{R}^{\lvert\mathcal{S}\rvert \times \lvert\mathcal{S}\rvert}$ such that $\mathbf{K}(\s, \s') = \sum_{\a \in \mathcal{A}} \p(\s'|\s, \a)\pol(\a|\s)$ and $\mathbf{K'}(\s, \s') = \sum_{\a \in \mathcal{A}} \p'(\s'|\s, \a)\pol'(\a|\s)$. Further denote the total variation norm as $||\cdot||_{TV}$. Note that $\Vert \pv - \pv'\Vert_\infty = \max\limits_{\s, \a} \sum_{\s'} \lvert \p(\s'|\s, \a) - \p'(\s'|\s, \a) \rvert$.

    From Theorem 3.1 in \citep{mitrophanov2005sensitivity}, we have, 
    \begin{align*}
        \dtv\left(d^{\polv, \pv}, d^{\polv', \pv'}\right) & \leq M \sup_{\Vert q \Vert_{TV} = 1} \left\Vert \int_{\mathcal{S}} q(ds) (\mathbf{K} - \mathbf{K'})(\s, \cdot) \right\Vert_{TV} \leq M \sup_{\Vert q \Vert_{TV} = 1} \int_{\mathcal{S}} \left\lvert \int_{\mathcal{S}} q(ds)(\mathbf{K} - \mathbf{K'})(s, ds')\right\rvert \\
        & \leq M \sup_{\Vert q \Vert_{TV} = 1} \int_{\mathcal{S}} \int_{\mathcal{S}} \lvert q(ds) \rvert \left\lvert \sum_{\a \in \mathcal{A}} \p(ds'|\s, \a) \pol(\a|\s) - \p'(ds'|\s, \a)\pol'(\a|\s) \right\rvert \\
        & \leq M \sup_{\Vert q \Vert_{TV} = 1} \int_{\mathcal{S}} \int_{\mathcal{S}} \sum_{\a} \lvert q(ds) \rvert \left\lvert \p(ds'|\s, \a)\pol(\a|\s) - \p(ds'|\s, \a)\pol'(\a|\s) \right\rvert \\ 
        & \quad + M \sup_{\Vert q \Vert_{TV} = 1} \int_{\mathcal{S}}\int_{\mathcal{S}}\sum_{\a} \lvert q(ds) \rvert \left\lvert\p(ds'|\s, \a)\pol'(\a|\s) - \p'(ds'|\s, \a)\pol'(\a|\s)\right\rvert \\
        & \leq  M\sqrt{|\gS||\gA|} \lvert\lvert\polv - \polv'\rvert\rvert_2 + M ||\pv - \pv'||_\infty.
    \end{align*}
    
    For the second inequality, we have, 
    \begin{align*}
         \dtv\left(d^{\polv, \pv} \otimes \polv, d^{\polv', \pv'} \otimes \polv'\right) & \leq \frac{1}{2} \int_{\mathcal{S}} \sum_{\a} \left\lvert d^{\polv, \pv}(ds)\pol(\a|\s) - d^{\polv', \pv'}(ds)\pol'(\a|\s) \right\rvert \\
         & \leq \frac{1}{2} \int_{\mathcal{S}} \sum_{\a} \left\lvert d^{\polv, \pv}(ds)\pol(\a|\s) - d^{\polv, \pv}(ds)\pol'(\a|\s) \right\rvert \\ & \quad + \frac{1}{2} \int_{\mathcal{S}} \sum_{\a} \left\lvert d^{\polv, \pv}(ds)\pol'(\a|\s) - d^{\polv', \pv'}(ds)\pol'(\a|\s) \right\rvert \\
         & \leq \sqrt{|\gS||\gA|} \Vert \polv - \polv' \Vert_2 + \dtv\left(d^{\polv, \pv}, d^{\polv', \pv'} \right) \\
         & \leq (M+1)\sqrt{|\gS||\gA|} \Vert\polv - \polv'\Vert_2 + M\Vert\pv - \pv'\Vert_\infty.
    \end{align*}
    The rest follow in a similar manner. 
\end{proof}

%%%%%%%%%%%%%%%%%%%

\begin{lemma} \label{plemma:dTVauxO}
    Consider observations $O_t = (\s_t, \a_t, \s_{t+1}, \a_{t+1})$ and $\Tilde{O}_t = (\Tilde{\s}_t, \Tilde{a}_t, \Tilde{\s}_{t+1}, \Tilde{a}_{t+1})$ and define $\mathcal{F}_{t-\mt}:= \{\s_{t-\tau}, \polv_{t-\tau-1}, \pv_{t-\mt}\}$. We have 
    % \ps{Perhaps the first term shouldn't have the $|\gS|$ factor.}
    \begin{align}
        \dtv\left(\p(O_t \in \cdot |\mathcal{F}_{t-\mt}), \p(\Tilde{O}_t \in \cdot | \mathcal{F}_{t-\mt}) \right) & \leq \sqrt{|\gS||\gA|} \sum_{i=t-\mt}^t \mathbb{E} \left[\Vert\pol_i - \pol_{t-\mt-1}\Vert_2 \Big\lvert \mathcal{F}_{t-\mt} \right] + \Vert\pv_{i}-\pv_{t-\mt}\Vert_\infty. \nonumber
    \end{align}
\end{lemma}
\begin{proof}
    \begin{align}
        \dtv & \left(\p(O_t \in \cdot |\mathcal{F}_{t-\mt}), \p(\Tilde{O}_t \in \cdot | \mathcal{F}_{t-\mt}) \right) \nonumber\\
        & = \frac{1}{2} \sum_{\s, \a, \s', \a'} \lvert \p(\overbrace{\s_t=s, \a_t=a}^{\mathcal{H}_t}, \s_{t+1}=\s', \a_{t+1}=\a' | \mathcal{F}_{t-\mt}) - \p(\Tilde{\s}_t=\s, \Tilde{\a}_t=\a, \Tilde{\s}_{t+1}=\s', \Tilde{\a}_{t+1}=\a' | \mathcal{F}_{t-\mt}) \rvert \nonumber\\
        & = \frac{1}{2} \sum_{\s, \a, \s', \a'} \lvert \p(\s_t = \s, \a_t =\a | \mathcal{F}_{t-\mt})\p_t(\s'|\s, \a)\mathbb{E}\left[\pol_t(\a'|\s')|\mathcal{F}_{t-\mt}, \mathcal{H}_t \right] \nonumber\\&\quad - \p(\Tilde{\s}_t=\s, \Tilde{\a}_t=\a|\mathcal{F}_{t-\mt})\p_{t-\mt}(\s'|\s, \a)\pol_{t-\mt-1}(\a'|\s') \rvert \nonumber\\
        & \leq \frac{1}{2} \sum_{\s, \a, \s', \a'} \lvert \p(\s_t=\s, \a_t=\a|\mathcal{F}_{t-\mt})\p_{t}(\s'|\s, \a)\mathbb{E}\left[\pol_t(\a'|\s')|\mathcal{F}_{t-\mt}, \mathcal{H}_t\right] \nonumber\\ & \quad - \p(\Tilde{\s}_t=\s, \Tilde{\a}_t=a|\mathcal{F}_{t-\mt})\p_t(\s'|\s, \a)\pol_{t-\mt-1}(\a'|\s')\rvert \nonumber\\ 
        & \quad + \frac{1}{2} \sum_{\s, \a, \s', \a'} \lvert \p(\Tilde{\s}_t = \s, \Tilde{\a}=\a |\mathcal{F}_{t-\mt})\p_t(\s'|\s, \a)\pol_{t-\mt-1}(\a'|\s') \nonumber\\ 
        & \quad - \p(\Tilde{\s}_t=\s, \Tilde{\a}_t=\a|\mathcal{F}_{t-\mt})\p_{t-\mt}(\s'|\s, \a)\pol_{t-\mt-1}(\a'|\s')\rvert \nonumber\\
        & = \frac{1}{2} \sum_{\s, \a, \s', \a'} \p_t(\s'|\s, \a)\p(\s_t=\s, \a_t=\a|\mathcal{F}_{t-\mt})\lvert\mathbb{E}\left[\pol_t(\a'|\s')|\mathcal{F}_{t-\mt}, \mathcal{H}_t\right] - \pol_{t-\mt-1} (\a'|\s')\rvert \nonumber\\ & \quad + \frac{1}{2} \sum_{\s, \a} \lvert \p(\s_t=\s, \a_t=\a|\mathcal{F}_{t-\mt}) - \p(\Tilde{\s}_t = \s, \Tilde{\a}_t=\a |\mathcal{F}_{t-\mt})\rvert \nonumber\\ & \quad + \frac{1}{2} \sum_{\s, \a, \s', \a'} \p(\Tilde{\s}_t = \s, \Tilde{\a}_t =\a|\mathcal{F}_{t-\mt})\pol_{t-\mt-1}(\a'|\s')\lvert\p_t(\s'|\s, \a) - \p_{t-\mt}(\s'|\s, \a) \rvert \nonumber\\
        & \leq \sqrt{|\gS||\gA|}\mathbb{E}\left[\Vert\pol_t - \pol_{t-\mt-1}\Vert_2 \Big\lvert \mathcal{F}_{t-\mt} \right] + \dtv\left(\p(O_{t-1}\in\cdot |\mathcal{F}_{t-\mt}), \p(\Tilde{O}_{t-1}\in\cdot|\mathcal{F}_{t-\mt})\right) + \Vert\pv_{t}-\pv_{t-\mt}\Vert_\infty. \nonumber
    \end{align}
    Finally, recursing backwards until $\mt$ yields the result.
\end{proof}

%%%%%%%%%%%%%%%%%%%

\begin{lemma} \label{plemma:maxValues}
    If an observation is denoted as $O = (\s, \a, \s', \a')$, then the following hold for all $t, t'$
    \begin{enumerate}
        \item $\Vert \q^\pol_t \Vert_2 \leq U_Q$; $\Vert \q_t \Vert_2 \leq R_Q = U_Q$
        \item $\Vert \Av (O) \Vert_\infty \leq 2$; $\Vert \Av(O) \Vert_2 \leq \sqrt{2}$
        \item $\Vert \Bar{\Av}^{\polv, \pv} - \Bar{\Av}^{\polv', \pv'}\Vert_\infty \leq 2 \dtv\left(d^{\polv, \pv}\otimes\polv\otimes\pv\otimes\polv, d^{\polv', \pv'}\otimes\polv'\otimes\pv'\otimes\polv' \right)$
        \item $\Vert \qdv_{t+1} - \qdv_{t} \Vert_2 \leq \Vert \qv_{t+1} - \qv_t \Vert_2 + \Vert \qv_{t+1}^{\pol_{t+1}} - \qv_{t}^{\pol_t} \Vert_2$
    \end{enumerate}
\end{lemma}
\begin{proof} 
    We have the following. \vspace{-10pt}
    \begin{enumerate}
        \item See the projection operator $\Pi_{R_Q}(\cdot)$ used in \Cref{algo} and discussed further in \cref{sec:regretAnalysisAssumptions}.
        \item Follows from the definition of $\Av(O)$ in \cref{sec:regretAnalysisAssumptions}
        \item Follows from the definition of $\Bar{\Av}^{\polv, \pv}$ in \cref{sec:regretAnalysisAssumptions} and 
        \begin{align*}
            \Vert \Bar{\Av}^{\polv, \pv} - \Bar{\Av}^{\polv', \pv'}\Vert_\infty & = \max_{\s, \a} \sum_{\s', \a'} |d^{\polv, \pv}(\s, \a)\polv(\a|\s)\pv(\s'|\s,\a)\polv(\a'|\s') \\ & \quad - d^{\polv', \pv'}(\s, \a)\polv'(\a|\s)\pv'(\s'|\s,\a)\polv'(\a'|\s')|
        \end{align*}
        \item By the definition of $\qdv_t = {\color{black} \Pi_E} \lbr \qv_t - \qv_t^{\polv_t} \rbr$ and triangle inequality
    \end{enumerate}
\end{proof}

%%%%%%%%%%%%%%%%%%%
%%%%%%%%%%%%%%%%%%%

\newpage
\section{NS-NAC with Function Approximation} \label{app:functionApproximation}

{\color{black}
In this section, we present the \algoName~algorithm with function approximated policy and the state-action value function and the associated regret bound. Consider the policy $\polv_\tv$ parameterized by $\tv \in \mathbb{R}^{d}$. Consider the state-action value function $\qv^{\polv_{\tv}}(\s, \a)$ approximated as a linear function $f^T_{\tv}(\s, \a) \omega$ where $f_{\tv}(\s, \a)$ denotes the feature vector and $\omega \in \mathbb{R}^d$. We assume the actor and the critic function approximations to be compatible as $f_\tv(\s, \a) = \nabla_{\tv} \log \polv_{\tv}(\a|\s)$ \citep{sutton2019policy, konda2003actor}. The natural policy gradient \citep{SuttonBartoBook} can hence be expressed as 
\begin{align*}
    \tv_{t+1} \leftarrow \tv_t + \alr F_{\tv_t}^{-1}\mathbb{E}_{\s,\a}\lbr f_{\tv_t}(\s, \a)(f^T_{\tv_t}(\s, \a)\omega_{\tv_t}^\star)  \rbr \qquad \text{where} \qquad \omega_{\tv_t}^\star = \argmin\limits_{\omega} \mathbb{E}\lbr (\qv_t^{\polv_{\tv_t}}(\s, \a) - f^T_{\tv_t}(\s, \a) \omega)^2 \rbr.
\end{align*}
In the absence of the information of the exact gradient, the \emph{actor} update step corresponding to line~\ref{line:actorUpdate} thus becomes 
\begin{align*}
    \tv_{t+1} \leftarrow \tv_{t} + \alr \omega_{t}.
\end{align*}
The TD update step of the \emph{critic} in line~\ref{line:criticUpdate} can be written as 
\begin{align*}
    \omega_{t+1} \leftarrow \omega_t + \clr \lbr \r_t(\s_t, \a_t) - \re_t + f^T_t(\s_{t+1}, \a_{t+1}) \omega_t - f^T_t(\s_t, \a_t)\omega_t \rbr f_t(\s_t, \a_t).
\end{align*}

We now detail the assumptions under which the following upper bound on the dynamic regret of \algoName~with function approximation holds.
\begin{assumption}[Uniform Ergodicity] \label{assumption:faergodic}
     A Markov chain generated by implementing policy $\polv_\tv$ and transition probabilities $\pv$ is called uniformly ergodic, if there exists $m > 0$ and $\rho \in (0, 1)$ such that 
    \begin{equation*}
        \dtv\left(\p(\s_\mt \in \cdot | \s_0 = \s), d^{\polv_\tv, \pv}\right) \leq m\rho^{\mt} \hbox{ } \forall \mt \geq 0, \s \in \mathcal{S},
    \end{equation*}
    where $d^{\polv_\tv, \pv}$ is the stationary distribution induced over the states. We assume Markov chains induced by all potential policies $\polv_{\tv_t}$ in all environments $\pv_t$, $t \in [T]$, are uniformly ergodic.
\end{assumption}

\begin{assumption} \label{assumption:faeig}
    For all potential parameters $\tv_t$ in all environments $\pv_t, t\in [T]$, the maximum eigenvalue of matrix $\Bar{A}^{\polv_{\tv_t}, \pv_t} = \mathbb{E}_{\s, \a, \s', \a'}\lbr f_t(\s, \a) (f_t(\s', \a') - f_t(\s, \a))^T\rbr$ is $-\lambda$.
\end{assumption}

\begin{assumption}[Smoothness and Boundedness] \label{assumption:fasmoothbounded}
    For any $\tv, \tv' \in \mathbb{R}^d$ and any state-action pair $\s \in \gS, \a \in \gA$, there exist positive constants $L_A, L_C$ such that 
    \begin{enumerate}
        \item $\| f_\tv \|_2 \leq 1$,
        \item $\|f_\tv(\s, \a) - f_{\tv'}(\s, \a) \| \leq L_C \| \tv - \tv'\|_2$, and
        \item $\| \polv_\tv(\cdot | \s) - \polv_{\tv'}(\cdot | \s) \|_{TV} \leq L_A \|\tv-\tv'\|_2$.
    \end{enumerate}
\end{assumption}

\begin{definition} \label{def:faerror}
    Define the compatible linear function approximation error as 
    \begin{align*}
        \epsilon_{app} := \max_{\tv} \min_{\omega} \mathbb{E}_{\s \sim d^{\polv_\tv, \pv_t}, \a \sim \polv_\tv} \lbr \|\qv_t^{\polv}(\s, \a) - f^T_\tv(\s, \a)\omega \|_2^2 \rbr. 
    \end{align*}
\end{definition}

The dynamic regret achieved by \algoName~with function approximation described above can be upper bounded as follows.
\begin{proposition}
    If assumptions~\ref{assumption:faergodic}, \ref{assumption:faeig} and \ref{assumption:fasmoothbounded} are satisfied, $\epsilon_{app}$ is the function approximation error defined in \ref{def:faerror} and the parameters of $\algoName$ with $d$-dimensioned function approximation are chosen optimally, then
    \begin{align*}
        \hbox{Dyn-Reg}(\mathcal{M}, T) = \mathbb{E}\left[\sum_{t=0}^{T-1} J_t^{\polv_t^\star} - \r_t(\s_t, \a_t) \right] = \Tilde{\mco} \left(d^{1/2} \Delta_T^{1/6} T^{5/6}\right) + \Tilde{\mco}\left(d^{1/2} \epsilon_{app} T \right).
    \end{align*}
\end{proposition}
\begin{proof}
        Function approximation has been used commonly in actor-critic \citep{chen2023finite, Wu2020A2CPG} and natural actor-critic \citep{wang2024NonAsymptotic} algorithms in the infinite-horizon average reward setting. For the sake of brevity, we choose not to repeat the proof here and instead we point the readers to \citet{wang2024NonAsymptotic} for the technique to incorporate function approximation into our analysis of \algoName~detailed in \Cref{app:regret}-\Cref{app:plemmas} above. Note that the structure of the proof including the methods of actor, critic and average reward analyses remains the same with the only difference lying in accounting for the function approximation error $\epsilon_{app}$. 
\end{proof}

}

%%%%%%%%%%%%%%%%%%%
%%%%%%%%%%%%%%%%%%%

\section{Unknown Variation Budgets: BORL-NS-NAC} \label{app:borl}
{\color{black}

In this section, inspired by the bandit-over-RL (BORL) framework in \citet{mao2021nearOptimal, cheung2020reinforcement}, we present a parameter-free algorithm \borlAlgoName~that does not require prior knowledge of the variation budget $\Delta_T$. Further, utilizing the EXP3.P analysis from \citet{bubeck2012regret}, we present an upper bound on the dynamic regret.

\begin{algorithm}
    \caption{Bandit-over-RL Non-Stationary Natural Actor-Critic (\borlAlgoName)} \label{borlAlgo}
    \begin{algorithmic}[1]
         \STATE \textbf{Input} time horizon $T$, projection radius $R_Q$
         \STATE \textbf{Initialize} $u_{0, j} = 0, p_{0, j} = \frac{1}{\lceil \ln T \rceil}$ $\forall j \in [\lceil \ln T \rceil]$
         \FOR{$i = 0, 1, \dots, \lfloor T/W \rfloor$}
            \STATE Sample $j_i \sim p_i$ where $p_{i, j} = (1-\zeta) \frac{\exp{(\xi u_{i, j}})}{\sum_j \exp{(\xi u_{i, j}})} + \frac{\zeta}{\lceil \ln T \rceil}$
            \STATE Set $\alr = \left(\frac{T^{j_i/\lfloor \ln T \rfloor}}{T}\right)^{1/2}, \clr = \rlr = \left(\frac{T^{j_i/\lfloor \ln T \rfloor}}{T}\right)^{1/3}$ and $N = \left(T^{j_i/\lfloor \ln T \rfloor}\right)^{5/6} T^{1/6}$
            \STATE Run \algoName~(\Cref{algo}) for $W$ time-steps and observe cumulative reward $R_{i, j_i} = \sum\limits_{t=iW}^{(i+1)W-1} \r_t(\s_t, \a_t)$ 
            \STATE Update posterior as $u_{i+1, j} = u_{i, j} + \frac{\sigma + \mathbb{I}_{j=j_i} \cdot R_{i, j_i}/W}{p_{i, j}}$
         \ENDFOR
    \end{algorithmic}
\end{algorithm}

\borlAlgoName~works by leveraging the adversarial bandit framework to tune the variation budget dependent parameters in \algoName~and hedges against changes in rewards and transition probabilities. \Cref{borlAlgo} runs the EXP3.P algorithm over $\lceil T/W \rceil$ epochs with \algoName~as a sub-routine in each epoch. In each epoch, an arm of the bandit is pulled to choose the parameters of the sub-routine and the cumulative rewards received are used to update the posterior.

The space of all possible parameters is appropriately discretized and the arms of the bandit are chosen as $\mathcal{T} = \{T^0, T^{1/\lfloor \ln T \rfloor}, T^{2/\lfloor \ln T \rfloor}, \dots, T\}$. In each epoch $i$, arm $j_i$ is pulled/sampled from the distribution
\begin{align*}
    p_{i, j} = (1-\zeta) \frac{\exp{(\xi u_{i, j}})}{\sum_j \exp{(\xi u_{i, j}})} + \frac{\zeta}{\lceil \ln T \rceil},
\end{align*}
and step-sizes are chosen as $\alr = \left(\frac{T^{j_i/\lfloor \ln T \rfloor}}{T}\right)^{1/2}, \clr = \rlr = \left(\frac{T^{j_i/\lfloor \ln T \rfloor}}{T}\right)^{1/3}$ and the number of restarts is chosen as $N = \left(T^{j_i/\lfloor \ln T \rfloor}\right)^{5/6} T^{1/6}$ for the \algoName~sub-routine. The cumulative reward observed $R_{i, j_i} = \sum\limits_{t=iW}^{(i+1)W-1} \r_t(\s_t, \a_t)$ is used to update the posterior as 
\begin{align*}
    u_{i+1, j} = u_{i, j} + \frac{\sigma + \mathbb{I}_{j=j_i} \cdot R_{i, j_i}/W}{p_{i, j}}.
\end{align*}
Note that we set 
\begin{align*}
    \xi = 0.95 \sqrt{\frac{\lceil \ln T \rceil}{\lceil \ln T \rceil \lceil T/W\rceil}}, \quad \sigma = \sqrt{\frac{\lceil \ln T \rceil}{\lceil \ln T \rceil \lceil T/W\rceil}}, \quad \zeta = 1.05 \sqrt{\frac{\lceil \ln T \rceil \lceil \ln T \rceil}{\lceil T/W\rceil}}.
\end{align*}

We now present an upper bound on the dynamic regret with the proof adapted from \citet{mao2021nearOptimal, cheung2020reinforcement} which present parameter-free non-stationary model-free value-based and model-based algorithms respectively.
\begin{theorem}
    If \Cref{assumption:ergodic} is satisfied and the time horizon $T$ is divided into epochs of length $W = \mco(T^{2/3})$ in \Cref{borlAlgo}, then
    \begin{align*}
        \hbox{Dyn-Reg}(\mathcal{M}, T) \leq \Tilde{\mco} \left(|\gS|^{1/2} |\gA|^{1/2} \Delta_T^{1/6} T^{5/6} \right).
    \end{align*}
\end{theorem}
\begin{proof}
    We start by decomposing the regret, for any choice of $j^\dagger \in [\lceil \ln T \rceil]$, as follows
    \begin{align*}
        \hbox{Dyn-Reg}(\mathcal{M}, T) & = \sum_{i=0}^{\lfloor T/W \rfloor} \mathbb{E} \lbr \sum_{t=iW}^{(i+1)W-1} J_t^{\polv_t^\star} - \r_t(\s_t, \a_t) \rbr \\
        & = \sum_{i=0}^{\lfloor T/W \rfloor} \mathbb{E} \lbr \sum_{t=iW}^{(i+1)W-1} J_t^{\polv_t^\star} - \frac{R_{i, j^\dagger}}{W} \rbr + \sum_{i=0}^{\lfloor T/W \rfloor} \mathbb{E} \lbr R_{i, j^\dagger} - R_{i, j_i} \rbr \\
        & \overset{(a)}{\leq} \sum_{i=0}^{\lfloor T/W \rfloor} \mathbb{E} \lbr \sum_{t=iW}^{(i+1)W-1} J_t^{\polv_t^\star} - \frac{R_{i, j^\dagger}}{W} \rbr + \Tilde{\mco} \left(W \sqrt{\ln T \cdot \frac{T}{W}} \right) \\
        & \overset{(b)}{\leq} \Bigg[ \sum_{i=0}^{\lfloor T/W \rfloor} \Tilde{\mco}\left(\frac{N^\dagger}{\alr^\dagger} \right) + \Tilde{\mco}\left(\sqrt{\frac{N^\dagger W}{\clr^\dagger}} \right) + \Tilde{\mco}\left(\frac{\alr^\dagger W}{\clr^\dagger} \right) + \Tilde{\mco}\left(W \sqrt{\alr^\dagger} \right) + \Tilde{\mco}\left(\frac{\alr^\dagger W}{\clr^\dagger} \right) + \Tilde{\mco}\left(W \sqrt{\rlr^\dagger} \right) \\ & \quad + \Tilde{\mco}\left(\sqrt{\frac{N^\dagger W}{\rlr^\dagger}} \right) + \Tilde{\mco}\left(W \sqrt{\clr^\dagger} \right) + \Tilde{\mco}\left(\frac{\Delta_{iW, (i+1)W} W}{N^\dagger} \right) + \Tilde{\mco}\left(\frac{\Delta_{iW, (i+1)W}^{1/3} W^{2/3}}{\sqrt{\clr^\dagger}} \right) \\ & \quad + \Tilde{\mco}\left(\frac{\Delta_{iW, (i+1)W}^{1/3} W^{2/3}}{\sqrt{\rlr^\dagger}} \right) \Bigg] + \Tilde{\mco} \left(W \sqrt{\ln T \cdot \frac{T}{W}} \right) \\
        & \overset{(c)}{\leq} \Tilde{\mco}\left(\frac{T N^\dagger}{W \alr^\dagger} \right) + \Tilde{\mco}\left(\frac{T}{W}\sqrt{\frac{N^\dagger W}{\clr^\dagger}} \right) + \Tilde{\mco}\left(\frac{\alr^\dagger T}{\clr^\dagger} \right) + \Tilde{\mco}\left(T \sqrt{\alr^\dagger} \right) + \Tilde{\mco}\left(\frac{\alr^\dagger T}{\clr^\dagger} \right) + \Tilde{\mco}\left(T \sqrt{\rlr^\dagger} \right) \\ & \quad + \Tilde{\mco}\left(\frac{T}{W}\sqrt{\frac{N^\dagger W}{\rlr^\dagger}} \right) + \Tilde{\mco}\left(T \sqrt{\clr^\dagger} \right) + \Tilde{\mco}\left(\frac{\Delta_{T} W}{N^\dagger} \right) + \Tilde{\mco}\left(\frac{\Delta_{T}^{1/3} T^{2/3}}{\sqrt{\clr^\dagger}} \right) \\ & \quad + \Tilde{\mco}\left(\frac{\Delta_{T}^{1/3} T^{2/3}}{\sqrt{\rlr^\dagger}} \right) \Bigg] + \Tilde{\mco} \left(W \sqrt{\ln T \cdot \frac{T}{W}} \right) \\
    \end{align*}
    where $(a)$ follows from \Cref{theorem:regretUpperBound}, $(b)$ follows from the EXP3.P regret bound of an $\lceil \ln T \rceil$-armed bandit with rewards in $[0, W\cdot U_R]$ as detailed in Section 3.2 of \citet{bubeck2012regret} and $(c)$ follows from Jensen's inequality.

    Further, here exists some $j^\dagger$ such that 
    \begin{align*}
        \left(\frac{T^{j^\dagger/\lfloor \ln T \rfloor}}{T}\right)^{1/2} & \leq \alr^\star = \left(\frac{\Delta_T}{T} \right)^{1/2} \leq \left(\frac{T^{(j^\dagger+1)/\lfloor \ln T \rfloor}}{T}\right)^{1/2}, \\
        \left(\frac{T^{j^\dagger/\lfloor \ln T \rfloor}}{T}\right)^{1/3} & \leq \clr^\star = \rlr^\star = \left(\frac{\Delta_T}{T} \right)^{1/3} \leq \left(\frac{T^{(j^\dagger+1)/\lfloor \ln T \rfloor}}{T}\right)^{1/3},\\
        \left(T^{j^\dagger/\lfloor \ln T \rfloor}\right)^{5/6} T^{1/6} & \leq \frac{N^\star T}{W} = \Delta_T^{5/6} T^{1/6} \leq \left(T^{(j^\dagger+1)/\lfloor \ln T \rfloor}\right)^{5/6} T^{1/6}.
    \end{align*}

    We conclude the proof by adapting $\alr^\star, \clr^\star, \rlr^\star, N^\star$ to $\alr^\dagger, \clr^\dagger, \rlr^\dagger, N^\dagger$ in the above regret expression and observing that $T^{1/\lfloor \ln T \rfloor} = \mco(1)$ results in the final upper bound presented in the theorem.
\end{proof}

}

%%%%%%%%%%%%%%%%%%%
%%%%%%%%%%%%%%%%%%%

% \newpage
\section{Simulation Setup}\label{sec:sim_setup}

\paragraph{Synthetic Environment.} We empirically evaluate the performance of our algorithms on a synthetic non-stationary MDP, comparing it with three baseline algorithms: SW-UCRL2-CW \cite{cheung2023nonstationary}, Var-UCRL2 \cite{ortner2020variational}, and RestartQ-UCB (\cite{mao2021nearOptimal}). The synthetic MDP environment simulates non-stationary dynamics by alternating between two sets of transition matrices and reward functions over the time horizon $T$. The switching frequency, controlled by $n_{\text{switches}}$, determines the degree of non-stationarity and the variation budget $\Delta_{P,T}$ for transitions and $\Delta_{R,T}$ for rewards. The MDP consists of $|\gS|$ states and $|\gA|$ actions per state, with two sets of transition probabilities and rewards sampled at initialization. Further, to benchmark the effect of the dynamic changes, the optimal policy is recalculated at each switching step $t_{switch}$ by solving a linear programming problem \cite{Puterman}.

The environment alternates between these two sets of transitions and rewards, $(\pv_1, \rv_1)$ and $(\pv_2, \rv_2)$, every $T/n_{\text{switches}}$ steps. The transition probabilities, $\pv_1$ and $\pv_2$, are drawn from a Dirichlet distribution with a concentration parameter set to $0.5$, ensuring a moderate degree of randomness in the state transitions. The first reward matrix $\rv_1$ is drawn from a Beta distribution with shape parameters $\alpha = 0.5$ and $\beta = 0.5$, leading to rewards spread across the interval $[0, 1]$, with a higher probability near the extremes of 0 and 1. The second reward matrix $\rv_2$ is sampled from a Beta distribution with shape parameters $\alpha = 0.2$ and $\beta = 0.9$, producing rewards skewed toward lower values, introducing diversity in the reward structure. We use $5$ random seeds to initialize the matrices, with standard deviation capturing variability across these runs.

\paragraph{Varying $T$.} We evaluate the performance of different algorithms in the synthetic environment with $|\gS| = 50$ and $|\gA| = 4$ under varying time horizons $T$. Specifically, the time horizon $T$ is varied over the values $50\times10^3, 70\times10^3, 100\times10^3, 150\times10^3, 180\times10^3, 200\times10^3$, and $250\times10^3$. For each $T$, we set $n_{\text{switches}} = 1000$, resulting in a transition variation budget $\Delta_{P, T} = 303$, indicating significant environmental changes across the time horizon. The reward function is kept stationary (no switching between $\rv_1$ and $\rv_2$), and therefore $\Delta_{R,T} = 0$.

\paragraph{Varying $\Delta_{T}$.} We investigate the impact of changing variation budget by adjusting the number of switches $n_{\text{switches}}$ while keeping the number of states $|\gS| = 50$, actions $|\gA| = 4$, and the time horizon $T = 50\times10^3$ constant. The number of switches is varied across $10, 45, 100$, and $1000$, with both the reward function and the transition dynamics being non-stationary. The observed variation in rewards $\Delta_{R,T}$ is $9, 48, 98$, and $1000$, respectively, and the observed variation in transitions $\Delta_{P,T}$ is $4, 14, 30$, and $303$, respectively, corresponding to different levels of non-stationarity. 

\begin{figure}[h]
    \centering

    \begin{minipage}{0.64\textwidth}
        % Subfigure 1
        \subfigure[$|\gA| = 4$, $T = 5\times10^3$, $\Delta_T = 14$]{
            \includegraphics[width=0.45\linewidth]{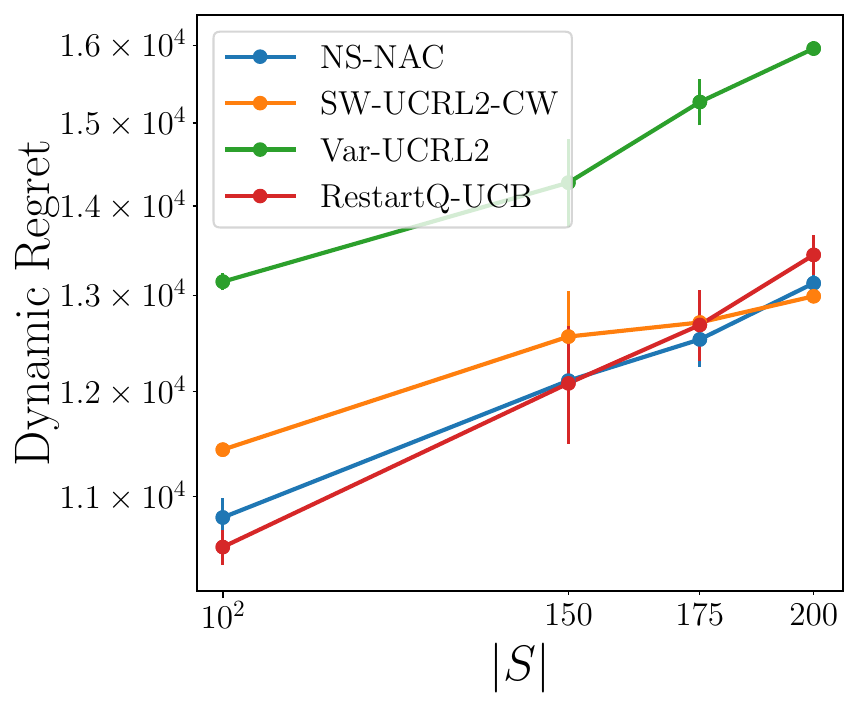}
            \label{fig:subfig4}
        }\quad
        % Subfigure 2
        \subfigure[$|\gS| = 50$, $T = 5\times10^3$, $\Delta_T = 14$]{
            \includegraphics[width=0.45\linewidth]{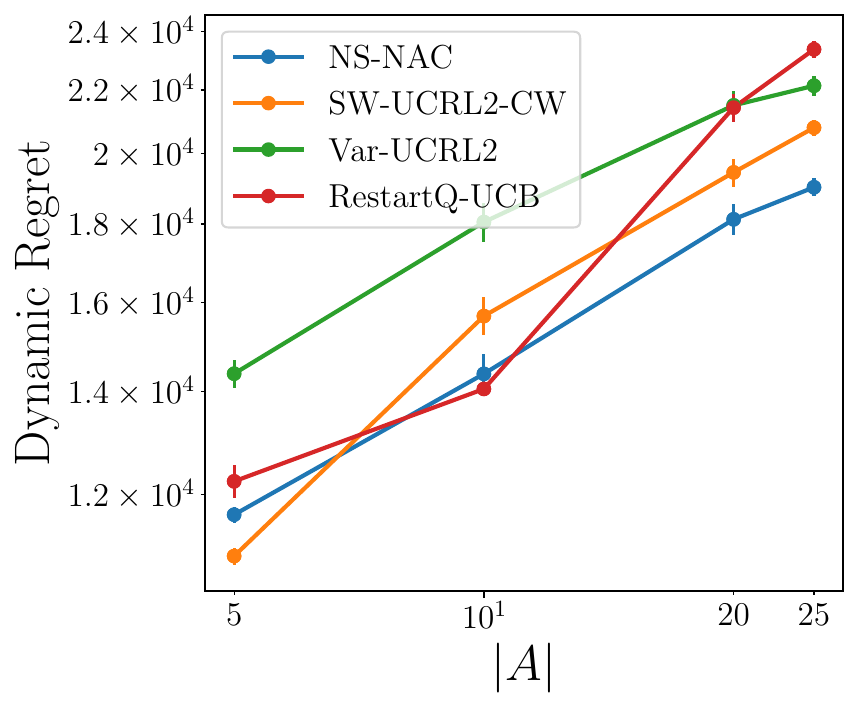}
            \label{fig:subfig5}
        }
    
        \caption{Log-log plots showing the effect of varying: (a) number of states $|\gS|$, and (b) number of actions $|\gA|$.}
        \label{fig:six_subfigures}
    \end{minipage}%
    \hfill\begin{minipage}{0.32\textwidth}
        \vspace{30pt}
        \subfigure[$|\gS| = 50$, $|\gA| = 50$, $\Delta_T = 7$]{
            \includegraphics[width=0.9\linewidth]{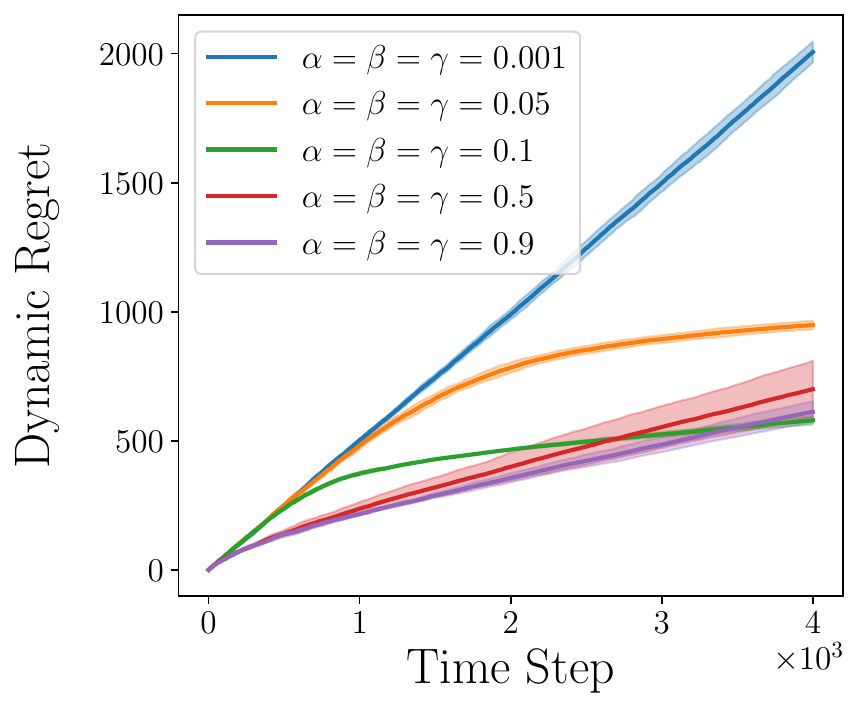}
            % \label{fig:additional_experiments_subfig1}
        }
    
        \caption{{\color{black}Performance of \algoName~with different step-sizes in an environment with 17 abrupt, randomly scheduled switches over $T = 4 \times 10^3$ steps.}}
        \label{fig:lr_expts}
    \end{minipage}
\end{figure}

\paragraph{Varying $|\gS|$.} We study the effect of varying the number of states while keeping the time horizon $T$, number of actions, and variation budget $\Delta_{T}$ constant. Specifically, the time horizon $T$ is fixed at $50\times10^3$ steps, and the number of states is varied across the values $100, 150, 175$, and $200$, corresponding to environments with different state sizes while keeping the number of actions fixed at $4$. The $n_{\text{switches}}$ is adjusted to $75, 100, 120$, and $150$, respectively, in order to maintain a consistent $\Delta_{P,T}$ of around $14$ for all environments. The reward function is kept stationary with $\Delta_{R,T} = 0$ (no switching between $\rv_1$ and $\rv_2$). 

\paragraph{Varying $|\gA|$.} We examine the effect of varying the number of actions while keeping the time horizon $T$, number of states, and variation budget $\Delta_{T}$ constant. Specifically, the time horizon $T$ is fixed at $50\times10^3$ steps, and the number of actions is varied across the values $5, 10, 20$, and $25$, corresponding to environments with different action sizes while keeping the number of states fixed at $50$. The $n_{\text{switches}}$ is kept constant at $45$ across all experiments to maintain a consistent variation budget $\Delta_{P,T}$ of around $14$ for all environments. The reward function is kept stationary with $\Delta_{R,T} = 0$ (no switching between $\rv_1$ and $\rv_2$).

\paragraph{Parameters.}
The true variation budgets, $\Delta_{P,T}$ and $\Delta_{R,T}$, are provided to each algorithm, while the remaining hyperparameters are configured according to the optimal expressions derived in their respective papers. For SW-UCRL2-CW, the parameters include the window size $W_*$ and the confidence widening parameter $\eta_*$, both set using the optimal expressions given in the paper, and the confidence parameter $\delta = 0.05$. For Var-UCRL2, the true values of the variation budgets for transitions probabilities $\Delta_{P,T}$ and rewards $\Delta_{R,T}$, along with the confidence parameter $\delta = 0.05$, are used. In RestartQ-UCB, the ending times of the stages $L$, confidence parameter $\delta = 0.05$, initial number of samples $N_0$, and number of epochs $D$ are configured as described in the original paper with $H = 1$ (to adapt from episodic setting for which the algorithm is designed to infinite horizon setting in our work). For \algoName, we tune the step-sizes and number of restarts by grid search. {\color{black} The effect of different choices of step-sizes can be observed in \Cref{fig:lr_expts}.} {\color{black} Further, for \borlAlgoName, we set the number of epochs as $W = \lfloor T^{2/3} \rfloor$.}

\begin{figure}[h]
    \centering
    % Subfigure 1
    \subfigure[$|\gS| = 50$, $|\gA| = 4$, $\Delta_T = 15$]{
        \includegraphics[width=0.4\textwidth]{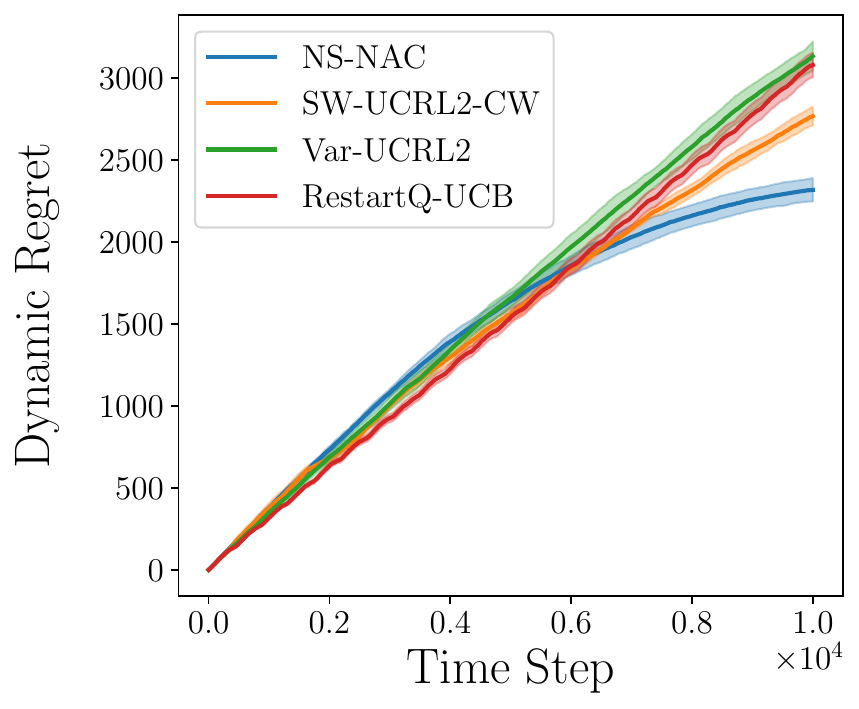}
        \label{fig:additional_experiments_subfig1}
    }\quad
    % Subfigure 2
    \subfigure[$|\gS| = 50$, $|\gA| = 4$, $\Delta_T = 0.06$]{
        \includegraphics[width=0.4\textwidth]{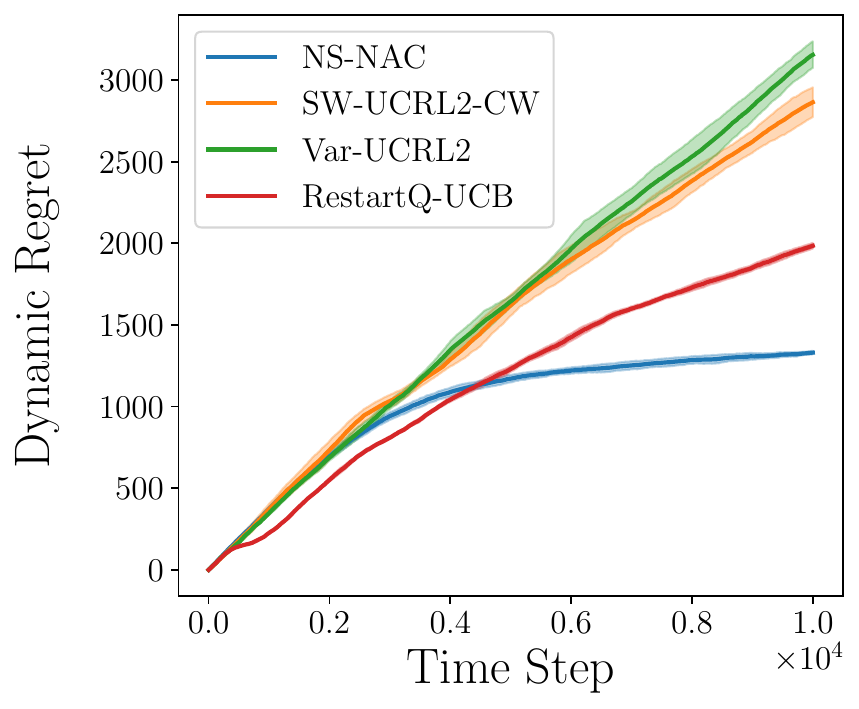}
        \label{fig:additional_experiments_subfig2}
    }

    \caption{{\color{black}Performance of NS-NAC~and baseline algorithms in various non-stationary settings. (a) Dynamic regret for a single instance over $T = 1\times10^4$ steps in an environment with 50 abrupt, randomly scheduled switches. (b) Dynamic regret for a single instance over $T = 1\times10^4$ steps in an environment with small, continuous changes.}}
    \label{fig:additional_experiments}
\end{figure}

{\color{black}
\paragraph{Additional Environments.}

We conducted further experiments to evaluate the adaptability of NS-NAC~and baseline algorithms across diverse non-stationary settings. Figure~\ref{fig:additional_experiments_subfig1} illustrates performance in an environment with 50 abrupt and randomly scheduled switches (between $\pv_1$ and $\pv_2$), simulating scenarios with non-periodic unpredictability. Figure~\ref{fig:additional_experiments_subfig2} captures performance in a continuously changing environment, where the transition from $\pv_1$ to $\pv_2$ occurred gradually over $T = 10^5$ steps resulting $\Delta_T = 0.06$. This scenario reflects real-world conditions where systems experience smooth drift rather than abrupt changes. The results highlight \algoName's effectiveness in handling both abrupt and gradual changes, consistently matching the performance of baseline methods.}

%%%%%%%%%%%%%%%%%%%
%%%%%%%%%%%%%%%%%%%

\end{document}

%% file: includes/macros.tex
\def\ceil#1{\lceil #1 \rceil}

\def\1{\bm{1}}
\def\gA{{\mathcal{A}}}

\def\gS{{\mathcal{S}}}

\newcommand{\KL}{D_{\mathrm{KL}}}
\newcommand{\nn}{\nonumber}
\newcommand{\lpar}{\left(}
\newcommand{\rpar}{\right)}

\newcommand{\lbr}{\left[}
\newcommand{\rbr}{\right]}
\newcommand{\lnr}{\left\|}
\newcommand{\rnr}{\right\|}

\newcommand\norm[1]{\lnr#1\rnr}

\newcommand{\mco}{\mathcal O}
\newcommand{\mbb}{\mathbb}
\newcommand{\mbf}{\mathbf}

\def\s{{s}} % state
\def\a{{a}} % action
\def\r{{r}} % reward
\def\p{{P}} % transition probability
\def\pol{{\pi}} % policy
\def\v{{V}} % state value function 
\def\q{{Q}} % state-action value function 
\def\polv{{\bm{\pol}}} % policy vector
\def\tv{{\bm{\theta}}} % policy parameter
\def\rv{{\mathbf{r}}} % reward vector
\def\Jv{{\mathbf{J}}}
\def\pv{{\mathbf{P}}} % transition probability vector

\def\vv{{\mathbf{V}}} % state value function vector
\def\qv{{\mathbf{Q}}} % action value function vector
\def\dtv{{d_{TV}}} % total variation distance
\def\Av{{\mathbf{A}}} % A matrix
\def\qdv{{\bm{\psi}}} % critic estimation error vector
\def\re{{\eta}} % reward estimate
\def\rev{{\bm{\eta}}} % reward estimate
\def\mt{{\tau}} % mixing time
\def\alr{{\beta}} % actor learning rate
\def\clr{{\alpha}} % critic learning rate
\def\rlr{{\gamma}} % reward learning rate
\def\temp{{\epsilon}} % entropy regularization
\def\algoName{{\textnormal{NS-NAC}}}
\def\borlAlgoName{{\textnormal{BORL-NS-NAC}}}
\def\concr{{C}}
\def\X{{\mathbf{X}}}
\def\Y{{\mathbf{Y}}}